%% file: main.tex
\documentclass[onefignum,onetabnum]{siamonline190516}

\input{1_shared}

\ifpdf
\hypersetup{
	  pdftitle={Identifiability in Two-Layer Sparse Matrix Factorization},
	  pdfauthor={L. Zheng, E. Riccietti, and R. Gribonval}
	}
\fi

\begin{document}
	
	\maketitle
	
	\input{1_article}

	\appendix
	
	\input{1_supplement}

\end{document}

%% file: 1_shared.tex
\usepackage{epstopdf}

\ifpdf
  \DeclareGraphicsExtensions{.eps,.pdf,.png,.jpg}
\else
  \DeclareGraphicsExtensions{.eps}
\fi

\usepackage{bm}
\usepackage{bbm}
\usepackage{cleveref}
\usepackage[inline]{enumitem}
\usepackage{amsopn}
\usepackage{amsfonts}
\usepackage{amssymb}
\usepackage{stmaryrd}
\usepackage{subcaption}
\usepackage{algorithmic}
\usepackage{eqparbox}

\usepackage{mathtools}
\usepackage{dsfont}

\Crefname{ALC@unique}{Line}{Lines}

\usepackage{mymacros}

\usepackage{enumitem}
\setlist[enumerate]{leftmargin=.5in}
\setlist[itemize]{leftmargin=.5in}

\newsiamremark{remark}{Remark}
\newsiamremark{example}{Example}
\crefname{example}{Example}{Example}

\headers{Identifiability in Two-Layer Sparse Matrix Factorization}{L. Zheng, E. Riccietti, and R. Gribonval}

\title{Identifiability in Two-Layer Sparse Matrix Factorization\thanks{Preprint.\funding{This project was supported in part by the AllegroAssai ANR project ANR-19-CHIA-0009.}}}

\author{L\'eon Zheng\footnotemark[3] \thanks{Universit\'e de Lyon, ENS de Lyon, UCBL, CNRS, Inria, LIP, F-69342, LYON Cedex 07, France.
(\email{leon.zheng@ens-lyon.fr}, \email{elisa.riccietti@ens-lyon.fr}, \email{remi.gribonval@inria.fr}). $^\ddag$valeo.ai, Paris, France.}
\and Elisa Riccietti\footnotemark[2]
\and R\'emi Gribonval\footnotemark[2]}


%% file: 1_article.tex
%
%
%
%
%
%
%

\begin{abstract}
	Sparse matrix factorization is the problem of approximating a matrix $\mat{Z}$ by a product of $J$ sparse factors $\matseq{\mat{X}}{J} \matseq{\mat{X}}{J-1} \ldots \matseq{\mat{X}}{1}$. 
	This paper focuses on identifiability issues that appear in this problem, in view of better understanding under which sparsity constraints the problem is well-posed.
	We give conditions under which the problem of factorizing a matrix into \emph{two} sparse factors admits a unique solution, up to unavoidable permutation and scaling equivalences. Our general framework considers an arbitrary family of prescribed sparsity patterns, allowing us to capture more structured notions of sparsity than simply the count of nonzero entries. These conditions are shown to be related to essential uniqueness of exact matrix decomposition into a sum of rank-one matrices, with structured sparsity constraints. In particular, in the case of  fixed-support sparse matrix factorization, we give a general sufficient condition for identifiability based on rank-one matrix completability, and we derive from it a completion algorithm that can verify if this sufficient condition is satisfied, and recover the entries in the two sparse factors if this is the case.
	A companion paper further exploits these conditions to derive identifiability properties and theoretically sound factorization methods for multi-layer sparse matrix factorization with support constraints associated to some well-known fast transforms such as the Hadamard or the Discrete Fourier Transforms.
\end{abstract}



\section{Introduction}
	\label{section:introduction}

The problem of \emph{sparse matrix factorization} is, given $J \geq 2$ and a matrix $\mat{Z}$, to find factors $\matseq{\mat{X}}{1}, \ldots, \matseq{\mat{X}}{J}$ such that
\begin{equation}
	\label{eq:sparse-approx}
	\mat{Z} \approx \matseq{\mat{X}}{J} \matseq{\mat{X}}{J-1} \ldots \matseq{\mat{X}}{1},
\end{equation}
and such that the factors $\matseq{\mat{X}}{\ell}$ ($1 \leq \ell \leq J$) are sparse, in the sense that they have few nonzero entries.
Such a factorization would allow to use the product of sparse factors as a surrogate of the linear operator associated to $\mat{Z}$ for speeding up numerical methods and reducing memory complexity \cite{le2016flexible, 7178579, 5452966, 5325694}.

One approach to approximate a matrix by a product of sparse factors is to consider specific hand-designed sparsity patterns for the factors, for example the butterfly factorization model \cite{dao2019learning}. Although common matrices such as the Discrete Fourier Transform (DFT), the Discrete Cosine Transform (DCT), the Discrete Sine Transform (DST), convolution or Hadamard matrices can be exactly decomposed into factors using this butterfly factorization model, the same doesn't hold for general matrices. An alternative approach is to consider a more general family of sparsity patterns, and try to find a sparsity structure which gives the smallest approximation error by efficiently exploring  this family. Classical families of sparse matrix supports are for instance those which are globally $s$-sparse (at most $s$ nonzero entries), $k$-sparse by column and/or $l$-sparse by row, etc. This approach has been considered for instance in \cite{le2016flexible}, where the authors formulated \eqref{eq:sparse-approx} as an optimization problem constrained to a general family of sparsity patterns, and proposed to explore it with a proximal gradient algorithm. 

Yet, it is still not clear why or when these heuristic algorithms succeed or fail in finding a good sparse approximation. Little is known about conditions for which we can hopefully find a good approximation of $\mat{Z}$ by a product of sparse factors with tractable algorithms. 
In particular, adopting a \emph{multilinear inverse problem} point of view, i.e., an inverse problem where the measurements are a multilinear function of the unknowns, it is not even known when the \textit{exact} sparse matrix factorization problem (the case in which a matrix can be exactly decomposed as a product of sparse factors) is well-posed, in the sense that the solution of the inverse problem is \emph{identifiable}, i.e., unique, up to natural and unavoidable scaling and permutation ambiguities. This paper studies such identifiability issues in exact sparse matrix factorization, for general families of sparsity patterns. Following the hierarchical approach proposed by \cite{le2016flexible}, we focus on the case with $J=2$ factors, and illustrate in the companion paper  \cite{LeonPart2} that our analysis can provide tools for the investigation of the problem extended to $J > 2$ factors.

\subsection{Related work}

Although identifiability in sparse linear inverse problems is now well understood \cite{elad2010sparse, foucart_mathematical_2013}, identifiability in \emph{multilinear} inverse problems regularized by sparsity is still very much an open question. 

\paragraph{Lifting in blind deconvolution}
One main framework for studying identifiability in bilinear inverse problems is the \emph{lifting} procedure \cite{choudhary2014identifiability}. This framework takes into account scaling ambiguities inherent to bilinear inverse problems, and transforms a bilinear inverse problem into a low-rank matrix recovery problem. 
Identifiability is then related to geometric properties of the corresponding lifting operator, and more precisely, to the rank-two null space of the operator. By giving a simple parametrization of this rank-two null space specific to the case of blind deconvolution \cite{choudhary2014identifiability}, it is possible for instance to show  a negative result about identifiability of the solution to blind deconvolution regularized by sparsity in the canonical basis
\cite{choudhary2014sparse}. However, such a parametrization is specific to the convolution, and does not generalize to the more general problem of sparse matrix factorization. In fact, the analysis of identifiability in blind deconvolution with sparsity or subspace constraints typically relies on an application of the lifting procedure in the frequency domain \cite{li2017identifiability}, which is specific to the convolution operation. Still, blind deconvolution can be seen as a particular instance of sparse matrix factorization into two factors, with a specific structure enforced on the two factors. One should expect some similarities in the analysis of identifiability in these two problems. In fact, this paper will reuse the lifting procedure to analyze identifiability in sparse matrix factorization.

\paragraph{Deep linear neural networks}
The lifting procedure has also been generalized to the multilinear case. In deep linear neural networks with structured layers \cite{malgouyres2016identifiability}, identifiability up to certain global layerwise scaling equivalences can be characterized with a tensorial lifting procedure \cite{malgouyres2019multilinear}. However, this analysis is limited to the specific case where each layer is constrained to a fixed support. In general, sparsity constraints are described by a family of sparsity patterns, and not reduced to specific supports. 
Considering a family of possible supports for the structured layers, some stability properties of deep structured linear networks can be also derived from the tensorial lifting procedure \cite{malgouyres2020stable}. However, when the family of possible supports presents some invariance property by permutation of their columns (as it is the case for the above mentioned classical families), permutation ambiguities should be considered in the analysis, in the spirit of \cite[Chapter 3]{stock2021efficiency}. Our work takes into account such permutation ambiguities.

\paragraph{Interpretability}
Sparsity in learned models is desirable for reducing their complexity and to improve interpretability.  
In particular, such an interpretability would require some kind of stability or identifiability of the parameters. 
Identifiability in nonnegative matrix factorization (NMF) \cite{gillis2020nonnegative} (or in simplex-structured matrix factorization, a generalization of the NMF problem \cite{abdolali2021simplex}) has been studied for guaranteeing interpretability of the factors, when there exists a physical model explaining the data, like in blind hyperspectral unmixing or audio source separation. For that purpose, a sufficient condition for identifiability of the factors when the right factor is $k$-sparse by column has been established \cite[Chapter 4]{gillis2020nonnegative}, but the result is only limited to this kind of sparsity structure. 
Instead, similarly to \cite{malgouyres2020stable}, our work considers, for more generality, a general family of sparsity patterns, or sparsity structures, in the spirit of model-based compressive sensing \cite{baraniuk2010model}. This general framework is exploitable to provide ground to guide further development toward the choice of a good family of sparsity patterns to obtain well-posed sparse matrix factorization problems, as it is shown in the companion paper \cite{LeonPart2}.

\paragraph{Compressive sensing and sampling complexity}
From an information-theoretic point of view, identifiability in inverse problems is studied to derive optimal sampling complexity bounds to ensure non-ambiguous reconstruction of a signal from its dimension-reduced linear measurement.
Typically, in sparsity-regularized linear inverse problems, it is known from the compressive sensing literature \cite{foucart_mathematical_2013} that at least $m \geq 2s$ measurements are needed in order to reconstruct an $s$-sparse vector from its linear measurement.
In bilinear inverse problems it is possible, based on results on information-theoretic limits of unique low-rank matrix recovery \cite{riegler2015information}, to give near-optimal sampling complexity results for blind deconvolution \cite{li2017identifiability}, which can be further improved using algebraic geometry \cite{kech2017optimal}. 
Yet, optimal sampling complexity bounds in sparse matrix factorization are still an open question.
As every instance of a sparse linear inverse problem can be seen as an instance of a sparse matrix factorization problem with two factors, where the left one is \emph{fixed}, we will show in this paper that our framework allows to generalize a classical result in compressive sensing \cite[Theorem 2.13]{foucart_mathematical_2013} stating the link between the identifiability of $s$-sparse vectors and the Kruskal-rank \cite{KRUSKAL197795} of the measurement matrix, i.e., the largest integer $j$ such that every set of $j$ columns in this matrix is linearly independent.
Typically, when the left factor is fixed, instead of considering $s$-sparsity by column on the right factor, which would then decouple the problem into independent linear inverse problems under an $s$-sparsity constraint, we show new identifiability results for other sparsity constraints on the right factor, such as sparsity by row or global sparsity (see \Cref{cor:uniform-right-identifiability-classical-family}). 

\subsection{Contributions}

Our paper studies identifiability in exact sparse matrix factorization {\em with two factors} $(\mat{X}, \mat{Y})$. Specifically, we formulate in \cref{section:problem-formulation} the problem of exact sparse matrix factorization of $\mat{Z} := \mat{X} \transpose{\mat{Y}}$ as a constrained bilinear inverse problem.
This framework allows to consider general families of sparsity patterns: the constraint set of the inverse problem is a union of subspaces corresponding to prescribed sparsity patterns, i.e., given supports. 
As the sparsity pattern of a matrix is invariant to scaling of its columns, scaling ambiguities are inherent in the sparse matrix factorization problem. Permutation ambiguities can also appear for certain families of sparsity patterns. Hence, \emph{identifiability} of the sparse factorization is then defined as the uniqueness of the solution to the inverse problem, up to scaling and permutation equivalences (\Cref{def:PS-uniqueness}). 

Our analysis of this uniqueness property starts by pinpointing some non-degeneration properties, which are necessary for a pair of factors to be identifiable in the sparse matrix factorization problem (\Cref{prop:non-degenerate-properties}). Then, inspired by the lifting procedure \cite{choudhary2014identifiability}, we show that, up to these non-degeneration properties, identifiability in exact sparse matrix factorization is equivalent to uniqueness of \emph{exact matrix decomposition} 
into rank-one matrices (\Cref{def:P-uniqueness-EMD}), with corresponding sparsity patterns (\Cref{thm:equivalence-with-EMD}). By analogy with sparse linear inverse problem, we propose to characterize such uniqueness in two steps: 
\begin{enumerate*}[label=(\roman*)]
	\item characterization of identifiability of the constraint supports for the rank-one matrices among a family of sparsity patterns;
	\item characterization of identifiability of the rank-one matrices after fixing these constraint supports.
\end{enumerate*}
In particular, based on rank-one matrix completability conditions \cite{kiraly2012combinatorial}, we give a condition for identifiability in fixed-support matrix factorization, which is shown to be naturally associated to connectivity in certain bipartite graphs (see \Cref{thm:SC-fixed-support-id}). We derive a completion algorithm that can verify if this sufficient condition is satisfied, and recover the entries in the two sparse factors if this is the case (see \Cref{algo:completion}).

When fixing the left factor, the bilinear inverse problem becomes a linear one, so conditions for identifiability of the right factor are naturally related to linear independence of specific subsets of columns in the fixed left factor (\Cref{thm:right-uniqueness-perm} and \Cref{prop:right-uniqueness-linear-inv-pb-matrix}). Based on this characterization, we can show some simple characterization of \emph{uniform} right identifiability with fixed left factor $\mat{X}$, i.e., the identifiability of \emph{every} sparse right factor $\mat{Y}$ from the observed matrix $\mat{Z} := \mat{X} \transpose{\mat{Y}}$. For instance, when enforcing at most $s$ nonzero entries on the right factor, we show that uniform right identifiability with fixed left factor $\mat{X}$ holds if, and only if, every subset of $2s$ columns in $\mat{X}$ is linearly independent (\Cref{cor:uniform-right-identifiability-classical-family}). 

In complement to this paper, the companion paper \cite{LeonPart2} presents an application of our framework to show some identifiability results in the multi-layer sparse matrix factorization of some well-known matrices, like the Hadamard or DFT matrices. 

\paragraph{Summary} The main contributions of this paper are the following.
\begin{enumerate}
	\item We show equivalence between uniqueness of exact sparse matrix factorization and uniqueness of exact sparse matrix decomposition into rank-one matrices up to natural ambiguities, except in trivial degenerate cases (\Cref{thm:equivalence-with-EMD}).
	\item We express some sufficient conditions for identifying the constraint support among the family of sparsity patterns (\Cref{prop:simple-SC-identify-supports}), which are verified in practice in the case of the sparse factorization of the DFT, DCT-II or DST-II matrices.
	\item We characterize right identifiability (\Cref{thm:right-uniqueness-perm}, \Cref{prop:right-uniqueness-linear-inv-pb-matrix,prop:right-id-up-permutation}), using linear independence of specific sets of columns in the fixed left factor, and derive as a by-product a characterization of uniform right identifiability (\Cref{cor:uniform-right-uniqueness,cor:uniform-right-identifiability-classical-family}). 
	\item We give a general sufficient condition for identifiability in fixed-support sparse matrix factorization based on rank-one matrix completability (\Cref{thm:SC-fixed-support-id}), and we derive from it a natural factorization procedure described by \Cref{algo:completion}.
\end{enumerate}

The paper is organized as follows: \cref{section:problem-formulation} is dedicated to the formulation of the exact sparse matrix factorization problem with two factors, and defines the uniqueness property that we want to characterize throughout the paper; \cref{section:right-identifiability} characterizes right identifiability, which is identifiability of the right factor when the left one is fixed; \cref{section:EMD-identifiability} studies identifiability in exact sparse matrix factorization via the uniqueness of exact matrix decomposition into rank-one matrices, up to permutation equivalence; \cref{section:uniform-identifiability} regroups some results on uniform identifiability in the case with two factors; \cref{section:conclusion} discusses perspectives of this work for stability issues. Technical proofs are deferred to the appendices.

\section{Problem formulation}
\label{section:problem-formulation}
We introduce our framework to analyze identifiability in exact sparse matrix factorization.

\subsection{Notations}
\label{section:notations}
The set of integers $\{1, \ldots, n\}$ is denoted $\integerSet{n}$. The cardinality of any finite set $F$ is denoted $\card(F)$. The complement of a set $I$ is denoted $I^c$.
The \emph{support} of a matrix $\mat{M} \in \mathbb{C}^{m \times n}$ of size $m \times n$ is the set of indices $\supp(\mat{M}) \subseteq \integerSet{m} \times \integerSet{n}$ of nonzero entries. It is identified by abuse of notation to the set of binary matrices $\mathbb{B}^{m \times n} := \{ 0, 1\}^{m \times n}$. Depending on the context, a matrix support can be seen as a set of indices, or a binary matrix with only nonzero entries for indices in this set. The cardinality of the matrix support $\supp(\mat{M})$ is also known as the $\ell_0$-norm of $\mat{M}$, denoted $\| \cdot \|_0$. 
The \emph{column support}, denoted $\colsupp({\mat{M}})$, is the subset of indices $i \in \integerSet{r}$ such that the $i$-th column of $\mat{M}$, denoted $\matcol{\mat{M}}{i}$, is nonzero. The entry of $\mat{M}$ indexed by $(k, l)$ is $\matindex{\mat{M}}{k}{l}$. The $j$-th column of $\mat{M}$ restricted only to row indices $i \in I$ is denoted $\matindex{\mat{M}}{I}{j}$. For any subset of column indices $I$, the notation $\submat{\mat{M}}{I}$ indicates the column submatrix $(\matcol{\mat{M}}{i})_{i \in I}$. The notation $\matseq{\mat{M}}{i}$ denotes the $i$-th matrix in a collection. 
For any subset of row and column indices $I$ and $J$, the set of submatrices defined on the set $I \times J$ is denoted $\mathbb{C}^{I \times J}$. 
The vector full of ones indexed by a subset of indices $I$ is denoted $\one_I$, and for any integer $n$, we write $\one_n = \one_{\integerSet{n}}$. 
Subscript is omitted when there is no ambiguity. 
Vectors or matrices full of zeros are denoted $\mat{0}$ with bold case, without specifying the dimension.
The canonical basis in $\mathbb{C}^n$ is denoted $\{ \basis{i} \}_{i=1}^n$.
The identity matrix of size $n$ is denoted $\identity{n}$. 
The kernel and range of a matrix $\mat{M}$ are denoted $\ker(\mat{M})$ and $\im(\mat{M})$. The Kronecker product \cite{Kronecker} between two matrices $\mat{A}$ and $\mat{B}$ is written $\mat{A} \otimes \mat{B}$.
The set of permutations on a set of indices $I$ is written $\mathfrak{S}(I)$. The function $f$ iterated $n$ times is denoted $\iterate{f}{n}$. 

\subsection{Exact matrix factorization}
Given an observed matrix $\mat{Z} \in \mathbb{C}^{m \times n}$, and a subset of feasible pairs of factors $\Sigma \subseteq \mathbb{C}^{m \times r} \times \mathbb{C}^{n \times r}$, the so-called \emph{exact matrix factorization}  (EMF) problem with two factors of $\mat{Z}$ in $\Sigma$ is the following bilinear inverse problem: 
\begin{equation}
	\label{eq:EMF}
	\text{find if possible}\footnote{For arbitrary $\mat{Z}$ there does not necessarily exist a pair $(\mat{X}, \mat{Y}) \in \Sigma$ such that $\mat{Z} = \mat{X} \transpose{\mat{Y}}$.} (\mat{X}, \mat{Y}) \in \Sigma \text{ such that } \mat{Z} = \mat{X} \transpose{\mat{Y}}.
\end{equation}
We are interested in the particular problem variation where the constraint set $\Sigma$ encodes some chosen sparsity pattern for the factorization. 
For a given binary matrix $\mat{S} \in \mathbb{B}^{m \times r}$ associated to a sparisty pattern, denote 
	\begin{displaymath}
		\Sigma_{\mat{S}} := \{ \mat{M} \in \mathbb{C}^{m \times r}\; | \; \supp(\mat{M}) \subseteq \supp(\mat{S}) \}
	\end{displaymath}
	the so-called \emph{model-set} defined by $\mat{S}$, which is the set of matrices with a sparsity pattern included in $\mat{S}$. A pair of sparsity patterns is written $\pair{S} := (\leftsupp{S}, \rightsupp{S})$, where $\leftsupp{S}$ and $\rightsupp{S}$ are the left and right sparsity patterns respectively. By abuse, we will refer to $\leftsupp{S}$ and $\rightsupp{S}$ as left and right (allowed) supports.
Given any pair of allowed supports represented by binary matrices $(\leftsupp{S}, \rightsupp{S}) \in \mathbb{B}^{m \times r} \times \mathbb{B}^{n \times r}$, the set $\Sigma_{\leftsupp{S}} \times \Sigma_{\rightsupp{S}} \subseteq \mathbb{C}^{m \times r} \times \mathbb{C}^{n \times r}$ is a linear subspace. Given any family $\Omega  \subset \mathbb{B}^{m \times r} \times \mathbb{B}^{n \times r}$ of such pairs of allowed supports, denote 
\begin{equation}
	\label{eq:model-set-family-pairs}
	\Sigma_{\Omega} := \bigcup_{\pair{S} \in \Omega} \Sigma_{\pair{S}}, \quad \text{with } \pair{S} := (\leftsupp{S}, \rightsupp{S}), \; \text{and } \Sigma_{\pair{S}} := \Sigma_{\leftsupp{S}} \times \Sigma_{\rightsupp{S}}.
\end{equation} 
Such a set is a union of subspaces. Moreover, since the support of a matrix is unchanged under arbitrary rescaling of its columns, $\Sigma_{\Omega}$ is invariant by column scaling for any family $\Omega$.
This framework covers some classical families $\Omega$ of structured sparse supports.

\begin{example} {\color{header1}\normalfont\sffamily Globally $s$-sparse matrices:} matrices with at most $s$ nonzero entries. This is the most general sparsity pattern, since it does not specify any kind of sparsity structure.
\end{example}

\begin{example} {\color{header1}\normalfont\sffamily Matrices that are $k$-sparse by column and/or $l$-sparse by row:} each column has at most $k$ nonzero entries, and/or each row has at most $l$ nonzero entries. For instance, in sparse coding, the dataset represented by  a matrix $\mat{X}$ is decomposed over an overcomplete dictionary $\mat{D}$, in such a way that each column of $\mat{X}$ can be expressed as a linear combination of few atoms, i.e., columns of $\mat{D}$. In terms of matrices, this is written as $\mat{X} = \mat{D} \mat{W}$, where $\mat{W}$ is sparse by column. 
\end{example}

In the following, the set of supports (of a given size) which are globally $s$-sparse, $k$-sparse by column and $l$-sparse by row are respectively denoted $\globalsparse{s}$,  $\columnsparse{k}$, and  $\rowsparse{l}$:
\begin{align}
	\globalsparse{s} :=& \; \{ \mat{S} \; | \; \| \mat{S} \|_0 \leq s\}, \label{eq:global-sparse}\\
	\columnsparse{k} :=& \; \{ \mat{S} \; | \; \forall j, \; \| \matcol{\mat{S}}{j} \|_0 \leq k\}, \label{eq:col-sparse}\\ 
	\rowsparse{l} :=& \; \{ \mat{S} \; | \; \forall i, \; \| \matcol{(\transpose{\mat{S}})}{i} \|_0 \leq l \}. \label{eq:row-sparse}
\end{align}

\begin{example} {\color{header1}\normalfont\sffamily $k$-regular matrices:} square matrices that are both  $k$-sparse by column and by row \cite{le2021structured}. In the butterfly factorization of the DFT matrix \cite{dao2019learning}, each butterfly factor is $2$-regular. 
\end{example}

In addition to the generic invariance to column scaling of $\Sigma_{\Omega}$, the above classical families are also invariant to column permutation, which leads to the following definition.

\begin{definition}[Stability by permutation]
	We say that a family of pairs of supports $\Omega$ is \emph{stable by permutation} if for any pair $\pair{S} \in \Omega$, we have $(\leftsupp{S} \mat{P}, \rightsupp{S} \mat{P}) \in \Omega$ for all $\mat{P} \in \permutation{r}$, where $\permutation{r}$ is the group of permutation matrices of size $r \times r$.
\end{definition}

Considering $\Lambda^L$ (resp. $\Lambda^R$) a classical family of sparse left (resp. right) supports, in the sense that they are one of the families presented in the previous examples, the family of pairs of supports $\Omega := \Lambda^L \times \Lambda^R$ is stable by permutation\footnote{Most families of pairs of supports are \emph{not} stable by permutation: consider for example any family $\Omega$ reduced to a single fixed pair of supports, or a few supports. 
	A concrete example can be built using supports of the butterfly factors of the DFT matrix \cite{dao2019learning}.
}.
Hence, uniqueness of a solution to \eqref{eq:EMF} with such sparsity constraints will always be considered up to scaling and permutation equivalences. The following definition is a generalization of the one of \cite[Chapter 4]{gillis2020nonnegative} to study identifiability in nonnegative matrix factorization (NMF). Hereafter, the group of diagonal matrices of size $r \times r$ with nonzero diagonal entries is denoted $\diagonal{r}$, while the group of \emph{generalized permutation} matrices of size $r \times r$ is denoted $\genperm{r} := \{ \mat{D} \mat{P} \; | \; \mat{D} \in \diagonal{r}, \; \mat{P} \in \permutation{r} \}$.

\begin{definition}[PS-uniqueness of an EMF in $\Sigma$]
	\label{def:PS-uniqueness}
	For any set $\Sigma$ of pairs of factors, the pair $(\mat{X}, \mat{Y}) \in \Sigma$ is the \emph{PS-unique} EMF of $\mat{Z} := \mat{X} \transpose{\mat{Y}}$ in $\Sigma$, if any solution $(\mat{X'}, \mat{Y'})$ to \eqref{eq:EMF} with $\mat{Z}$ and $\Sigma$ is equivalent to $(\mat{X}, \mat{Y})$, written $(\mat{X'}, \mat{Y'}) \sim (\mat{X}, \mat{Y})$, in the sense that there exists a permutation matrix $\mat{P} \in \permutation{r}$ and a diagonal matrix $\mat{D} \in \diagonal{r}$ with nonzero diagonal entries such that $(\mat{X'}, \mat{Y'}) = (\mat{X} \mat{D} \mat{P}, \mat{Y} \mat{D}^{-1} \mat{P})$; or, alternatively, there exists a generalized permutation matrix $\mat{G} \in \genperm{r}$ such that $(\mat{X'}, \mat{Y'}) = \left(\mat{X} \mat{G}, \mat{Y} \transpose{(\mat{G}^{-1})} \right)$.
\end{definition}
For any set $\Sigma$ of pairs of factors, the set of all pairs $(\mat{X}, \mat{Y}) \in \Sigma$ such that $(\mat{X}, \mat{Y})$  is the PS-unique EMF of $ \mat{Z} := \mat{X} \transpose{\mat{Y}} \text{ in } \Sigma$ is denoted $\unique(\Sigma)$. In other words, we define:
\begin{equation}
	\unique(\Sigma) := \left \{ (\mat{X}, \mat{Y}) \in \Sigma \; | \; \forall (\mat{X'}, \mat{Y'}) \in \Sigma, \; \mat{X'} \transpose{\mat{Y'}} = \mat{X} \transpose{\mat{Y}} \implies (\mat{X'}, \mat{Y'}) \sim (\mat{X}, \mat{Y})  \right \}.
\end{equation}

\begin{remark}
	The property $(\mat{X}, \mat{Y}) \in \unique(\Sigma)$, referred to as \emph{instance} PS-uniqueness of an EMF in $\Sigma$, corresponds to the notion of weak identifiability in \cite{li2017identifiability}, while the property $\unique(\Sigma) = \Sigma$, referred to as \emph{uniform} PS-uniqueness of EMF in $\Sigma$ (or uniform identifiability), corresponds to the notion of strong identifiability in \cite{li2017identifiability}. In this paper, we prefer the terminology ``instance" and ``uniform", for they are more explicit about whether the uniqueness is specific to a particular instance of pair of factors.
\end{remark}

\begin{remark}
	When considering the set $\Sigma$ of pairs of nonnegative factors, $\unique(\Sigma)$ has  been characterized with necessary conditions and sufficient conditions in \cite[Chapter 4]{gillis2020nonnegative}, in the sense that we can rewrite identifiability conditions  in NMF as inclusions of the set $\unique(\Sigma)$ with respect to other sets. 
\end{remark}

This paper aims to characterize $\unique(\Sigma_{\Omega})$ for $\Sigma_{\Omega}$ defined as in \eqref{eq:model-set-family-pairs} with $\Omega$ any family of pairs of supports that is \emph{stable} by permutation. 
Our analysis of such uniqueness property relies on the following abstract but simple lemma. 

\begin{lemma}
	\label{lemma:nec-conditions-EMF}
	Let $\Sigma$ be any set of pairs of factors, and $(\mat{X}, \mat{Y}) \in \Sigma$ be a pair of factors. Then:
	\begin{displaymath}
		(\mat{X}, \mat{Y}) \in \unique(\Sigma) \iff (\mat{X}, \mat{Y}) \in \bigcap_{\substack{\Sigma' \subseteq \Sigma \\ (\mat{X}, \mat{Y}) \in \Sigma'}} \unique(\Sigma').
	\end{displaymath}
\end{lemma}

\begin{remark}
	The spirit of this result is reminiscent of \cite[Chapter 3]{gillis2020nonnegative} where restricted exact nonnegative matrix factorization can be seen as a subset of exact nonnegative matrix factorization.
\end{remark}

\begin{proof}
	Let $(\mat{X}, \mat{Y}) \in \unique(\Sigma)$, and consider $\Sigma' \subseteq \Sigma$ such that $(\mat{X}, \mat{Y})  \in \Sigma'$, as well as $(\mat{X'}, \mat{Y'}) \in \Sigma'$ such that $\mat{X'} \transpose{\mat{Y'}} = \mat{X} \transpose{\mat{Y}}$. Since $\Sigma' \subseteq \Sigma$, we have $(\mat{X'}, \mat{Y'}) \in \Sigma$, and because $(\mat{X}, \mat{Y}) \in \unique(\Sigma)$, $(\mat{X'}, \mat{Y'}) \sim (\mat{X}, \mat{Y})$. Moreover, $(\mat{X}, \mat{Y}) \in \Sigma'$. In conclusion, $(\mat{X}, \mat{Y}) \in \unique(\Sigma')$. This is true for every $\Sigma' \subseteq \Sigma$, so it proves one implication. 
	The converse is true by considering the particular case $\Sigma' := \Sigma$.
\end{proof}

\subsection{Non-degeneration properties for a pair of factors}
\label{section:non-degen-prop}
A first analysis of PS-uniqueness of an EMF in $\Sigma_{\Omega}$ leads to the formulation of two non-degeneration properties for a pair of factors, i.e., two necessary conditions for identifiability,  involving their so-called \emph{column support}.
They can be derived from the following trivial but crucial observation. 
\begin{lemma}
	\label{lemma:on-met-ce-qu-on-veut-a-droite}
	Let $\Sigma$ be any set of pairs of factors, $\mat{X}$ be a left factor, and $\mat{Y}, \mat{Y'}$ be two right factors such that $(\mat{X}, \mat{Y}), (\mat{X}, \mat{Y'}) \in \Sigma$. If $\mat{X}\transpose{\mat{Y}}=\mat{X}\transpose{\mat{Y}'}$ and $\card(\colsupp(\mat{Y'})) \neq \card(\colsupp(\mat{Y}))$, then $(\mat{X}, \mat{Y}) \not\sim (\mat{X}, \mat{Y'})$ and hence $(\mat{X}, \mat{Y}) \notin \unique(\Sigma)$ and $(\mat{X}, \mat{Y'}) \notin \unique(\Sigma)$. This is true in particular if there exists an index $i \in \integerSet{r}$ for which $\matcol{\mat{X}}{i} = \vctor{0}$, $\matcol{\mat{Y}}{i} = \vctor{0}$, $\matcol{\mat{Y'}}{i} \neq \vctor{0}$, and $\matcol{\mat{Y}}{j} = \matcol{\mat{Y'}}{j}$ for all $j \neq i$.
\end{lemma}

The first non-degeneration property for identifiability of an EMF in $\Sigma_{\Omega}$ thus requires the left and right factors to have \emph{the same column supports}. Define the set of pairs of factors with \emph{identical column supports} in $\Sigma_{\Omega}$ as 
\begin{equation}
	\label{eq:DefIC}
	\idcolsupp{\Omega} := \left \{ (\mat{X}, \mat{Y}) \in \Sigma_{\Omega} \; | \; \colsupp (\mat{X}) = \colsupp (\mat{Y}) \right\}. 
\end{equation}

\begin{lemma}
	\label{lemma:identical-colsupp}
	For any family of pairs of supports $\Omega$, we have: $\unique(\Sigma_{\Omega}) \subseteq \idcolsupp{\Omega}$.
\end{lemma}

\begin{proof}
	We prove the contraposition. Let $(\mat{X}, \mat{Y}) \in \Sigma_{\Omega}$, and suppose that $\colsupp(\mat{X}) \neq \colsupp(\mat{Y})$. Up to matrix transposition, we can suppose without loss of generality that $\colsupp(\mat{Y})$ is not a subset of $\colsupp(\mat{X})$, so there is $i \in \integerSet{r}$ such that $\matcol{\mat{X}}{i} = \vctor{0}$ and $\matcol{\mat{Y}}{i} \neq \vctor{0}$. 
	Define  $\mat{Y'}$ a right factor such that $\matcol{\mat{Y'}}{i} = \vctor{0}$ and $\submat{\mat{Y'}}{\integerSet{r} \backslash \{ i \}} = \submat{\mat{Y}}{\integerSet{r} \backslash \{ i \}}$. By construction, $\supp(\mat{Y'}) \subseteq \supp(\mat{Y})$, so $(\mat{X}, \mat{Y'}) \in \Sigma_{\Omega}$.
	Applying \Cref{lemma:on-met-ce-qu-on-veut-a-droite} to $\Sigma = \Sigma_{\Omega}$, we obtain $(\mat{X}, \mat{Y}) \notin \unique(\Sigma_{\Omega})$.
\end{proof}

The second non-degeneration property for a pair of factors requires the column supports of the left and right factors to be ``maximal''. Define the set of pairs of factors with \emph{maximal column supports} in $\Sigma_{\Omega}$ as 
\begin{align}
	\label{eq:DefMC}
	\maxcolsupp{\Omega} := 
	\{ (\mat{X}, \mat{Y}) \in \Sigma_{\Omega} \; | \; \forall \pair{S} \in \Omega \text{ such that } (\mat{X}, \mat{Y}) \in \Sigma_{\pair{S}}, \, & \colsupp(\mat{X}) = \colsupp(\leftsupp{S})\\
	\text{ and } & \colsupp(\mat{Y}) = \colsupp(\rightsupp{S})\}. \notag
\end{align}

\begin{lemma}
	\label{lemma:maximal-colsupp}
	For any family of pairs of supports $\Omega$, we have: $\unique(\Sigma_{\Omega}) \subseteq \maxcolsupp{\Omega}$.
\end{lemma}

The proof is deferred to  \cref{proof:lemma-maximal-colsupp}.

Therefore, without loss of generality, we only need to characterize the pairs $(\mat{X}, \mat{Y}) \in \idcolsupp{\Omega} \cap \maxcolsupp{\Omega}$ such that $(\mat{X}, \mat{Y})$ is the PS-unique EMF of $\mat{Z} := \mat{X} \transpose{\mat{Y}}$ in $\Sigma_{\Omega}$. 
Considering these two non-degeneration properties, we conclude this section by a key result which will be useful in \cref{section:EMD-identifiability}. The proof is deferred to  \cref{proof:prop-non-degenerate-properties}.

\begin{proposition}
	\label{prop:non-degenerate-properties}
	For any family of pairs of supports $\Omega$ that is stable by permutation, we have: 
	\begin{displaymath}
		\unique (\Sigma_{\Omega}) = \unique (\idcolsupp{\Omega}) \cap \maxcolsupp{\Omega}.
	\end{displaymath}
\end{proposition}

\section{Identifiability when fixing one factor}
\label{section:right-identifiability}

A natural analysis of PS-uniqueness in EMF with sparsity constraints is to consider the case where one of the factors is fixed. As one can always consider matrix transposition, the fixed factor will be the left one. For a family of right supports $\Theta \subseteq \mathbb{B}^{n \times r}$, denote:
\begin{equation}
	\Sigma_{\Theta} := \bigcup_{\mat{S} \in \Theta} \Sigma_{\mat{S}}, \quad \text{where we recall that } \Sigma_{\mat{S}} := \{ \mat{Y} \in \mathbb{C}^{n \times r} \; | \; \supp(\mat{Y}) \subseteq \mat{S} \}.
\end{equation}
Studying identifiability when fixing one factor is interesting because it allows one to characterize identifiability of an EMF. Indeed,
applying the general framework for studying identifiability up to a transformation group 
proposed by \cite{li2015unified} to exact sparse matrix factorization, we obtain the following proposition. The specific transformation group considered here is the group of permutations.

\begin{proposition}[Application of {\cite[Theorem 2.8]{li2015unified}}]
	\label{prop:identifiability-left-and-right}
	Suppose that $\Omega = \Lambda \times \Theta$ where $\Lambda \subseteq \mathbb{B}^{m \times r}$ and $\Theta \subseteq \mathbb{B}^{n \times r}$ are respectively families of left and right supports, and that $\Omega$ is stable by permutation. Then, $(\mat{X}, \mat{Y}) \in \unique(\Sigma_{\Omega})$ if, and only if, both of the following conditions are verified:
	\begin{enumerate}[label=(\roman*)]
		\item For all $(\mat{X'}, \mat{Y'}) \in \Sigma_{\Omega}$ such that $\mat{X'} \transpose{\mat{Y'}} = \mat{X} \transpose{\mat{Y}}$, there exists a generalized permutation matrix $\mat{G} \in \genperm{r}$ such that $\mat{X'} = \mat{X} \mat{G}$.
		\item The pair $(\mat{X}, \mat{Y})$ is the unique EMF of $\mat{Z} := \mat{X} \transpose{\mat{Y}}$ in $\{ \mat{X} \} \times \Sigma_{\Theta}$.
	\end{enumerate}
\end{proposition}

\begin{remark}
	In general, even without assuming that $\Omega$ is stable by permutation, $(\mat{X}, \mat{Y}) \in \unique(\{\mat{X}\} \times \Sigma_{\Theta})$ is a necessary condition for $(\mat{X}, \mat{Y}) \in \unique(\Sigma_{\Omega})$, by \Cref{lemma:nec-conditions-EMF} with $\Sigma' := \{\mat{X}\} \times \Sigma_{\Theta}$.
\end{remark}

This section will show that condition (ii), referred to as \emph{right identifiability}, can be entirely characterized using linear independence of specific subsets of columns in the fixed left factor $\mat{X}$. Characterization of condition (i) will be discussed in \cref{section:conclusion} as a future work. 
Throughout this section, we thus consider a fixed left factor $\mat{X}$ and a family of right supports $\Theta$, and we characterize the set $\unique(\{ \mat{X}\} \times \Sigma_{\Theta})$. Since the left factor is fixed, we will use an equivalent representation for the equivalences between two pairs of factors $(\mat{X}, \mat{Y})$ and $(\mat{X}, \mat{Y'})$.
Denote:
\begin{equation*}
	\label{eq:group-X-stable}
	\begin{split}
		\mathcal{G}(\mat{X}) &:= \{ \mat{G} \in \mathcal{G}_{r} \; | \; \mat{X} \mat{G} = \mat{X} \}, \\
		\mathcal{P}(\mat{X}) &:= \{ \mat{P} \in \mathcal{P}_{r} \; | \; \mat{X} \mat{P} = \mat{X} \}, \\
		\mathcal{D}(\mat{X}) &:= \{ \mat{D} \in \mathcal{D}_{r} \; | \; \mat{X} \mat{D} = \mat{X} \},
	\end{split}
\end{equation*}
respectively the subgroups of generalized permutation matrices, permutation matrices, and diagonal matrices with nonzero diagonal entries, that leave matrix $\mat{X}$ unchanged after right multiplication. 

\subsection{Eliminating scaling ambiguities}

The characterization of right identifiability, i.e., $(\mat{X}, \mat{Y}) \in \unique(\{\mat{X}\} \times \Sigma_{\Theta})$ with fixed left factor $\mat{X}$ in the family of right supports $\Theta$, starts with the following necessary condition, which can also be seen as a non-degeneration property in the spirit of \cref{section:non-degen-prop}.

\begin{lemma}
	\label{lemma:right-id-first-lemma}
	Suppose that $(\mat{X}, \mat{Y}) \in \unique(\{\mat{X}\} \times \Sigma_{\Theta})$. Then, $\colsupp(\mat{Y}) \subseteq \colsupp(\mat{X})$.
\end{lemma}

\begin{proof}
	Suppose that there exists $i \in \integerSet{r}$ such that $\matcol{\mat{Y}}{i} \neq \vctor{0}$ but $\matcol{\mat{X}}{i} = \vctor{0}$. Then, defining $\mat{Y'} \in \Sigma_{\Theta}$ such that $\matcol{\mat{Y'}}{i} = \vctor{0}$ and $\matcol{\mat{Y'}}{j} = \matcol{\mat{Y}}{j}$ for $j \neq i$, we apply \Cref{lemma:on-met-ce-qu-on-veut-a-droite} to show that $(\mat{X}, \mat{Y}) \notin \unique(\{ \mat{X} \} \times \Sigma_{\Theta})$.
\end{proof}

Hence, we can assume without loss of generality that $\colsupp(\mat{Y}) \subseteq \colsupp(\mat{X})$.
We now show that we can restrict our analysis only to column indices in $\colsupp(\mat{X})$. Let us introduce the notion of \emph{signature} for a family of supports. For any subset of column indices $J \subseteq \integerSet{r}$, define the signature on $J$ of the family $\Theta \subseteq \mathbb{B}^{n \times r}$ of right supports  as the family: 
\begin{equation}
	\label{eq:signature}
	\signature{\Theta}{J} := \{ \submat{\mat{S}}{J} \; | \; \mat{S} \in \Theta \} \subseteq \mathbb{B}^{\integerSet{n} \times J},
\end{equation}
where we recall that $\submat{\mat{S}}{J}$ is the submatrix of $\mat{S}$ in $\mathbb{B}^{\integerSet{n} \times J}$ with only columns indexed by $J$.

\begin{example}[Signature in a classical family of supports]
	Consider $\Theta$ the family of sparse supports which have $r$ columns, and are $s$-sparse by column. Then, for any subset $J \subseteq \integerSet{r}$, the signature of $\Theta$ on $J$ is  the family of subsupports $s$-sparse by column, where the column indices of the subsupports are indexed by $J$.
\end{example}

\begin{lemma}
	\label{lemma:eliminating-zero-col-right-id}
	Given a left factor $\mat{X}$ and a right factor $\mat{Y} \in \Sigma_{\Theta}$, suppose that $\colsupp(\mat{Y}) \subseteq \colsupp(\mat{X})$, and denote $J := \colsupp(\mat{X})$. Then:
	\begin{displaymath}
		(\mat{X}, \mat{Y}) \in \unique(\{ \mat{X} \} \times \Sigma_{\Theta}) \iff (\submat{\mat{X}}{J}, \submat{\mat{Y}}{J}) \in \unique(\{ \submat{\mat{X}}{J} \}\times \Sigma_{\signature{\Theta}{J}}).
	\end{displaymath}
\end{lemma}

The proof is deferred to \cref{proof:lemma-eliminating-zero-col-right-id}.

As the submatrix $\submat{\mat{X}}{J}$ does not have any zero column, we can assume from now on that, without loss of generality, the fixed left factor does not have any zero column.  As we now show, this allows one to get rid of scaling ambiguities using column normalization: when all columns of the fixed left factor are normalized, PS-uniqueness in EMF is equivalent to uniqueness up to permutation only.

\begin{definition}[P-uniqueness of an EMF in $\Sigma$]\label{def:puniqueness}
	For any set $\Sigma$ of pairs of factors, the pair $(\mat{X}, \mat{Y}) \in \Sigma$ is the \emph{P-unique} EMF of $\mat{Z} := \mat{X} \transpose{\mat{Y}}$ in $\Sigma$, if any solution $(\mat{X'}, \mat{Y'})$ to \eqref{eq:EMF} with $\mat{Z}$ and $\Sigma$ is equivalent to $(\mat{X}, \mat{Y})$ up to permutation, written $(\mat{X'}, \mat{Y'}) \sim_p (\mat{X}, \mat{Y})$, in the sense that there exists a permutation matrix $\mat{P} \in \permutation{r}$ such that $(\mat{X'}, \mat{Y'}) = (\mat{X} \mat{P}, \mat{Y} \mat{P})$. 
\end{definition}

The set of all pairs $(\mat{X}, \mat{Y}) \in \Sigma$ such that $(\mat{X}, \mat{Y}) \text{ is the P-unique EMF of } \mat{Z} := \mat{X} \transpose{\mat{Y}} \text{ in } \Sigma$ is denoted $\uniqueperm(\Sigma)$. 
We now claim the main result of this subsection. 

\begin{proposition}
	\label{prop:right-id-up-permutation}
	Consider $\mat{X}$ with no zero column and $\mat{N} \in \diagonal{r}$ the (unique) diagonal matrix that normalizes its columns, in such a way that for each column of $\mat{X} \mat{N}$, the nonzero entry with the smallest row index is 1.
	 We have:
	\begin{equation*}
		\begin{split}
			(\mat{X}, \mat{Y}) \in \unique(\{ \mat{X} \}\times \Sigma_{\Theta}) \iff& (\mat{X} \mat{N}, \mat{Y} \inverse{\mat{N}}) \in \unique(\{ \mat{X} \mat{N} \} \times \Sigma_{\Theta})  \\
			\iff & (\mat{X} \mat{N}, \mat{Y} \inverse{\mat{N}}) \in \uniqueperm(\{ \mat{X} \mat{N}\} \times \Sigma_{\Theta}).
		\end{split}
	\end{equation*}
\end{proposition}

\begin{proof}
	For the first equivalence, suppose that $(\mat{X}, \mat{Y}) \in \unique( \{\mat{X}\} \times \Sigma_{\Theta})$. 
	Let $\mat{Y} \in \Sigma_{\Theta}$ such that $(\mat{X} \mat{N}) \transpose{\mat{Y}} = (\mat{X} \mat{N}) \transpose{( \mat{Y} \inverse{\mat{N}}) }$. 
	Then, $\mat{X} \transpose{(\mat{Y} \mat{N})} = \mat{X} \transpose{\mat{Y}}$, so by assumption, $(\mat{X}, \mat{Y} \mat{N}) \sim (\mat{X}, \mat{Y})$. But $(\mat{X}, \mat{Y} \mat{N}) \sim (\mat{X} \mat{N}, \mat{Y})$. And $(\mat{X}, \mat{Y}) \sim (\mat{X} \mat{N}, \mat{Y} \inverse{\mat{N}} )$. So we conclude that $(\mat{X} \mat{N}, \mat{Y}) \sim (\mat{X} \mat{N}, \mat{Y} \inverse{\mat{N}})$, which shows $(\mat{X} \mat{N}, \mat{Y} \inverse{\mat{N}} ) \in \unique( \{ \mat{X} \mat{N} \} \times \Sigma_{\Theta} )$. The converse is true by applying the previous implication where $\mat{X}$ and $\mat{D}$ are respectively replaced by $\mat{XN}$ and $\inverse{\mat{N}}$.
	The second equivalence comes from the fact that $\mathcal{G}(\mat{X}) = \mathcal{P}(\mat{X})$, assuming that $\mat{X}$ has no zero column, and that the first nonzero entry of each column in $\mat{X}$ is 1. 
\end{proof}

The rest of the section characterizes $\uniqueperm( \{ \mat{X} \} \times \Sigma_{\Theta})$ when $\mat{X}$ has no zero column.

\subsection{Identifiability up to permutation when fixing the left factor}

When fixing a factor, the bilinear inverse problem under sparsity constraints becomes a linear one, so conditions of right identifiability are naturally related to linear independence of specific subsets of columns in the fixed left factor $\mat{X}$.

To see that, we define an equivalence relation between the columns of $\mat{X}$ defined by collinearity: we say that two columns in $\mat{X}$ are equivalent if they are collinear. When the columns of $\mat{X}$ are normalized as in \Cref{prop:right-id-up-permutation}, equivalent columns are simply identical columns. We denote $K$ the total number of equivalent classes of columns (with $1 \leq K \leq r$), and $[k]$ the subset of column indices in the $k$-th equivalent class\footnote{The order of these equivalent classes does not matter.} ($k \in \integerSet{K}$). In this way, the sets $[1], \ldots, [K]$ form a partition of $\integerSet{r}$. 

As shown in the previous subsection, we can suppose without loss of generality that the fixed left factor has nonzero columns, and that the first nonzero entry of each column in $\mat{X}$ is normalized to 1. As a consequence, denoting $\vctorseq{\vctor{x}}{k}$ an arbitrary representative of the $k$-th equivalent class of collinear columns, the matrix $\mat{X}$ can be expressed as
\begin{displaymath}
	\mat{X} = \begin{pmatrix}
		  \vctorseq{\vctor{x}}{1} \transpose{\one}_{\card([1])} & \ldots &  \vctorseq{\vctor{x}}{K} \transpose{\one}_{\card([K])} 
	\end{pmatrix} \mat{P},
\end{displaymath} 
where $\mat{P}$ is a permutation matrix. Thanks to the following lemma, we can suppose that $\mat{P}$ is the identity matrix without loss of generality.

\begin{lemma}
	\label{lemma:right-identifiability-permutation-equiv}
	Let $\mat{X}$ be a left factor, $\mat{P}$ be a permutation matrix, and $\Theta$ be a family of right supports. Denote $\Theta \inverse{\mat{P}} := \{ \mat{S} \inverse{\mat{P}} \; | \; \mat{S} \in \Theta\}$. Then, for any $\mat{Y} \in \Sigma_{\Theta}$:
	\begin{displaymath}
		(\mat{X} \mat{P}, \mat{Y}) \in \uniqueperm(\{ \mat{X} \mat{P} \} \times \Sigma_{\Theta}) \iff (\mat{X}, \mat{Y} \inverse{\mat{P}}) \in \uniqueperm(\{ \mat{X} \} \times \Sigma_{\Theta \inverse{\mat{P}}}).
	\end{displaymath}
\end{lemma}

\begin{remark}
	When $\Theta$ is invariant by permutation of columns, we simply have $\Theta \inverse{\mat{P}} = \Theta$. 
\end{remark}

\begin{proof}
	The proof is similar to the one of the first equivalence in \Cref{prop:right-id-up-permutation}.
\end{proof}

Before giving our characterization of identifiability up to permutation when fixing the left factor, let us introduce the following notation.
For any family of right supports $\Theta \subseteq \mathbb{B}^{n \times r}$ and any partition $\{ [k] \}_{k=1}^K$ of $\integerSet{r}$, we introduce the \emph{fingerprint} 
of $\Theta$ on the partition $\{ [k] \}_{k=1}^K$ as:
\begin{equation}
	\label{eq:fingerprint}
	\fingerprint{\Theta}{\{ [k] \}_{k=1}^K} := \left \{ \supp \begin{pmatrix} \submat{\mat{S}}{[1]} \one_{[1]} & \ldots & \submat{\mat{S}}{[K]} \one_{[K]} \end{pmatrix} \; \vert \; \mat{S} \in \Theta \right \} \subseteq \mathbb{B}^{n \times K}.
\end{equation}

\begin{example}[Fingerprint of a classical family of supports]
	Consider $\Theta$ a family of supports with $r = 2K$ columns, which are $s$-sparse by column. Consider the partition $\integerSet{r} = \bigcup_{k=1}^{K} I_k$ where $I_k = \{ 2k-1; 2k\}$. Then, the fingerprint of $\Theta$ on the partition $\{I_k\}_{k=1}^K$ is the family of supports with $K$ columns, which are $2s$-sparse by column.
\end{example}

We are now able to characterize right identifiability in the following theorem, whose proof is deferred to \cref{proof:right-uniqueness-perm}.
\begin{theorem}
	\label{thm:right-uniqueness-perm}
	Consider $[1], \ldots, [K]$ a partition of $\integerSet{r}$, and a fixed left factor of the form $\mat{X} = \begin{pmatrix}
		\vctorseq{\vctor{x}}{1} \transpose{\one}_{\card([1])} & \ldots &  \vctorseq{\vctor{x}}{K} \transpose{\one}_{\card([K])} 
	\end{pmatrix}$, with $\vctorseq{\vctor{x}}{k} \neq 0$ for every $k$. Consider also a family of right supports $\Theta$, and a right factor $\mat{Y} \in \Sigma_{\Theta}$.	
	Denote $\mat{\underline{X}} := ( \vctorseq{\vctor{x}}{k} )_{k \in \integerSet{K}} \in \mathbb{C}^{m \times K}$ and
	$\mat{\underline{Y}} := \begin{pmatrix}
		\submat{\mat{Y}}{[1]} \one_{[1]} & \ldots & \submat{\mat{Y}}{[K]} \one_{[K]}
	\end{pmatrix} \in \mathbb{C}^{n \times K}$.
	Let $\tilde{\Theta}$ be the fingerprint of $\Theta$ on the partition $\{ [k]\}_{k=1}^K$, and
	$\Theta_{[k]}$ be the signature of $\Theta$ on $[k]$ defined by \eqref{eq:signature}.
	
	Then, $(\mat{X}, \mat{Y})$ is the P-unique EMF (or equivalently the PS-unique EMF) of $\mat{Z} := \mat{X} \transpose{\mat{Y}}$ in $\{ \mat{X} \} \times \Sigma_{\Theta}$ if, and only if, the following conditions are verified:
	
	\begin{enumerate}[label=(\roman*)]
		\item For all $\mat{\underline{Y}'} \in \Sigma_{\tilde{\Theta}}$ such that $\mat{\underline{X}} \transpose{ \mat{{\underline{Y}'}} } = \mat{\underline{X}} \transpose{ \mat{\underline{Y}}}$, we have $\mat{\underline{Y}} = \mat{\underline{Y}'} $.
		\item For each $k \in \integerSet{K}$, we have
		$(\transpose{\one_{[k]}}, \submat{\mat{Y}}{[k]}) \in \uniqueperm(\{ \transpose{\one_{[k]}}\} \times \Sigma_{\Theta_{[k]}})$.
	\end{enumerate}
\end{theorem}

\begin{remark}
	An equivalent manner to express condition (ii) is to write $(\submat{\mat{Y}}{[k]}, \transpose{\one}_{[k]}) \in \uniqueperm(\Sigma_{\Theta_{[k]}} \times \{ \transpose{\one}_{[k]}\} )$ instead of $(\transpose{\one_{[k]}},\submat{\mat{Y}}{[k]}) \in \uniqueperm(\{ \transpose{\one_{[k]}}\} \times \Sigma_{\Theta_{[k]}})$.
\end{remark}

Let us now characterize more precisely condition (i) and (ii) of the theorem. The following proposition is analogue to \cite[Proposition 1]{lu2008theory}, but instead of considering \emph{uniform} uniqueness, i.e., uniqueness of the reconstruction of each signal from its observation, we consider \emph{instance} uniqueness.

\begin{proposition}
	\label{prop:right-uniqueness-linear-inv-pb-matrix}
	Let $\Theta \subseteq \mathbb{B}^{n \times r}$ be a family of right supports and $(\mat{X},  \mat{Y}) \in \mathbb{C}^{m \times r} \times \Sigma_{\Theta}$ be a pair of factors. 
	Then, the property of uniqueness
	\begin{displaymath}
		\forall \mat{Y'} \in \Sigma_{\Theta}, \quad \mat{X} \transpose{\mat{Y}} = \mat{X} \transpose{\mat{Y'}} \implies \mat{Y} = \mat{Y'}
	\end{displaymath}
	holds if, and only if, both of the following conditions hold:
	\begin{enumerate}[label=(\roman*)]
		\item For each $\mat{S} \in \Theta$ such that $\mat{Y} \in \Sigma_{\mat{S}}$, for each $j \in \integerSet{n}$, the columns $\{ \matcol{\mat{X}}{l} \; | \; l \in \matcol{(\transpose{\mat{S}})}{j} \}$ are linearly independent.
		
		\item For each $\mat{S} \in \Theta$ such that the $j$-th column of $\mat{X} \transpose{\mat{Y}}$ (for all $j \in \integerSet{n}$) is a linear combination of columns $\{ \matcol{\mat{X}}{l} \; | \; l \in \matcol{(\transpose{\mat{S}})}{j} \}$, we have $\mat{Y} \in \Sigma_{\mat{S}}$.
	\end{enumerate}
\end{proposition}

The proof is deferred to  \cref{proof:prop-right-uniqueness-linear-inv-pb-matrix}.

\begin{proposition}
	\label{prop:right-uniqueness-perm-vector-of-ones}
	Let $\Theta \subseteq \mathbb{B}^{n \times r}$ be a family of supports, and $\mat{Y} \in \mathbb{C}^{n \times r}$ be a matrix. 
	Then, $( \transpose{\one},\mat{Y}) \in \unique( \{ \transpose{\one} \} \times \Sigma_{\Theta})$
	if, and only if, the following conditions hold:
	\begin{enumerate}[label=(\roman*)]
		\item The column supports $\{ \supp(\mat{Y})_j \}_{j=1}^r$ are pairwise disjoint.
		\item For any $\mat{S} \in \Theta$ such that $I := \supp(\mat{Y} \one) \subseteq \bigcup_{i=1}^r \matcol{\mat{S}}{i}$, the columns $( \matindex{\mat{S}}{I^c}{j} )_{j=1}^r$ are pairwise disjoint, and the columns $( \matindex{\mat{S}}{I}{j} )_{j=1}^r$ and $(\matindex{\supp(\mat{Y})}{I}{j})_{j=1}^r$ are equal, up to a permutation of column indices $j$.
	\end{enumerate}
\end{proposition}

The proof is 
	deferred to  \cref{proof:prop-right-uniqueness-perm-vector-of-ones}.

In \cref{section:EMD-identifiability}, this characterization of right identifiability will be used to derive necessary conditions in fixed-support identifiability. It will be also used to derive a characterization of uniform right identifiability in \cref{section:uniform-identifiability}.

\section{Identifiability in exact matrix decomposition}
\label{section:EMD-identifiability}

In this section we go back to the general characterization of $\unique(\Sigma_{\Omega})$ without fixing a particular factor. We present our analysis of identifiability in exact matrix sparse factorization with the lifting approach \cite{choudhary2014identifiability}, based on \emph{exact matrix decomposition} into rank-one matrices with sparsity constraints.

\subsection{Principle of the lifting approach}

As the matrix product $\mat{X} \transpose{\mat{Y}}$ can be decomposed into the sum of rank-one matrices $\sum_{i=1}^r \matcol{\mat{X}}{i} \transpose{\matcol{\mat{Y}}{i}}$, the lifting procedure \cite{choudhary2014identifiability, malgouyres2016identifiability} suggests to represent a pair $(\mat{X}, \mat{Y})$ by its $r$-tuple of so-called \emph{rank-one contributions} 
\begin{equation}\label{eq:DefRankOneMapping}
	\varphi(\mat{X}, \mat{Y}) := (\matcol{\mat{X}}{i} \transpose{\matcol{\mat{Y}}{i}})_{i=1}^r.
\end{equation} 
Indeed, one can always identify, up to scaling ambiguities, the columns $\matcol{\mat{X}}{i}$, $\matcol{\mat{Y}}{i}$ from their outer product $\rankone{C}{i} = \matcol{\mat{X}}{i} \transpose{\matcol{\mat{Y}}{i}}$ ($1 \leq i \leq r$), \emph{as long as} the rank-one contribution $\rankone{C}{i}$ is not the zero matrix.
\begin{lemma}[Reformulation of {\cite[Chapter 7, Lemma 1]{le2016matrices}}]
	\label{lemma:id-from-outerproduct}
	Consider $\rankone{C}{}$ the outer product of two vectors $\vctor{a}$, $\vctor{b}$. If $\rankone{C}{} = \vctor{0}$, then $\vctor{a} = \vctor{0}$ or $\vctor{b} = \vctor{0}$. If $\rankone{C}{} \neq \vctor{0}$, then $\vctor{a}$, $\vctor{b}$ are nonzero, and for any $(\vctor{a'}, \vctor{b'})$ such that $\vctor{a'} \transpose{\vctor{b'}} = \rankone{C}{}$, there exists a scalar $\lambda \neq 0$ such that $\vctor{a'} = \lambda \vctor{a}$ and $\vctor{b'} = \frac{1}{\lambda} \vctor{b}$.
\end{lemma}

\begin{remark}
	In the case where $\rankone{C}{} = \vctor{0}$ with $\vctor{a} = \vctor{0}$, then $\vctor{a} \transpose{\vctor{b'}} = \vctor{0} = \rankone{C}{}$ for any $\vctor{b'}$. In other words, it is not possible to identify the factors $(\vctor{a}, \vctor{b})$ from the outer product $\rankone{C}{} = \vctor{a} \transpose{\vctor{b}}$ when $\rankone{C}{} = \vctor{0}$. 
\end{remark}

With this lifting approach,
each support constraint $\pair{S} = (\leftsupp{S}, \rightsupp{S})$ is represented by the $r$-tuple of rank-one support constraints $\tuplerkone{S} = \varphi(\leftsupp{S}, \rightsupp{S})$. 
Thus, when $(\mat{X}, \mat{Y})$ follows a sparsity structure given by $\Omega$, i.e., belongs to $\Sigma_{\Omega} \subseteq \mathbb{C}^{m \times r} \times \mathbb{C}^{n \times r}$, the $r$-tuple of rank-one matrices $\varphi(\mat{X}, \mat{Y}) \in (\mathbb{C}^{m \times n})^{r}$ belongs to the set:
\begin{gather}
	\Gamma_{\Omega} := \bigcup_{\tuplerkone{S} \in \varphi(\Omega)} \Gamma_{\tuplerkone{S}}, \quad \text{with} \quad \tuplerkone{S} := (\rankone{S}{1}, \ldots, \rankone{S}{r}), \quad \text{and} \\
	\Gamma_{\tuplerkone{S}} := \left \{ (\rankone{C}{i})_{i=1}^r \; | \; \forall i \in \integerSet{r}, \;\rank(\rankone{C}{i}) \leq 1, \, \supp(\rankone{C}{i}) \subseteq \rankone{S}{i} \right \} \subseteq (\mathbb{C}^{m \times n})^{r}.
\end{gather}

As we considered permutation and scaling equivalences between pairs of factors, it is natural to consider similar equivalences between tuples of rank-one matrices. However, as we will see in \Cref{lemma:varphi-properties} below, the application $\varphi$ removes scaling ambiguities, so we only have to introduce permutation equivalence between $r$-tuple of rank-one matrices.

\begin{definition}[Permutation equivalence between $r$-tuples of rank-one matrices]
	For two $r$-tuples of rank-one matrices $\tuplerkone{C} = (\rankone{C}{i})_{i=1}^r$ and $\tuplerkone{C'} = (\rankone{C'}{i})_{i=1}^r$, we write $\tuplerkone{C} \sim \tuplerkone{C'}$ if the tuples are equal up to a permutation of the index $i$.
\end{definition}

The following lemma 
states that there is a one-on-one correspondence between a pair of sparse factors $(\mat{X}, \mat{Y})$ and its rank-one contributions representation $\varphi(\mat{X}, \mat{Y})$, up to equivalences, which justifies the lifting procedure to analyze identifiability in sparse matrix factorization. The proof is deferred to  \cref{proof:lemma-varphi-properties}.

\begin{lemma}
	\label{lemma:varphi-properties}
	The application $\varphi$ preserves equivalences, in the sense that for any equivalent pair of factors $(\mat{X}, \mat{Y}) \sim (\mat{X'}, \mat{Y'})$, we have $\varphi(\mat{X}, \mat{Y}) \sim \varphi(\mat{X'}, \mat{Y'})$. 
	For any family of pairs of supports $\Omega$, the application $\varphi$ restricted to $\idcolsupp{\Omega}$, denoted $\varphi_{\Omega} \colon \idcolsupp{\Omega} \to \Gamma_{\Omega}$, is surjective, and injective up to equivalences, in the sense that for all $(\mat{X}, \mat{Y}), (\mat{X'}, \mat{Y'}) \in \idcolsupp{\Omega}$ such that $\varphi_{\Omega}(\mat{X}, \mat{Y}) \sim \varphi_{\Omega}(\mat{X'}, \mat{Y'})$, we have $(\mat{X}, \mat{Y}) \sim (\mat{X'}, \mat{Y'})$.
\end{lemma}

As we show next, the most important property of the lifting approach with respect to identifiability is
that PS-uniqueness of an EMF in $\Sigma = \idcolsupp{\Omega}$ is equivalent to identifiability of the rank-one contributions in $\Gamma_{\Omega}$. Denote
\begin{equation}
	\mathcal{A}: \tuplerkone{C} = (\rankone{C}{i})_{i=1}^r\mapsto \sum_{i=1}^r \rankone{C}{i}
\end{equation}
the linear operator which sums the $r$ matrices of a tuple $\tuplerkone{C}$.

\begin{definition}[P-uniqueness of an EMD in $\Gamma$]
	\label{def:P-uniqueness-EMD}
	For any set $\Gamma \subseteq (\mathbb{C}^{m \times n})^{r}$ of $r$-tuples of rank-one matrices, the $r$-tuple $\tuplerkone{C} \in \Gamma$ is the P-unique \emph{exact matrix decomposition} (EMD) of $\mat{Z} := \mathcal{A}(\tuplerkone{C})$ in $\Gamma$ if, for any $\tuplerkone{C'} \in \Gamma$ such that $\mathcal{A}(\tuplerkone{C'}) = \mat{Z}$, we have $\tuplerkone{C'} \sim \tuplerkone{C}$.
\end{definition}

The set of all $r$-tuples $\tuplerkone{C} \in \Gamma$ such that $\tuplerkone{C}$ is the P-unique EMD of $\mat{Z} := \mathcal{A}(\tuplerkone{C})$ in $\Gamma$ is denoted $\unique(\Gamma)$, where the notation $\unique(\cdot)$ has been slightly abused, as $\Gamma$ and $\Sigma$ are subsets of different nature. 	
The following key result, which is reminiscent of \cite[Theorem 1]{choudhary2014identifiability}, makes the connection between PS-uniqueness of an EMF in $\Sigma = \idcolsupp{\Omega}$ and P-uniqueness of an EMD in $\Gamma = \Gamma_{\Omega}$. The proof is a direct consequence of \Cref{lemma:varphi-properties}.

\begin{lemma}
	\label{lemma:id-EMD}
	For any family of pairs of supports $\Omega$ and any pair of factors $(\mat{X}, \mat{Y})$:
	\begin{displaymath}
		(\mat{X}, \mat{Y}) \in \unique(\idcolsupp{\Omega}) \iff \varphi(\mat{X}, \mat{Y}) \in \unique(\Gamma_{\Omega}) \text{ and } (\mat{X}, \mat{Y}) \in \idcolsupp{\Omega}.
	\end{displaymath}
\end{lemma}

\begin{remark}
	In \Cref{lemma:id-from-outerproduct}, we saw that there exist ambiguities in the identification of the vectors $\vctor{a}$, $\vctor{b}$ from their outer product $\rankone{C}{} = \vctor{a} \transpose{\vctor{b}}$ when $\rankone{C}{} = \vctor{0}$. To remove this ambiguity, we can enforce a constraint ``$\vctor{a} \transpose{\vctor{b}} = \vctor{0} \iff \mat{a} = \vctor{0} \text{ \underline{and} } \vctor{b} = \vctor{0}$". This is precisely the role of $\idcolsupp{\Omega}$ in the lemma.
\end{remark}

Combining this lemma with \Cref{prop:non-degenerate-properties}, we obtain the following main theorem which summarizes the application of the lifting procedure to the analysis of identifiability in sparse matrix factorization.

\begin{theorem}
	\label{thm:equivalence-with-EMD}
	For any family of pairs of supports $\Omega$ stable by permutation, and any pair of factors $(\mat{X}, \mat{Y})$:
	\begin{displaymath}
		(\mat{X}, \mat{Y}) \in \unique(\Sigma_{\Omega}) \iff \varphi(\mat{X}, \mat{Y}) \in \unique(\Gamma_{\Omega}) \text{ and } (\mat{X}, \mat{Y}) \in \idcolsupp{\Omega} \cap \maxcolsupp{\Omega}.
	\end{displaymath}
\end{theorem}

Hence, it is sufficient to characterize $(\mat{X}, \mat{Y}) \in \idcolsupp{\Omega} \cap \maxcolsupp{\Omega}$ such that the $r$-tuple of rank-one contributions $\varphi(\mat{X}, \mat{Y})$ is the P-unique EMD of $\mat{Z} := \mat{X} \transpose{\mat{Y}}$ in $\Gamma_{\Omega}$. The characterization of $\unique(\Gamma_{\Omega})$ relies on an analogy with sparse linear inverse problem \cite{elad2010sparse}, in the spirit of  \Cref{prop:right-uniqueness-linear-inv-pb-matrix}: P-uniqueness of $\tuplerkone{C}$ in the EMD of $\mat{Z} := \mathcal{A}(\tuplerkone{C})$ in $\Gamma_{\Omega}$ is equivalent to the ability to successively identify the constraint supports $\{ \rankone{S}{i} \}_{i=1}^r$ on $\tuplerkone{C}$ among the family $\varphi(\Omega)$, and then to identify $\tuplerkone{C}$ after fixing these supports.

\begin{proposition}
	\label{prop:id-support-and-entries}
	For any family of pairs of supports $\Omega$ stable by permutation, we have $\tuplerkone{C} \in \unique(\Gamma_{\Omega})$ if, and only if, both of the following conditions are verified: 
	\begin{enumerate}[label=(\roman*)]
		\item For all $\tuplerkone{S} \in \varphi(\Omega)$ such that $\mathcal{A} (\tuplerkone{C})\in \mathcal{A}(\Gamma_{\tuplerkone{S}})$, there exists a permutation $\sigma \in \mathfrak{S}(\integerSet{r})$ such that: $\supp(\rankone{C}{i}) \subseteq \rankone{S}{\sigma(i)}$ for each $1 \leq i \leq r$.
		\item For all $\tuplerkone{S} \in \varphi(\Omega)$ such that $\tuplerkone{C} \in \Gamma_{\tuplerkone{S}}$, we have $\tuplerkone{C} \in \unique (\Gamma_{\tuplerkone{S}})$.
	\end{enumerate}
\end{proposition}

The proof is deferred to the  \cref{proof:prop-id-support-and-entries}.

The two following subsections deal respectively with some characterizations for identifiability of the constraint supports, in the sense of condition (i) in \Cref{prop:id-support-and-entries}, and for identifiability of an $r$-tuple of rank-one matrices with fixed rank-one constraint supports as in condition (ii).

We end this subsection by deriving the following corollary of \Cref{thm:equivalence-with-EMD} with fixed-support (whose proof is deferred to  \cref{proof:cor-equivalence-EMD-fixed-support}), which will soon turn out to be useful for the next steps. 
Specializing the definitions of $\idcolsupp{\Omega}$ and $\maxcolsupp{\Omega}$ from \eqref{eq:DefIC} and \eqref{eq:DefMC} to the case where $\Omega$ is reduced to a single pair of supports $\pair{S}$, we denote
\begin{align}
	\idcolsupp{\pair{S}} :=& \; \{ (\mat{X}, \mat{Y}) \in \Sigma_{\pair{S}} \; | \; \colsupp(\mat{X}) = \colsupp(\mat{Y})\}, \label{eq:idcolsupp-fixed}\\
	\maxcolsupp{\pair{S}} :=& \; \{ (\mat{X}, \mat{Y}) \in \Sigma_{\pair{S}} \; | \; \colsupp(\mat{X}) = \colsupp(\leftsupp{S}), \, \colsupp(\mat{Y}) = \colsupp(\rightsupp{S}) \}. \label{eq:maxcolsupp-fixed}
\end{align}

\begin{corollary}[Application of \Cref{thm:equivalence-with-EMD}]
	\label{cor:equivalence-EMD-fixed-support}
	For any pair of supports $\pair{S}$, and any pair of factors $(\mat{X}, \mat{Y})$, denoting $\tuplerkone{S} := \varphi(\pair{S})$:
	\begin{displaymath}
		(\mat{X}, \mat{Y}) \in \unique(\Sigma_{\pair{S}}) \iff
		\varphi(\mat{X}, \mat{Y}) \in \unique(\Gamma_{\tuplerkone{S}}) \text{ and } (\mat{X}, \mat{Y}) \in \idcolsupp{\pair{S}} \cap \maxcolsupp{\pair{S}}.
	\end{displaymath}
\end{corollary}

\subsection{Identifying the rank-one constraint supports among a family}
Inspired by \cite[Chapter 7]{le2016matrices}, we derive simple conditions implying condition (i) of \Cref{prop:id-support-and-entries}. Let us introduce the \emph{completion} of $\varphi(\Omega)$, i.e., the family $\varphi(\Omega)$ completed by all the $r$-tuples of rank-one supports that are included in an $r$-tuple of rank-one matrices in $\varphi(\Omega)$, or more formally
\begin{displaymath}
	\varphi^{\mathrm{cp}}(\Omega) := \bigcup_{\tuplerkone{S} \in \varphi(\Omega)} \left \{ \tuplerkone{S}' \in (\mathbb{B}^{m \times n})^r \; | \; \forall i=1,\dots,r, \, \rank(\rankone{S'}{i}) \leq 1, \, \rankone{S'}{i} \subseteq \rankone{S}{i} \right \},
\end{displaymath}
where $\rankone{S'}{i} \subseteq \rankone{S}{i}$ means the inclusion of the supports viewed as subsets of indices.
\begin{proposition}
	\label{prop:simple-SC-identify-supports}
	Consider $\tuplerkone{C} \in \Gamma_{\Omega}$ and $\mat{Z} := \mathcal{A}(\tuplerkone{C})$, and assume that:
	\begin{enumerate}[label=(\roman*)]
		\item for each $\tuplerkone{C'} \in \Gamma_{\Omega}$ such that  $\mathcal{A}(\tuplerkone{C'}) = \mat{Z}$, the supports $\{ \supp(\rankone{C'}{i}) \}_{i=1}^r$ are pairwise disjoint;
		
		\item all $\tuplerkone{S} \in \varphi^{\mathrm{cp}}(\Omega)$ such that $\left \{ \rankone{S}{i} \right \}_{i=1}^r$ is a partition of $\supp(\mat{Z})$, and such that the rank of $(\matindex{\mat{Z}}{k}{l})_{(k,l) \in \rankone{S}{i}}$ ($1 \leq i \leq r$) is at most one, are equivalent up to a permutation.
	\end{enumerate}
	Then, the supports $\left \{ \supp ( \rankone{C}{i} ) \right \}_{i=1}^r$ are identifiable in the sense of condition (i) in \Cref{prop:id-support-and-entries}.
\end{proposition}

The proof is deferred to  \cref{proof:prop-simple-SC-identify-supports}.

\begin{remark}
	The assumption (i) of the proposition can be verified for some examples of $\mat{Z}$ and $\Omega$, using a simple counting argument of the cardinality of the rank-one supports, see for instance \cite[Chapter 7, Lemma 4]{le2016matrices}. 
\end{remark}

\Cref{prop:simple-SC-identify-supports} can be used to show identifiability of the supports of sparse factors of well-known matrices. For instance, 
\cite[Chapter 7, Section 7.4]{le2016matrices} implicitly uses such conditions for the DFT matrix. In the companion paper \cite{LeonPart2}, we show that standard sparse factorizations of the DCT and DST matrices of type II also verify these sufficient conditions.

\subsection{Necessary conditions for fixed-support identifiability}
Condition (ii) of \Cref{prop:id-support-and-entries} involves identifiability given a fixed support. To characterize it,
we analyze next the set $\unique(\Gamma_{\tuplerkone{S}})$ for a given fixed $r$-tuple of rank-one supports $\tuplerkone{S}$, 
and provide necessary conditions for fixed-support identifiability.
By \Cref{thm:equivalence-with-EMD}, fixed-support identifiability in EMD is equivalent to fixed-support identifiability in EMF, up to non-degeneration properties. We can thus exploit the analysis of \cref{section:right-identifiability} for EMF when fixing one factor to derive necessary conditions of fixed-support identifiability in EMD.
In the following theorem, the conditions are derived by successively fixing the left factor and the right factor.

\begin{theorem}
	\label{thm:NC-fixed-support-id}
	Consider $\tuplerkone{S}$ an $r$-tuple of rank-one supports, $\tuplerkone{C} \in \unique(\Gamma_{\tuplerkone{S}})$. 	Then, the following conditions hold:
	\begin{enumerate}[label=(\roman*)]
		\item For all $ i,i' \in \integerSet{r}$, $i \neq i'$, if the rank-one contributions $\rankone{C}{i}$ and $\rankone{C}{i'}$ are nonzero with the same row span and/or the same column span, then the rank-one supports $\rankone{S}{i}$ and $\rankone{S}{i'}$ are disjoint.
		
		\item Consider $\pair{S} = (\leftsupp{S}, \rightsupp{S})$ a pair of supports such that $\tuplerkone{S} = \varphi(\pair{S})$, and $(\mat{X}, \mat{Y}) \in \Sigma_{\pair{S}}$ such that $\tuplerkone{C} = \varphi(\mat{X}, \mat{Y})$. Denote $\{ [1], \ldots, [K] \}$ (resp. $ \{ [1], \ldots, [U] \}$ \footnote{The use of the same notation $[\cdot]$ for both partitions is an abuse of notation, but it should be clear from the context that $[k]$ (for $k \in \integerSet{K}$) is an element of the partition $\{ [1], \ldots, [K] \}$, while $[u]$ (for $u \in \integerSet{U}$) is an element of the partition $\{ [1], \ldots, [U]\}$.}) a partition of $\colsupp(\mat{X})$ (resp. $\colsupp(\mat{Y})$) into equivalence classes defined by collinearity of \emph{nonzero} columns in $\mat{X}$ (resp. $\mat{Y}$). For each $k \in \integerSet{K}$ (resp. each $u \in \integerSet{U}$), denote $ \vctorseq{\vctor{x}}{k} $ (resp. $\vctorseq{\vctor{y}}{u}$) a representative of the nonzero collinear columns indexed by $[k]$ (resp. by $[u]$). For each $i \in \integerSet{m}$, $j \in \integerSet{n}$:
		\begin{itemize}
			\item the columns $\{  \vctorseq{\vctor{x}}{k} \; : \; k \in \integerSet{K}, \, \supp(\rightsupp{S}_{[k]}\one_{[k]})  \ni j\}$  are linearly independent; and
			\item the columns $\{ \vctorseq{\vctor{y}}{u} \; : \; u \in \integerSet{U}, \, \supp(\leftsupp{S}_{[u]}\one_{[u]}) \ni i\}$  are linearly independent.
		\end{itemize}
	\end{enumerate}
\end{theorem}

The proof is deferred to  \cref{proof:thm-NC-fixed-support-id}.

\begin{remark}
	As shown in  \cref{proof:thm-NC-fixed-support-id}, when $\tuplerkone{C} \in \unique(\Gamma_{\tuplerkone{S}})$, for any $(\mat{X}, \mat{Y}) \in \Sigma_{\pair{S}}$ such that $\tuplerkone{C} = \varphi(\mat{X}, \mat{Y})$, we have $\colsupp(\mat{X}) = \colsupp(\mat{Y}) = \colsupp(\leftsupp{S}) = \colsupp(\rightsupp{S})$, 
	so in condition (ii), the partitions $\{[1], \ldots, [K]\}$ and $\{[1], \ldots, [U]\}$ are independent of the choice of such a pair $(\mat{X}, \mat{Y})$. 
\end{remark}

\begin{remark}
	It is possible to design an algorithm to check if these necessary conditions are violated for a given $\tuplerkone{C} \in \Gamma_{\tuplerkone{S}}$. Such an algorithm needs to check linear independence between vectors, so its numerical implementation has to take into account conditioning issues. 
		In practice, to check whether a set of vectors is linearly independent, one can for instance compute the LU decomposition (when it exists) of the matrix whose columns are defined by the considered vectors. Such a decomposition requires $p q^2 - q^3 / 3$ flops for a matrix of size $p \times q$ with $p \geq q$ \cite{golub2013matrix}. 
		In total, for $(\mat{X}, \mat{Y})$ of size $m \times r$ and $n \times r$ with $r \leq \min(n, m)$, the algorithm would typically require $r(r-1) / 2 (n+m) + m (nr^2 - r^3 / 3) + n (m r^2 - r^3 / 3)$ flops in the worst case scenario.
\end{remark}

As illustrated on the following example, the necessary conditions for fixed-support identifiability in \Cref{thm:NC-fixed-support-id} are  also sufficient in some cases.
\begin{example}
	\label{example:NC-fixed-supp-id-is-sometime-SC}
	Consider $a, b, c, d, x, y \in \mathbb{C}$, such that $(a,b) \neq (0,0)$ and $(c,d) \neq (0,0)$. Define:
	\begin{displaymath}
		\rankone{S}{1} := \begin{pmatrix}
			1 & 1 & 0 \\ 1 & 1 & 0
		\end{pmatrix}, \quad \rankone{C}{1} := \begin{pmatrix}
			a & ax & 0 \\ b & by & 0
		\end{pmatrix}, \quad \rankone{S}{2} := \begin{pmatrix}
			0 & 1 & 1 \\ 0 & 1 & 1 
		\end{pmatrix}, \quad \rankone{C}{2} := \begin{pmatrix}
			0 & cy & c \\ 0 & dy & d
		\end{pmatrix}.
	\end{displaymath}
	As detailed in  \cref{proof:example-NC-fixed-supp-id-is-sometime-SC}, $(\rankone{C}{1}, \rankone{C}{2})$ is the PS-unique EMF of $\mat{Z} := \rankone{C}{1} + \rankone{C}{2}$ in $\Gamma_{\tuplerkone{S}}$ with $\tuplerkone{S} := (\rankone{S}{1}, \rankone{S}{2})$ if, and only if, conditions (i) and (ii) of \Cref{thm:NC-fixed-support-id} are verified, or equivalently, if the columns $\transpose{\begin{pmatrix}
		a & b
	\end{pmatrix}}$ and $\transpose{\begin{pmatrix}
		c & d
	\end{pmatrix}}$ are linearly independent.
\end{example}

For future work, one might take into account this kind of example in order to establish tighter conditions for fixed-support identifiability.

\subsection{Sufficient conditions for fixed-support identifiability}
\label{section:SC-fixed-support}

In complement to the previous necessary conditions, we now show some sufficient conditions for fixed-support identifiability.
This condition will be based on a notion of \emph{closability} in some bipartite graphs associated to ``observable entries'' in rank-one supports.

\subsubsection{Observable \emph{vs} missing entries, bipartite graphs, and matrix completion}
Given an $r$-tuple $\tuplerkone{S}$ of rank-one supports and $\tuplerkone{C} \in \Gamma_{\tuplerkone{S}}$, for any $i \in \integerSet{r}$, an entry is said to be \emph{missing} in $\rankone{C}{i}$ if its index $(k,l)$ belongs to the intersection $\rankone{S}{i} \cap \rankone{S}{j}$ for another $j \neq i$; otherwise, it is said to be \emph{observable}. The intuition behind these notions is that each observable entry is equal to the corresponding entry of the factorized matrix $\sum_{i=1}^r \rankone{C}{i}$, while missing entries need to be deduced, if possible, from observable ones using the low-rank constraints. 

Low-rank matrix completion problems without sparsity constraints  are naturally 
associated to properties of certain bipartite graphs \cite[Proposition 2.15]{kiraly2012combinatorial}. This fact was exploited to prove fixed-support identifiability of certain factorizations of the DFT \cite[Chapter 7]{le2016matrices} using completion operations
inside each of the rank-one contributions.
More precisely, it was established that when all the missing values inside each contribution can be completed without ambiguity from the observable ones using the rank-one constraint, the considered tuple of rank-one contributions is identifiable for the EMD with fixed support.
We extend these results by showing  a more refined sufficient condition for $\tuplerkone{C} \in \unique(\Gamma_{\tuplerkone{S}})$, which is basically: all the missing entries in the contributions $\{ \rankone{C}{i} \}_{i=1}^r$ can be recovered through \emph{iterative} and \emph{partial} completion operations based on the rank-one constraint, starting only from the knowledge of the observable entries. These completion operations include completion \emph{inside} each contribution, as well completion \emph{across} the contributions.

\paragraph{Notations for graphs}
A graph $\graph{G}$ is defined by a set of vertices $V$ and a set of edges $E \subseteq V^2$, and is denoted $\graph{G}(V, E)$. 
The set of vertices and the set of edges in a given graph $\graph{G}$ are denoted respectively $V(\graph{G})$ and $E(\graph{G})$.  
The complement of a graph $\graph{G}(V, E)$ is the graph $\graphcomplement{\graph{G}}= \graph{G}(V, V^2 \backslash E)$.
In a \emph{bipartite} graph $\graph{G}$, the set of vertices is partitioned into two sets of vertices $V_1$ and $V_2$, and its set of edges satisfies $E \subseteq V_1 \times V_2$. This will be denoted $\graph{G}(V_1, V_2, E)$. 
By convention, the elements in the first (resp. second) group of vertices $V_1$ (resp. $V_2$) will be called red (resp. blue) vertices.
Then, the set of red (resp. blue) vertices for a given bipartite graph $\graph{G}$ is denoted by $V_1(\graph{G})$ (resp. $V_2(\graph{G})$).
A bipartite graph $\graph{G}(V_1, V_2, E)$ with $V_1 \subseteq \integerSet{m}$, $V_2 \subseteq \integerSet{n}$ is a representation of a matrix support $\mat{S}$ of size $n \times m$, where the nonzero entries corresponds to the set $E$. In other words, $\mat{S}$ is the adjacency matrix of $\graph{G}$. 
The \emph{complete bipartite graph} $\graph{G}(V_1, V_2, V_1 \times V_2)$ is denoted $\completegraph{V_1, V_2}$. Given a bipartite graph $\graph{G}$, we denote  $\completegraph{\graph{G}}$ the corresponding \emph{completed bipartite graph} defined as $\completegraph{\graph{G}} := \completegraph{V_1(\graph{G}), V_2(\graph{G})}$. An $r$-tuple of graphs is denoted $\tuplegraph{G} := (\graphindex{G}{1}, \ldots, \graphindex{G}{r})$.
In the following, we consider only undirected graphs.

We define bipartite graphs whose adjacency matrices correspond to observable entries. 
\begin{definition}[Observable supports, observable bipartite graphs]
	\label{def:obs-bipartite-graph}
	Consider an $r$-tuple of rank-one supports $\tuplerkone{S} = (\rankone{S}{1}, \ldots, \rankone{S}{r})$.  
	Viewing each $\rankone{S}{i}$ as an index set, the \emph{observable supports} are the subsets:
	\begin{equation}
		\obssupport{\tuplerkone{S}}{i} := \rankone{S}{i} \backslash \Big( \bigcup_{j \neq i} \rankone{S}{j} \Big), \quad 1 \leq i \leq r.
	\end{equation}
	Viewing each $\rankone{S}{i}$ as a binary matrix of rank one,  the $i$-th \emph{observable bipartite graph} of $\tuplerkone{S}$ ($1 \leq i \leq r$) is the bipartite graph $\obsgraph{\tuplerkone{S}}{i} := (\rowsupp(\rankone{S}{i}), \colsupp(\rankone{S}{i}), \obssupport{\tuplerkone{S}}{i})$. The $r$-tuple of observable bipartite graphs $\obsgraph{\tuplerkone{S}}{i}$, $1 \leq i\leq r$, associated to $\tuplerkone{S}$ is denoted $\tupleobsgraph{\tuplerkone{S}}$.
\end{definition}

\begin{figure}[!htbp]
	\centering
	\includegraphics[width=0.8\textwidth]{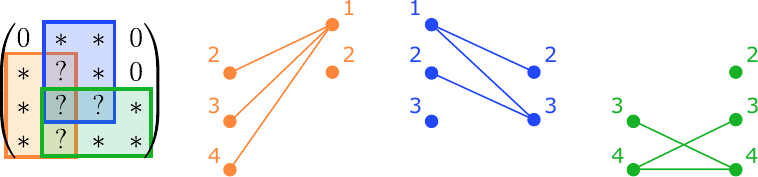}
	\caption{Illustration of the observable bipartite graphs of $\tuplerkone{S} := (\rankone{S}{1}, \rankone{S}{2}, \rankone{S}{3})$ where $\rankone{S}{1} = \{ 2, 3, 4\} \times \{ 1, 2\}$ (orange), $\rankone{S}{2} = \{ 1, 2, 3 \} \times \{ 2, 3\}$ (blue) and $\rankone{S}{3} = \{ 3, 4\} \times \{ 2, 3, 4\}$ (green). The indices marked with ``$*$" are those which belong to only one rank-one support (observable supports). The indices marked with ``$?$" are those which belong to at least two different rank-one supports.}
	\label{fig:observable_graph}
\end{figure}

This is illustrated on an example in \Cref{fig:observable_graph}. The role of bipartite graphs in characterizing fixed-support identifiability in EMD will become apparent once we recall an existing result in low-rank matrix completion.
For any matrix support $\mat{S} \in \mathbb{B}^{m \times n}$ interpreted as a binary mask, and any observed submatrix $(\matindex{\mat{Z}}{k}{l})_{(k, l) \in \mat{S}}$, define the \emph{rank-one matrix completion} problem as:
\begin{equation}
	\label{eq:completion-pb}
	\text{find, if possible, } \mat{M} \text{ of rank at most one such that } (\mat{M}_{ij})_{(i,j) \in \mat{S}} = (\matindex{\mat{Z}}{k}{l})_{(k, l) \in \mat{S}}.
\end{equation}
The next proposition from \cite{kiraly2012combinatorial}, illustrated by \Cref{fig:rankone-completion}, characterizes uniqueness in rank-one matrix completion.

\begin{figure}[!htbp]
	\centering
	\begin{subfigure}[b]{0.45\textwidth}
		\centering
		\includegraphics[width=0.9\textwidth]{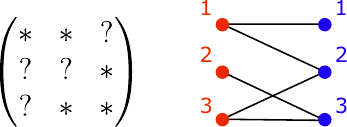}
		\caption{Connected bipartite graph.}
	\end{subfigure}
	\hfill
	\begin{subfigure}[b]{0.45\textwidth}
		\centering
		\includegraphics[width=0.9\textwidth]{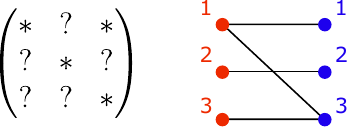}
		\caption{Not connected bipartite graph.}
	\end{subfigure}
	\caption{Illustration of \Cref{prop:connected-bipartite-graph}: on the left, it is possible to complete the missing values denoted by ``$?$" from the nonzero observed values ``$*$" in the mask, because the bipartite graph associated to this mask is connected. On the right, this condition is not verified.}
	\label{fig:rankone-completion}
\end{figure}

\begin{proposition}[Adapted from {\cite[Proposition 2.15]{kiraly2012combinatorial}}]
	\label{prop:connected-bipartite-graph}
	A rank-one matrix $\mat{A}$ of size $m \times n$ which has no zero entry is the unique solution of the completion problem \eqref{eq:completion-pb} for the observed submatrix $(\matindex{\mat{A}}{k}{l})_{(k, l) \in \mat{S}}$ on the mask $\mat{S}$ if, and only if, the bipartite graph $\graph{G}(\integerSet{m}, \integerSet{n}, \mat{S})$ is connected. 
\end{proposition}

\subsubsection{Completion operations for tuples of bipartite graphs}
The sufficient conditions for fixed-support identifiability that we will express rely on two completion operations  in the observable bipartite graphs $\tupleobsgraph{\tuplerkone{S}}$ that we now define. These operations are illustrated in \Cref{fig:completion-operations} for the tuple of graphs introduced in  \Cref{fig:observable_graph}.

The first completion operation is directly inspired from \Cref{prop:connected-bipartite-graph} and \cite[Chapter 7]{le2016matrices}. It corresponds to the completion of all missing edges in all connected subgraphs in each bipartite graph. In terms of completion of missing values in rank-one matrices, it corresponds to solving successively trivial linear systems with one variable. For instance, knowing that $\begin{pmatrix}
	a & x
\end{pmatrix}^\intercal$ and $\begin{pmatrix}
	b & c
\end{pmatrix}^\intercal$ are collinear, $x$ can be deduced from the values of $a$, $b$, and $c$, as long as $b$ is nonzero. 

The second completion operation involves all the bipartite graphs in the tuple.
In terms of completion of missing values in rank-one contributions, it corresponds to the fact that, for a given index pair $(k,l)$, when all the entries $\matindex{\rankone{C}{j}}{k}{l}$ ($j \in \integerSet{r}$) are known except for $j = i$, the entry $\matindex{\rankone{C}{i}}{k}{l}$ is easily deduced from the knowledge of $\matindex{\mat{Z}}{k}{l} := \sum_{j=1}^r \matindex{\rankone{C}{j}}{k}{l}$.

\begin{figure}[!htbp]
	\centering
	\begin{subfigure}[b]{\textwidth}
		\centering
		\includegraphics[width=0.7\textwidth]{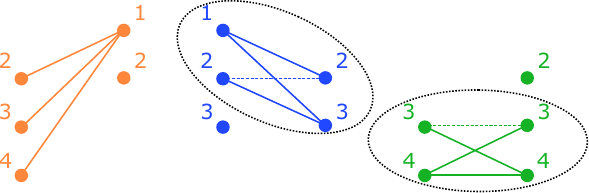}
		\caption{Completion \emph{inside} each graph. We look for all the subgraphs isomorphic to $\completegraph{2, 2}^-$ (circled by dashed lines), and we complete the edge that is missing in each of these subgraphs.}
	\end{subfigure}
	\vfill
	\begin{subfigure}[b]{\textwidth}
		\centering
		\includegraphics[width=0.7\textwidth]{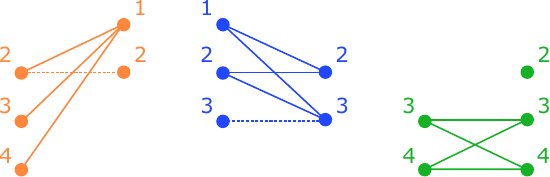}
		\caption{Completion \emph{across} graphs. The edge $(2, 2)$ belongs to the complete bipartite graphs $\completegraph{\graphindex{G}{1}}$ and $\completegraph{\graphindex{G}{2}}$, and it is missing in $\graphindex{G}{1}$ but not in $\graphindex{G}{2}$, so this edge is added to $\graphindex{G}{1}$ by the completion operation across graphs. Similarly, the edge $(3, 3)$ is added to $\graphindex{G}{2}$.}
	\end{subfigure}
	\caption{Illustration of completion operations defined by \Cref{def:completion-operations}. The tuple of bipartite graphs before completion $(\graphindex{G}{1}, \graphindex{G}{2}, \graphindex{G}{3})$, already introduced in  \Cref{fig:observable_graph},  is represented by full lines. Dashed lines represent edges that are added after the completion operation.}
	\label{fig:completion-operations}
\end{figure}

\begin{definition}[Completion operations]
	\label{def:completion-operations}
	Let $\tuplegraph{G} := (\graphindex{G}{i})_{i=1}^r$ be an $r$-tuple of bipartite graphs.
	Denote $\completegraph{2,2}^{-}$ the complete bipartite graph $\completegraph{\integerSet{2}, \integerSet{2}}$ minus one edge.
	\begin{itemize}
		\item {\bf Completion \emph{inside} each graph:} For each $i \in \integerSet{r}$, define $\graphindex{G'}{i}$ as the bipartite graph which is obtained from $\graphindex{G}{i}$ by completing all the missing edges in each subgraph of $\graphindex{G}{i}$ that is isomorphic to $\completegraph{2,2}^{-}$, and denote $a(\tuplegraph{G}) := (\graphindex{G'}{i})_{i=1}^r$.
		
		\item {\bf Completion \emph{across} graphs:} 
		If a given edge $e$ is missing only in the graph $\graphindex{G}{i}$ but not in the other graphs $\graphindex{G}{j}$ for $j \neq i$, we complete this missing edge in $\graphindex{G}{i}$. Formally, for each $i \in \integerSet{r}$, define $\graphindex{G''}{i}$ as the bipartite graph which is obtained from $\graphindex{G}{i}$ by adding all the edges in the set $
		\left \{ e \in E(\completegraph{\graphindex{G}{i}}) \; | \; e \notin  E(\graphindex{G}{i}) \; \text{and} \; e \in \bigcap_{\substack{j \neq i \\ e \in E(\completegraph{\graphindex{G}{j}})}} E(\graphindex{G}{j}) \right \}.$
	\end{itemize}
\end{definition}

\begin{remark}
	In the spirit of \Cref{prop:connected-bipartite-graph}, the completion operation inside each graph can be defined equivalently in the following way: ``For each $i \in \integerSet{r}$, define $\graphindex{G'}{i}$ as the bipartite graph which is obtained from $\graphindex{G}{i}$ by completing all the missing edges in each connected subgraph of $\graphindex{G}{i}$."
\end{remark}

The so-called \emph{closure} $\closure{\tuplegraph{G}}$ is then obtained by completing iteratively the missing edges in $\tuplegraph{G}$ through the completion operations $a$ and $b$ of \Cref{def:completion-operations}, until no more edges can be added. This process indeed reaches a fixed point after a finite number of steps.

\begin{lemma}\label{le:fixedpointcompl}
	Let $\tuplegraph{G}$ be an $r$-tuple of bipartite graphs. There is a positive integer $N$ such that $\iterate{(b \circ a)}{N+1}(\tuplegraph{G}) = \iterate{(b \circ a)}{N}(\tuplegraph{G})$, where $a$ and $b$ are the completion operations introduced in \Cref{def:completion-operations}.
\end{lemma}
The proof exploits a partial order between bipartite graphs sharing the same red and blue vertices, that will also be used in the proof of \Cref{thm:SC-fixed-support-id} below. 
\begin{definition}
	For any bipartite graphs $\graph{G}, \graph{H}$ for which $V_1(\graph{G}) = V_1(\graph{H})$, $V_2(\graph{G}) = V_2(\graph{H})$, we write $\graph{G} \preceq \graph{H}$ if $E(\graph{G}) \subseteq E(\graph{H})$.
	For any $r$-tuples of bipartite graphs $\tuplegraph{G}$, $\tuplegraph{H}$, we write:
	\begin{equation}
		\tuplegraph{G} \preceq \tuplegraph{H} \iff \forall i \in \integerSet{r}, \; \graphindex{G}{i} \preceq \graphindex{H}{i}. 
	\end{equation}
	One verifies that this partial order is indeed reflexive, anti-symmetric and transitive.
\end{definition}

\begin{proof}[Proof of~\Cref{le:fixedpointcompl}]
	First, we remark that $\tuplegraph{G} \preceq a(\tuplegraph{G}) \preceq (b \circ a)(\tuplegraph{G})$. Secondly, for any bipartite graph $\graph{H}$ with red vertices $V_1$ and blue vertices $V_2$, $\graph{H} \preceq \completegraph{V_1, V_2}$. But the number of edges in the complete bipartite graph $\completegraph{V_1, V_2}$ is finite. In conclusion, we obtain the claimed result by applying the monotone convergence theorem.
\end{proof}

\begin{definition}[Closure of a tuple of bipartite graphs]
	The \emph{closure} of an $r$-tuple of bipartite graphs $\tuplegraph{G}$ is the $r$-tuple of bipartite graphs $\closure{\tuplegraph{G}} := \iterate{(b \circ a)}{N(\tuplegraph{G})}(\tuplegraph{G})$ where
	\begin{equation}
		N(\tuplegraph{G}) := \min \; \{ N \in \mathbb{N} \; | \; \iterate{(b \circ a)}{N+1}(\tuplegraph{G}) = \iterate{(b \circ a)}{N}(\tuplegraph{G}) \}.
	\end{equation}
	
\end{definition}

\subsubsection{Sufficient condition for identifiability with fixed support}
We can now formulate our sufficient condition for identifiability with fixed support of $\tuplerkone{C} \in \unique(\Gamma_{\tuplerkone{S}})$.

\begin{definition}[Closable tuple of rank-one supports]
	\label{def:closable}
	We say that the $r$-tuple of rank-one supports $\tuplerkone{S}$ is \emph{closable} if the closure of its observable bipartite graphs, $\closure{\tupleobsgraph{\tuplerkone{S}}}$, is the $r$-tuple of complete bipartite graphs $(\completegraph{\obsgraph{\tuplerkone{S}}{i}})_{i=1}^r$.
\end{definition}

\begin{theorem}
	\label{thm:SC-fixed-support-id}
	Suppose that $\tuplerkone{S}$ is a closable $r$-tuple of rank-one supports. 
	Denote the set of $r$-tuples of rank-one matrices 
	with a support \emph{exactly} equal to $\tuplerkone{S}$ as:
	\begin{equation}
		\exact{\tuplerkone{S}} := \left \{ \tuplerkone{C} \in \Gamma_{\tuplerkone{S}} \; | \; \forall i \in \integerSet{r}, \; \supp(\rankone{C}{i}) = \rankone{S}{i} \right \}.
	\end{equation}
	Any $\tuplerkone{C} \in \exact{\tuplerkone{S}}$ is the P-unique EMD of $\mat{Z} := \mathcal{A}(\tuplerkone{C})$ in $\exact{\tuplerkone{S}}$. In other words, $\exact{\tuplerkone{S}} \subseteq \unique(\Gamma_{\tuplerkone{S}})$.
\end{theorem}

The proof is deferred to \cref{app:SC-fixed-support-id}.

It can be shown \cite[Chapter 7, Corollary 1]{le2016matrices} that any $r$-tuple of rank-one supports $\tuplerkone{S}$ such that $\rankone{S}{i} \cap \rankone{S}{i'} = \emptyset$ for all $i \neq i'$ is closable. As an immediate corollary we obtain P-uniqueness of the corresponding rank-one factors.
\begin{corollary}
	\label{cor:disjoint-support-EMD-fixed-support}
	Suppose that $\tuplerkone{S}$ is such that $\rankone{S}{i} \cap \rankone{S}{i'} = \emptyset$ for all $i \neq i'$. Then $\exact{\tuplerkone{S}} \subseteq \unique(\Gamma_{\tuplerkone{S}})$.
\end{corollary}

However, a closable $r$-tuple of rank-one supports $\tuplerkone{S}$ does not have necessarily pairwise disjoint rank-one supports. 
Consider for instance the tuple $\tuplerkone{S}$ of \Cref{fig:observable_graph}.
Even though its rank-one supports are not pairwise disjoint,  $\tuplerkone{S}$  will also be shown to be closable (see discussion around \Cref{fig:example-completion}). The following example illustrates an application of \Cref{thm:SC-fixed-support-id} for a given $\tuplerkone{C} \in \Gamma_{\tuplerkone{S}}$.

\begin{example}
	\label{ex:completion}
	Denote $\tuplerkone{S} := (\rankone{S}{1}, \rankone{S}{2}, \rankone{S}{3})$ as the tuple of supports of size $4 \times 4$ defined in \Cref{fig:observable_graph}, where $\rankone{S}{1} = \{ 2, 3, 4\} \times \{ 1, 2\}$, $\rankone{S}{2} = \{ 1, 2, 3 \} \times \{ 2, 3\}$, and $\rankone{S}{3} = \{ 3, 4\} \times \{ 2, 3, 4\}$. Define also:
	\begin{equation*}
		\mat{Z} = \begin{pmatrix}
			0 & 1 & 2 & 0 \\
			1 & 2 & 2 & 0 \\
			2 & 6 & 5 & 6 \\
			3 & 5 & 2 & 4
		\end{pmatrix}.
	\end{equation*}
	Then, by \Cref{thm:SC-fixed-support-id}, the matrix $\mat{Z}$ admits a P-unique EMD in $\Gamma_{\tuplerkone{S}}$, because any $\tuplerkone{C} \in \Gamma_{\tuplerkone{S}}$ such that $\mathcal{A}(\tuplerkone{C}) = \mat{Z}$ belongs to $\Gamma_{\tuplerkone{S}}^{(=)}$, and $\tuplerkone{S}$ is closable. 
	
\end{example}

\subsubsection{Algorithm based on rank-one matrix completion}

The closability condition of the previous section suggests an algorithm based on rank-one matrix completion to decompose, if possible, an observed matrix $\mat{Z}$ into a sum of rank-one matrices $\rankone{C}{i}$ ($1 \leq i \leq r$) satisfying the rank-one sparsity constraints associated to $\tuplerkone{S}$. 
By design, \Cref{algo:completion} greedily completes missing values only if the completion is non ambiguous. One the one hand, in the case when there exist some ambiguities during the completion, the algorithm returns a tuple of rank-one contributions with remaining missing values, and nothing can be said about the uniqueness of the exact matrix decomposition. 
On the other hand, the absence of missing values of the output indicates that the input $\mat{Z}$ admits a P-unique EMD with fixed supports constraint $\tuplerkone{S}$.

\begin{proposition}
	\label{prop:completion-algo}
	Let us run \Cref{algo:completion} with a given $r$-tuple of rank-one supports $\tuplerkone{S}$ and a given observed matrix $\mat{Z}$ as inputs. The following assertions hold:
	\begin{enumerate}[label=(\roman*)]
		\item If the algorithm breaks the loop because of incompatibility at \Cref{line:break-loop}, then $\mat{Z} \notin \mathcal{A}(\Gamma_{\tuplerkone{S}})$.
		\item If the algorithm does not break the loop because of incompatibility, and 
		outputs $\tuplerkone{C}$ without missing values $``?"$, then $\mat{Z} = \mathcal{A}(\tuplerkone{C})$. Moreover, $\tuplerkone{C}$ is the unique solution to $\mathcal{A}(\tuplerkone{C'}) = \mat{Z}$ such that $\tuplerkone{C'} \in \Gamma_{\tuplerkone{S}}$. A fortiori, $\tuplerkone{C} \in \unique(\Gamma_{\tuplerkone{S}})$.
	\end{enumerate}	
	
\end{proposition}

\begin{algorithm}[!t]
	\caption{Fixed-support exact matrix decomposition algorithm}
	\label{algo:completion}
	\begin{algorithmic}[1]
		\setcounter{ALC@unique}{0}
		\REQUIRE Observed matrix $\mat{Z} \in \mathbb{C}^{m \times n}$, $r$-tuple of rank-one matrices $\tuplerkone{S} = (\rankone{S}{i})_{i=1}^r$ 
		\ENSURE{Rank-one contributions $\tuplerkone{C}$, with possible missing values $``?"$}
		\STATE For $1 \leq i \leq r$, set $\matindex{\rankone{C}{i}}{k}{l} = \begin{cases}
			\mat{Z}_{k, l} & \text{if } (k, l) \in \obssupport{\tuplerkone{S}}{i} \\
			0 & \text{if } (k, l) \notin \rankone{S}{i} \\
			``?" & \text{otherwise}
		\end{cases}$
		\WHILE{completion of $``?"$ in $\tuplerkone{C}$ is still possible}
		\FOR[Completion \emph{inside} each bipartite graph]{$i=1$ \TO $r$}
		\FOR{each submatrix $\bigl[ \begin{smallmatrix} a_1 & a_2 \\ a_4 & a_3 \end{smallmatrix} \bigr]$ in $\rankone{C}{i}$}
		\IF{there exist three non missing entries and one missing entry $a_i = ``?"$}
		\STATE $i' \leftarrow (i + 2) \mod 4$
		\IF{$a_{i'} \neq 0$}
		\STATE $a_i \leftarrow \frac{a_k \times a_j}{a_{i'}}$ with $\{k, j\}$ the remaining indices in  $\integerSet{4} \backslash \{i, i'\}$
		\ENDIF
		\ENDIF
		\ENDFOR
		\ENDFOR
		\FOR[Checking compatibility of the completion]{$(k,l) \in \integerSet{m} \times \integerSet{n}$}
		\IF{$\matindex{\rankone{C}{i}}{k}{l} \neq ``?"$ for each $i \in \integerSet{r}$ and $\matindex{\mat{Z}}{k}{l} \neq \sum_{i} \matindex{\rankone{C}{i}}{k}{l}$}
		\RETURN\label{line:break-loop}{break the ``while" loop because of incompatibility}
		\ENDIF
		\ENDFOR
		
		\FOR[Completion \emph{accross} bipartite graphs]{$(k,l) \in \integerSet{m} \times \integerSet{n}$}
		\IF{there exists $i$ such that $\matindex{\rankone{C}{i}}{k}{l} = ``?"$ and $\matindex{\rankone{C}{j}}{k}{l} \neq ``?"$ for all $j \neq i$}
		\STATE $\matindex{\rankone{C}{i}}{k}{l} \leftarrow \matindex{\mat{Z}}{k}{l} - \sum_{j \neq i} \matindex{\rankone{C}{j}}{k}{l}$
		\ENDIF
		\ENDFOR			
		\ENDWHILE
		\RETURN{Rank-one contributions $\tuplerkone{C}$}
	\end{algorithmic}
\end{algorithm}

The proof is deferred to \Cref{proof:completion-algo}.
In other words, when \Cref{algo:completion}, run with $\mat{Z}$ and $\tuplerkone{S}$ as inputs, returns an output $\tuplerkone{C}$ without remaining missing values, $\tuplerkone{C}$ is guaranteed to be the P-unique EMD of $\mat{Z}$ in $\Gamma_{\tuplerkone{S}}$.
Moreover in this case  the PS-unique EMF of $\mat{Z}$ in $\pair{S}$ can be constructed from $\tuplerkone{C}$, where $\pair{S}$ is the unique pair of supports such that $\varphi(\pair{S}) = \tuplerkone{S}$. Indeed, one easily constructs $(\mat{X}, \mat{Y}) \in \idcolsupp{\pair{S}}$ such that $\matcol{\mat{X}}{i} \transpose{\matcol{\mat{Y}}{i}} = \rankone{C}{i}$.
We illustrate an application of this algorithm to the supports of \Cref{fig:observable_graph} / \Cref{ex:completion} in \Cref{fig:example-completion}.

\begin{figure}[!htbp]
	\centering
	\begin{subfigure}[b]{0.8\textwidth}
		\centering
		\includegraphics[width=\textwidth]{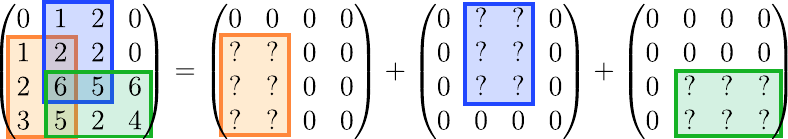}
		\caption{Initialization.}
	\end{subfigure}
	\vfill
	\begin{subfigure}[b]{0.8\textwidth}
		\centering
		\includegraphics[width=\textwidth]{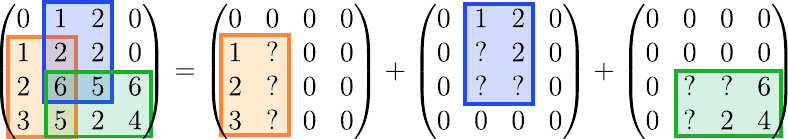}
		\caption{Filling values in observable supports.}
	\end{subfigure}
	\vfill
	\begin{subfigure}[b]{0.8\textwidth}
		\centering
		\includegraphics[width=\textwidth]{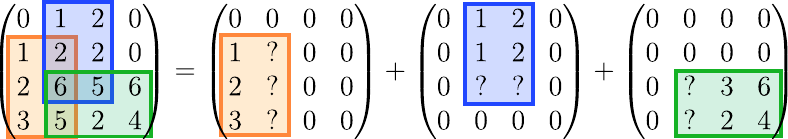}
		\caption{Completion operation inside each bipartite graph.}
	\end{subfigure}
	\vfill
	\begin{subfigure}[b]{0.8\textwidth}
		\centering
		\includegraphics[width=\textwidth]{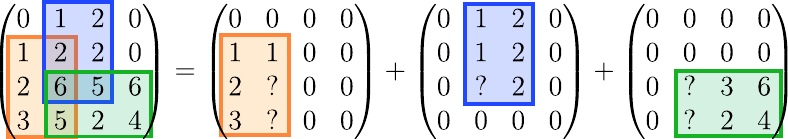}
		\caption{Completion operation across bipartite graphs.}
	\end{subfigure}
	\vfill
	\begin{subfigure}[b]{0.8\textwidth}
		\centering
		\includegraphics[width=\textwidth]{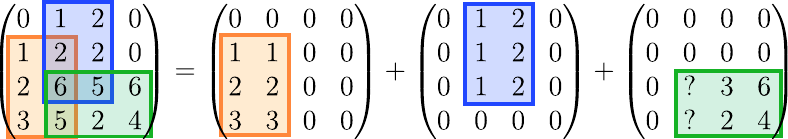}
		\caption{Completion operation inside each bipartite graph.}
	\end{subfigure}
	\vfill
	\begin{subfigure}[b]{0.8\textwidth}
		\centering
		\includegraphics[width=\textwidth]{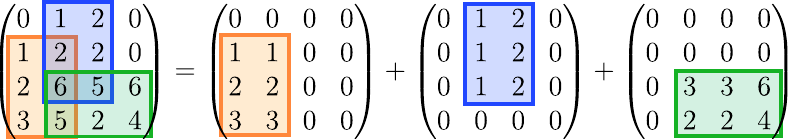}
		\caption{Completion operation across bipartite graphs and output.}
	\end{subfigure}
	\caption{Application of \Cref{algo:completion} on the supports of \Cref{fig:observable_graph} / \Cref{ex:completion}. We represent each contribution $\rankone{C}{i}$ at several steps of the completion algorithm. The output does not have remaining values, so by \Cref{prop:completion-algo}, the obtained tuple of rank-one contributions is identifiable in the EMD with fixed support.}
	\label{fig:example-completion}
\end{figure}

From \Cref{thm:SC-fixed-support-id}, we can deduce that one way to guarantee that \Cref{algo:completion} returns an output without remaining missing values is to enforce closability on the input tuple of rank-one supports $\tuplerkone{S}$. The proof is in essence a corollary of \Cref{thm:SC-fixed-support-id}.

\begin{lemma}
	Suppose that $\tuplerkone{S}$ is closable. Then, for any $\mat{Z} \in \mathcal{A}(\Gamma_{\tuplerkone{S}}^{(=)})$, the output of \Cref{algo:completion} with inputs $\tuplerkone{S}$ and $\mat{Z}$ is an $r$-tuple of rank-one contributions without remaining missing values.
\end{lemma}

\section{Uniform identifiability results}
\label{section:uniform-identifiability}
Uniform identifiability is a stronger property that requires identifiability of all instances $(\mat{X}, \mat{Y})$ that satisfy the sparsity constraints, from the observation $\mat{Z} := \mat{X} \transpose{\mat{Y}}$. They can be characterized by simple conditions, based on our previous analysis on instance identifiability in \cref{section:right-identifiability,section:EMD-identifiability}.

\subsection{Uniform P-uniqueness in EMD}

We start by stating a uniform identifiability result for fixed-support exact matrix decomposition. The proof is deferred to  \cref{proof:lemma-disjoint-support-characterizing-uniform-EMD}.

\begin{lemma}
	\label{lemma:disjoint-support-characterizing-uniform-EMD}
	Let $\tuplerkone{S}$ be an $r$-tuple of rank-one supports. Then, the following assertions are equivalent:
	\begin{enumerate}[label=(\roman*)]
		\item Uniform P-uniqueness of EMD in $\Gamma_{\tuplerkone{S}}$ holds, i.e., $\unique(\Gamma_{\tuplerkone{S}}) = \Gamma_{\tuplerkone{S}}$.
		\item The linear application $\mathcal{A}: \tuplerkone{C} \mapsto \sum_{i=1}^r \rankone{C}{i}$ restriced to $\Gamma_{\tuplerkone{S}}$ is injective.
		\item The rank-one supports $\{ \rankone{S}{i} \}_{i=1}^r$ are pairwise disjoint.
	\end{enumerate} 
\end{lemma}

Consequently, when a pair of supports $\pair{S} := (\leftsupp{S}, \rightsupp{S})$ is such that $\varphi(\pair{S})$ has disjoint rank-one supports, almost every pair of factors $(\mat{X}, \mat{Y})$ is identifiable for the EMF of $\mat{Z} = \mat{X} \transpose{\mat{Y}}$ in $\Sigma_{\pair{S}}$. In fact, a stronger identifiability property than PS-uniqueness is verified in this case, which is identifiability up to scaling ambiguities \emph{only}. 

\begin{definition}[S-uniqueness of an EMF in $\Sigma$]
	For any set $\Sigma$ of pairs of factors, the pair $(\mat{X}, \mat{Y}) \in \Sigma$ is the \emph{S-unique} EMF of $\mat{Z} := \mat{X} \transpose{\mat{Y}}$ in $\Sigma$, if any solution $(\mat{X'}, \mat{Y'})$ to \eqref{eq:EMF} with $\mat{Z}$ and $\Sigma$ is equivalent to $(\mat{X}, \mat{Y})$ up to scaling ambiguities only, written $(\mat{X'}, \mat{Y'}) \sim_s (\mat{X}, \mat{Y})$.
\end{definition}

The set of all pairs $(\mat{X}, \mat{Y}) \in \Sigma$ such that $(\mat{X}, \mat{Y})$  is the S-unique EMF of $ \mat{Z} := \mat{X} \transpose{\mat{Y}} \text{ in } \Sigma$ is denoted $\uniquescaling(\Sigma)$.

\begin{proposition}
	\label{prop:disjoint-rank-one-supports-unique-EMF}
	Suppose that $\pair{S} := (\leftsupp{S}, \rightsupp{S})$ is such that $\varphi(\pair{S})$ has disjoint rank-one supports. Then:
	\begin{displaymath}
		\uniquescaling(\Sigma_{\pair{S}}) = \idcolsupp{\pair{S}} \cap \maxcolsupp{\pair{S}},
	\end{displaymath}
	where we recall that $\idcolsupp{\pair{S}}$ and $\maxcolsupp{\pair{S}}$ has been defined by \eqref{eq:idcolsupp-fixed} and \eqref{eq:maxcolsupp-fixed}.
\end{proposition}

\begin{proof}
	Since $\uniquescaling(\Sigma_{\pair{S}}) \subseteq \unique(\Sigma_{\pair{S}})$, by \Cref{cor:equivalence-EMD-fixed-support}, $\uniquescaling(\Sigma_{\pair{S}}) \subseteq \idcolsupp{\pair{S}} \cap \maxcolsupp{\pair{S}}$. Conversely, consider $(\mat{X}, \mat{Y}) \in \idcolsupp{\pair{S}} \cap \maxcolsupp{\pair{S}}$, and let us show that $(\mat{X}, \mat{Y}) \in \unique(\Sigma_{\pair{S}})$. Let $(\mat{X'}, \mat{Y'}) \in \Sigma_{\pair{S}}$ such that $\mat{X} \transpose{\mat{Y}} = \mat{X'} \transpose{\mat{Y'}}$. Then, $\mathcal{A}(\varphi(\mat{X}, \mat{Y})) = \mathcal{A}(\varphi(\mat{X'}, \mat{Y'}))$. But $\varphi(\mat{X}, \mat{Y}), \varphi(\mat{X'}, \mat{Y'}) \in \Gamma_{\varphi(\pair{S})}$. This means, by \Cref{lemma:disjoint-support-characterizing-uniform-EMD}, that $\varphi(\mat{X}, \mat{Y}) = \varphi(\mat{X'}, \mat{Y'})$. As $\colsupp(\mat{X}) = \colsupp(\mat{Y})$, we conclude $(\mat{X}, \mat{Y}) \sim_s (\mat{X'}, \mat{Y'})$ by \Cref{lemma:varphi-properties}.
\end{proof}

In the companion paper \cite{LeonPart2}, such a condition of disjoint rank-one supports will be met in the analysis of identifiability in sparse matrix factorization with multiple factors, when constraining the factors to some well-chosen sparsity structure, like the so-called \emph{butterfly} structure appearing in some sparse factorization of the DFT or the Hadamard matrices \cite{dao2019learning}.
We now generalize \Cref{lemma:disjoint-support-characterizing-uniform-EMD} to the case where $\Omega$ is a general family of pairs of supports and not reduced to a singleton.

\begin{proposition}
	\label{prop:uniform-uniqueness-EMD}
	Uniform P-uniqueness of EMD in $\Gamma_{\Omega}$ holds, i.e., $\unique(\Gamma_{\Omega}) = \Gamma_{\Omega}$, if, and only if, the two following conditions are verified:
	\begin{enumerate}[label=(\roman*)]
		\item For all $\tuplerkone{S} \in \varphi(\Omega)$, the rank-one supports $\{ \rankone{S}{i} \}_{i=1}^r$ are pairwise disjoint.
		
		\item For all $\tuplerkone{S}, \tuplerkone{S'} \in \varphi^{\mathrm{cp}}(\Omega)$ such that $\bigcup_{i=1}^r \rankone{S}{i} = \bigcup_{i=1}^r \rankone{S'}{i}$, we have $\tuplerkone{S} \sim \tuplerkone{S}'$.
	\end{enumerate}
\end{proposition}

The proof is deferred to  \cref{proof:prop-uniform-uniqueness-EMD}.

Conditions for uniform P-uniqueness in EMD with sparsity constraints are too restrictive to be met in practice for classical choices of sparsity patterns mentionned in the examples of \cref{section:problem-formulation}. 
Nevertheless, we will see that \Cref{prop:uniform-uniqueness-EMD} will be useful for the characterization of uniform right identifiability in the next paragraph.

\subsection{Uniform right identifiability}
Uniform right identifiability for a given fixed left factor $\mat{X}$ and a given family of sparsity patterns for the right factors $\Theta$ is the equality $\unique(\{ \mat{X} \} \times \Sigma_{\Theta}) = \{ \mat{X} \} \times \Sigma_{\Theta}$. 
We extend the analysis of \cref{section:right-identifiability} to the uniform paradigm, and show at the end of the section that uniform right identifiability of classical family of sparsity patterns introduced in \cref{section:problem-formulation} is simply characterized by the Kruskal-rank of the fixed left factor. We recall that the Kruskal-rank \cite{KRUSKAL197795} of a matrix $\mat{X}$, denoted $\krank(\mat{X})$, is the largest number $j$ such that every set of $j$ columns in $\mat{X}$ is independent. 

\begin{corollary}[Application of \Cref{thm:right-uniqueness-perm}]
	\label{cor:uniform-right-uniqueness}
	Consider a fixed left factor of the form $\mat{X} = \begin{pmatrix}
		   \vctorseq{\vctor{x}}{1} \transpose{\one}_{[1]} & \ldots &  \vctorseq{\vctor{x}}{K} \transpose{\one}_{[k]} 
	\end{pmatrix}$ ($1 \leq K \leq r$), where $\vctorseq{\vctor{x}}{k} \neq 0$ for every $k$. Consider also a family of right supports $\Theta$. 
	Denote $\mat{\underline{X}} := ( \vctorseq{\vctor{x}}{k} )_{k \in \integerSet{K}} \in \mathbb{C}^{m \times K}$, with $[k]$ the set of indices of the columns of block $ \vctorseq{\vctor{x}}{K} \transpose{\one}$ ($1 \leq k \leq K$).
	Let $\tilde{\Theta}$ be the fingerprint of $\Theta$ on the partition $\{ [k]\}_{k=1}^K$ defined by \eqref{eq:fingerprint}, and
	$\Theta_{[k]}$ be the signature of $\Theta$ on $[k]$ defined by \eqref{eq:signature}.
	
	Then, uniform P-uniqueness of EMF in $\{ \mat{X} \} \times \Sigma_{\Theta}$, i.e., $\uniqueperm(\{ \mat{X} \} \times \Sigma_{\Theta}) = \{ \mat{X} \} \times \Sigma_{\Theta}$, holds if, and only if, the following conditions are verified:
	\begin{enumerate}[label=(\roman*)]
		\item The linear application $\mu_{\mat{\underline{X}}}: \Sigma_{\tilde{\Theta}} \to \mathbb{C}^{m \times n}, \; \mat{\underline{Y}} \mapsto \mat{\underline{X}} \transpose{ \mat{\underline{Y}} }$ is injective.
		\item For each $k \in \integerSet{K}$, uniform P-uniqueness of EMF in $\{ \transpose{\one}\} \times \Sigma_{\Theta_{[k]}}$ holds, i.e., $\uniqueperm(\{ \transpose{\one}\} \times \Sigma_{\Theta_{[k]}}) = \{ \transpose{\one}\} \times \Sigma_{\Theta_{[k]}}$.
	\end{enumerate}
\end{corollary}

\begin{proof}
	The two conditions are obtained by considering the two conditions of \Cref{thm:right-uniqueness-perm} for all $\mat{Y} \in \Sigma_{\Theta}$.
\end{proof}

Condition (i) in \Cref{cor:uniform-right-uniqueness} can be easily characterized by applying \cite[Proposition 1]{lu2008theory} to the linear application $\mu_{\mat{\underline{X}}}$ and the union of subspaces $\Sigma_{\tilde{\Theta}} = \bigcup_{\mat{S} \in \tilde{\Theta}} \Sigma_{\mat{S}}$.

\begin{proposition}[{Application of \cite[Proposition 1]{lu2008theory}}]
	\label{prop:uniform-right-uniqueness-inv-prob}
	For any family of right supports $\Theta \subseteq \mathbb{B}^{n \times r}$ viewed as subsets of indices, denote the set $\mathcal{I} := \left \{ \matcol{(\transpose{\mat{S}})}{k} \cup \matcol{(\transpose{\mat{S'}})}{k} \; | \; \mat{S}, \mat{S'} \in \Theta, \; k \in \integerSet{n} \right \}$. 
	For any left factor $\mat{X} \in \mathbb{C}^{m \times r}$, the linear application $\varphi_{\mat{X}}: \Sigma_{\Theta} \to \mathbb{C}^{m \times n}, \; \mat{Y} \mapsto \mat{X} \transpose{\mat{Y}}$ is injective if, and only if, the columns $\{ \matcol{\mat{X}}{l} \; | \; l \in I \}$ are linearly independent for each collection $I \in \mathcal{I}$.
\end{proposition}

\begin{proof}
	By \cite[Proposition 1]{lu2008theory}, the linear operator $\varphi_{\mat{X}}$ is invertible if, and only if, for every $\mat{S}, \mat{S'} \in \Theta$, the restriction of $\varphi_{\mat{X}}$  to the space $\Sigma_{\mat{S}} + \Sigma_{\mat{S'}} := \{ \mat{Y} + \mat{Y'} \; | \; \mat{Y} \in \Sigma_{\mat{S}}, \mat{Y'} \in \Sigma_{\mat{S}}\}$ is injective. 
	But one easily remarks that $\Sigma_{\mat{S}} + \Sigma_{\mat{S'}} = \Sigma_{\mat{S} \cup \mat{S'}}$, where $\mat{S} \cup \mat{S'}$ is an abuse of notation for $\supp(\mat{S} + \mat{S'})$ if $\mat{S}$ and $\mat{S'}$ are viewed as binary matrices. Then, injectivity of $\varphi_{\mat{X}}$ on $\Sigma_{\mat{S} \cup \mat{S'}}$ is verified if, and only if, for each row index $k \in \integerSet{n}$, the columns $\matcol{\mat{X}}{l}$ indexed by $l \in \matcol{(\transpose{\mat{S}} \cup \transpose{\mat{S'}})}{k} = \matcol{(\transpose{\mat{S}})}{k} \cup \matcol{(\transpose{\mat{S'}})}{k}$ are linearly independent.
\end{proof}

Condition (ii) can be characterized by applying \Cref{prop:uniform-uniqueness-EMD} to the case where we consider a family of $r$-tuples of rank-one supports of size $n \times 1$. Viewing the supports as subset of indices, denote the completion of $\Theta$ as:
\begin{equation*}
	\Theta^{\mathrm{cp}} := \bigcup_{\mat{S} \in \Theta} \{ \mat{S'} \in \mathbb{B}^{n \times r} \; | \; \mat{S'} \subseteq \mat{S} \}.
\end{equation*}

\begin{proposition}
	\label{prop:uniform-right-id-perm-vector-of-ones}
	Uniform P-uniqueness of EMF in $\{ \transpose{\one} \} \times \Sigma_{\Theta}$ holds if, and only if, the two following conditions are verified:
	\begin{enumerate}[label=(\roman*)]
		\item For each $\mat{S} \in \Theta$, the columns of $\mat{S}$ are pairwise disjoint.
		\item For all supports $\mat{S}, \mat{S'} \in \Theta^{\mathrm{cp}}$ such that $\bigcup_{i = 1}^r \mat{S}_{i} = \bigcup_{i = 1}^r \mat{S'}_{i}$, we have $\mat{S} \sim_p \mat{S'}$.
	\end{enumerate}
\end{proposition}

\begin{proof}
	We apply \Cref{prop:uniform-uniqueness-EMD} to the family $\Omega := \{ \transpose{\one} \} \times \Theta \subseteq \mathbb{B}^{1 \times r} \times \mathbb{B}^{n \times r}$. 
\end{proof}

Let us apply these results to obtain a simple characterization of uniform right identifiability of classical families of sparsity patterns of \cref{section:problem-formulation}. Recall the notations \eqref{eq:global-sparse}, \eqref{eq:col-sparse}, \eqref{eq:row-sparse}.

\begin{corollary}
	\label{cor:uniform-right-identifiability-classical-family}
	Let $\mat{X}$ be a left factor of size $m \times r$, and consider that right factors are of size $m \times r$. Let $\alpha \in \integerSet{r}$, $\beta \in \integerSet{n}$, and $s \in \integerSet{nr}$ be some sparsity parameters.
	\begin{enumerate}[label=(\roman*)]
		\item Suppose that $n \geq 2$ or $\alpha \geq 2$. Then, uniform PS-uniqueness of EMF in $\{ \mat{X} \} \times \Sigma_{\rowsparse{\alpha}}$ holds if, and only if, $\krank(\mat{X}) \geq \min(r, 2\alpha)$.
		\item Uniform PS-uniqueness in $\{ \mat{X} \} \times \Sigma_{\columnsparse{\beta}}$ holds if, and only if, $\krank(\mat{X}) = r$.
		\item Suppose that $\alpha \geq 2$, or suppose that $n \geq 2$ and $\beta \geq 2$. Then, uniform PS-uniqueness of EMF in $\{ \mat{X} \} \times \Sigma_{\rowsparse{\alpha} \cap \columnsparse{\beta}}$ holds if, and only if, $\krank(\mat{X}) \geq \min(r, 2\alpha)$.
		\item Suppose that $n \geq 2$ or $s \geq 2$. Then, uniform PS-uniqueness of EMF in $\{ \mat{X} \} \times \Sigma_{\globalsparse{s}}$ holds if, and only if, $\krank(\mat{X}) \geq \min(r, 2 s)$.
	\end{enumerate}
\end{corollary}

The proof is deferred to  \cref{proof:cor-uniform-right-identifiability-classical-family}.

These results generalize well-known results in the compressive sensing literature \cite[Theorem 2.13]{foucart_mathematical_2013}, in the case where permutation ambiguities are taken into account for uniqueness.

\section{Conclusion}
\label{section:conclusion}

In this work, we presented a general framework to study identifiability in exact matrix factorization into two factors, when considering arbitrary sparsity constraints. When sparsity constraints are encoded by a family of pairs of supports stable by permutation of columns, our framework takes into account these permutation ambiguities (in addition to the inherent scaling ambiguities) to study uniqueness of exact sparse factorization. Our analysis of identifiability relies on the lifting procedure via the matrix decomposition approach into rank-one matrices. We now discuss some important perspectives of this work.

\paragraph{Identifiability of the left factor without fixing the right one} The characterization of condition (i) in \Cref{prop:identifiability-left-and-right} can be explored as a complementary approach to the lifting approach proposed in this work, in continuation of the work proposed in \cite{li2015unified}.
They originally introduced this approach to characterize identifiabiliy in the blind gain and phase calibration problem with sparsity and subspace constraints, which can be seen as an instance of the matrix factorization problem into two structured factors.

\paragraph{Fixed-support identifiability}
A first possible improvement of this work is to better understand fixed-support identifiability, as the necessary condition (\Cref{thm:NC-fixed-support-id}) and sufficient condition (\Cref{thm:SC-fixed-support-id})
given in this paper are not tight. Having a better understanding of fixed-support identifiability would then allow to establish tighter sufficient conditions for identifiability of the supports than \Cref{prop:simple-SC-identify-supports}, which was specific to the case where any sparse EMD of the matrix to factorize has disjoint rank-one contributions.

\paragraph{Extension to the multi-layer case}
The companion paper \cite{LeonPart2} provides an application of the presented general framework presented in this section to show some identifiability results for the {\em multi-layer} sparse matrix factorization of the Hadamard or DFT matrices, based on a hierarchical factorization method \cite{le2016flexible, le2016matrices}.
For instance, in the case $J=3$, when enforcing a sparsity constraint $\Lambda_i$ on the $i$-th factor ($1 \leq i \leq 3$), we can consider, by the hierarchical factorization method \cite{le2016matrices}, the family of pairs of supports $\Omega = \Lambda_3 \times \Theta_2$ where $\Theta_2 := \{ \transpose{ \supp(\mat{X}^{(2)} \matseq{\mat{X}}{1}) } \; | \; \supp(\mat{X}^{(i)}) \in \Lambda_i, \; i=1,2 \}$, so that the analysis of identifiability in the factorization $\mat{Z} = \mat{X}^{(3)} \transpose{\mat{Y}}$ with two factors $(\mat{X}^{(3)}, \mat{Y}) \in \Sigma_{\Omega}$ proposed by our framework can provide insights about the one with three factors $\mat{Z} = \matseq{\mat{X}}{3} \matseq{\mat{X}}{2} \matseq{\mat{X}}{1}$, or more. 

\paragraph{Algorithm}
Our analysis suggests other approaches for sparse matrix factorization than iterative first-order optimization methods \cite{le2016flexible, dao2019learning}. 
One can rely on low-rank matrix completion operations with sparsity constraints, in the spirit of \Cref{algo:completion}. This kind of approach based on the lifting procedure has been considered in algorithms for blind deconvolution \cite{ahmed2013blind} or signal recovery from magnitude measurements \cite{candes2013phaselift}. However, it still remains to explore such an approach when considering sparsity constraints.

\paragraph{Stability}
This work focused on identifiability aspects in exact sparse matrix factorization, which are necessary to study in order to understand well-posedness of the sparse matrix factorization problem. 
The other important condition for well-posedness is \emph{stability}, which is the ability to recover the solution to the sparse factorization problem under noisy measurements of the observed matrix. To study these aspects, one can rely on stability results in blind deconvolution with sparsity constraints \cite{li2017identifiability}, which are also derived from the lifting procedure. More related to the sparse matrix factorization problem, stability in deep structured linear networks under sparsity constraints has been studied with the tensorial lifting approach \cite{malgouyres2020stable}, but in contrary to our framework, permutation ambiguities were not taken into account in that work. As our approach relies on matrix decomposition into rank-one matrices, one perspective in continuation to our work is to exploit existing stability results on rank-one matrix completability \cite{candes2010matrix, cosse2020stable} to study stability in sparse matrix factorization. 
Typically, the application of \Cref{algo:completion} under the noisy case might suffer from instability issues, as the propagation scheme to complete missing entries with the rank-one constraint is unstable \cite{cosse2020stable}, so it is necessary to adapt \Cref{algo:completion} in order to allow for stable recovery of exact matrix decomposition with fixed-support.

\section*{Acknowledgments}
The authors are grateful to Valentin Emiya, Jovial Cheukam, Luc Giffon and Quoc-Tung Le for
several important discussions which were helpful for this work.

\bibliographystyle{siamplain}
\bibliography{references}


%% file: 1_supplement.tex
%
%
%
%
%

\section{Proof of \Cref{lemma:maximal-colsupp}}
\label{proof:lemma-maximal-colsupp}

\begin{proof}
	We prove the contraposition. Let $(\mat{X}, \mat{Y}) \in \Sigma_{\Omega}$ such that there exists $\pair{S} \in \Omega$ verifying $(\mat{X}, \mat{Y}) \in \Sigma_{\pair{S}}$, with $\colsupp(\mat{X}) \neq \colsupp(\leftsupp{S})$ or $\colsupp(\mat{Y}) \neq \colsupp(\rightsupp{S})$. Up to matrix transposition,  we can suppose $\colsupp(\mat{X}) \neq \colsupp(\leftsupp{S})$.
	By \Cref{lemma:identical-colsupp}, we can also assume without loss of generality that $\colsupp(\mat{X}) = \colsupp(\mat{Y})$.
	Suppose that $\colsupp(\rightsupp{S}) \not\subseteq \colsupp(\leftsupp{S})$. Then, we can fix  $i \in \integerSet{r}$ such that $\matcol{\leftsupp{S}}{i} = \vctor{0}$ and $\matcol{\rightsupp{S}}{i} \neq \vctor{0}$. This means that $\matcol{\mat{X}}{i} = \matcol{\mat{Y}}{i} = \vctor{0}$. Setting $\mat{Y'} \in \Sigma_{\rightsupp{S}}$ such that $\matcol{\mat{Y'}}{i} = \matcol{\rightsupp{S}}{i}$ and $\matcol{\mat{Y'}}{j} = \matcol{\mat{Y}}{j}$ for all $j \neq i$, we build an instance as in 
	\Cref{lemma:on-met-ce-qu-on-veut-a-droite} with $\Sigma = \Sigma_{\pair{S}}$, to show that $(\mat{X}, \mat{Y}) \notin \unique(\Sigma_{\pair{S}})$. The reasoning is symmetric for the case where $\colsupp(\leftsupp{S}) \not\subseteq \colsupp(\rightsupp{S})$. It remains the case where $\colsupp(\leftsupp{S}) = \colsupp(\rightsupp{S})$. Let us now fix $i \in \colsupp(\leftsupp{S}) \backslash \colsupp(\mat{X})$. Then, $\matcol{\leftsupp{S}}{i} \neq 0$, $\matcol{\rightsupp{S}}{i} \neq 0$, and $\matcol{\mat{X}}{i} = \matcol{\mat{Y}}{i} = \vctor{0}$. Again, construct $\mat{Y'} \in \Sigma_{\rightsupp{S}}$ with $\matcol{\mat{Y'}}{i} = \matcol{\rightsupp{S}}{i}$ and $\matcol{\mat{Y'}}{j} = \matcol{\mat{Y}}{j}$ for all $j \neq i$, and we obtain an instance as in \Cref{lemma:on-met-ce-qu-on-veut-a-droite}
	with $\Sigma = \Sigma_{\pair{S}}$, showing that $(\mat{X}, \mat{Y}) \notin \unique(\Sigma_{\pair{S}})$.
\end{proof}

\section{Proof of \Cref{prop:non-degenerate-properties}}
\label{proof:prop-non-degenerate-properties}

\begin{proof}
	The direct inclusion is immediate by applying \Cref{lemma:nec-conditions-EMF,lemma:identical-colsupp,lemma:maximal-colsupp}. 		
	For the inverse inclusion, let $(\mat{X^*}, \mat{Y^*}) \in \unique (\idcolsupp{\Omega}) \cap \maxcolsupp{\Omega}$, and $(\mat{X}, \mat{Y}) \in \Sigma_{\Omega}$ such that $\mat{X} \transpose{\mat{Y}} = \mat{X^*} \transpose{\mat{Y^*}}$. The goal is to show $(\mat{X}, \mat{Y}) \sim (\mat{X^*}, \mat{Y^*})$.
	Fix $\pair{S} \in \Omega$ such that $(\mat{X}, \mat{Y}) \in \Sigma_{\pair{S}}$.
	Denote $J = \colsupp(\mat{X}) \cap \colsupp(\mat{Y})$. 
	Define $(\mat{X'}, \mat{Y'}) \in \idcolsupp{\Omega}$ such that $(\submat{\mat{X'}}{J}, \submat{\mat{Y'}}{J}) = (\submat{\mat{X}}{J}, \submat{\mat{Y}}{J})$ and $(\matcol{\mat{X'}}{i}, \matcol{\mat{Y'}}{i}) = (\vctor{0}, \vctor{0})$ for $i \notin J$. 
	Since $\mat{X'} \transpose{\mat{Y'}} = \mat{X} \transpose{\mat{Y}} = \mat{X^*} \transpose{\mat{Y^*}}$, and $(\mat{X^*}, \mat{Y^*}) \in \unique(\idcolsupp{\Omega})$, we have $(\mat{X'}, \mat{Y'}) \sim (\mat{X^*}, \mat{Y^*})$. Fix $\mat{G} \in \genperm{r}$ such that $\mat{X^*} = \mat{X'} \mat{G}$ and $\transpose{\mat{Y^*}} = \mat{G}^{-1} \transpose{\mat{Y'}}$, and denote $\mat{P} := \supp(\mat{G})$.
	Then, $\supp(\mat{X^*}) = \supp(\mat{X'} \mat{G}) = \supp(\mat{X'}) \mat{P} \subseteq \leftsupp{S} \mat{P}$, and similarly, $\supp(\mat{Y^*}) \subseteq \rightsupp{S} \mat{P}$. By stability of $\Omega$, $(\leftsupp{S} \mat{P}, \rightsupp{S} \mat{P}) \in \Omega$. Therefore, since $(\mat{X^*}, \mat{Y^*}) \in \idcolsupp{\Omega} \cap \maxcolsupp{\Omega}$, we have $\colsupp(\leftsupp{S} \mat{P}) = \colsupp(\mat{X^*}) = \colsupp(\mat{Y^*}) = \colsupp(\rightsupp{S} \mat{P})$. But $\card(\colsupp(\mat{X'})) = \card(\colsupp(\mat{X^*}))$, because $(\mat{X'}, \mat{Y'}) \sim (\mat{X^*}, \mat{Y^*})$. 
	As $J = \colsupp(\mat{X'})$, we obtain 
	\begin{equation*}
		\begin{split}
			\card(J) = \card(\colsupp(\mat{X'})) &= \card(\colsupp(\mat{X^*})) \\
			&=  \card(\colsupp(\leftsupp{S} \mat{P})) = \card(\colsupp(\leftsupp{S})).
		\end{split}
	\end{equation*}
	Therefore, $J = \colsupp(\leftsupp{S})$, as $J \subseteq \colsupp(\leftsupp{S})$. This means $\mat{X'} = \mat{X}$. Similarly, we show that $\mat{Y'} = \mat{Y}$. In conclusion, $(\mat{X}, \mat{Y}) = (\mat{X'}, \mat{Y'}) \sim (\mat{X^*}, \mat{Y^*})$.
\end{proof}

\section{Proof of \Cref{lemma:eliminating-zero-col-right-id}}
\label{proof:lemma-eliminating-zero-col-right-id}

\begin{proof}
	Suppose $(\submat{\mat{X}}{J}, \submat{\mat{Y}}{J}) \in \unique(\{ \submat{\mat{X}}{J} \}\times \Sigma_{\signature{\Theta}{J}})$. Let $\mat{B} \in \Sigma_{\Theta}$ such that $\mat{X} \transpose{\mat{Y}} = \mat{X} \transpose{\mat{B}}$. This means $\submat{\mat{X}}{J} \transpose{\submat{\mat{Y}}{J}} = \submat{\mat{X}}{J} \transpose{\submat{\mat{B}}{J}}$. 
	By definition of the signature, $\submat{\mat{B}}{J} \in 
	\Sigma_{\signature{\Theta}{J}}$. Therefore, by assumption, $(\submat{\mat{X}}{J}, \submat{\mat{B}}{J}) \sim (\submat{\mat{X}}{J}, \submat{\mat{Y}}{J})$. 
	As the columns of $\mat{X}$, $\mat{Y}$ and $\mat{B}$ indexed by $\integerSet{r} \backslash J$ are all zero columns, we obtain $(\mat{X}, \mat{Y}) \sim (\mat{X}, \mat{B})$. This shows $(\mat{X}, \mat{Y}) \in \unique(\{ \mat{X} \} \times \Sigma_{\Theta})$.
	
	Conversely, suppose $(\mat{X}, \mat{Y}) \in \unique(\{ \mat{X} \} \times \Sigma_{\Theta})$. Let $\mat{B'} \in \Sigma_{\signature{\Theta}{J}}$ such that $\submat{\mat{X}}{J} \transpose{\submat{\mat{Y}}{J}} = \submat{\mat{X}}{J} \transpose{\mat{B'}}$. 
	Define $\mat{B}$ such that $\submat{\mat{B}}{J} = \mat{B'}$ and 	$\submat{\mat{B}}{\integerSet{r} \backslash J} = \mat{0}$. By definition of the signature, $\mat{B} \in \Sigma_{\Theta}$.
	Since $\colsupp(\mat{Y}) \subseteq J$, $\submat{\mat{X}}{J} \transpose{\submat{\mat{Y}}{J}} = \mat{X} \transpose{\mat{Y}}$. Similarly, since $\colsupp(\mat{B}) \subseteq J$, $\submat{\mat{X}}{J} \transpose{\mat{B'}} = \mat{X} \transpose{\mat{B}}$. Hence, $\mat{X} \transpose{\mat{Y}} = \mat{X} \transpose{\mat{B}}$, and by assumption, $(\mat{X}, \mat{B}) \sim (\mat{X}, \mat{Y})$. Fix $\mat{G} \in \mathcal{G}(\mat{X})$ such that $\transpose{\mat{B}} = \inverse{\mat{G}} \transpose{\mat{Y}}$. 
	The permutation matrix $\supp(\mat{G})$ only permutes nonzero columns of $\mat{X}$ with nonzero columns of $\mat{X}$. In other words, $J$ is stable by the permutation induced by $\supp(\mat{G})$. This means that the submatrix $\mat{G'} := (\matindex{\mat{G}}{k}{l})_{(k,l) \in J^2}$ verifies $\mat{G'} \in \mathcal{G}(\submat{\mat{X}}{J})$, and $\transpose{\mat{B'}} = \inverse{\mat{G'}} \transpose{\submat{\mat{Y}}{J}}$, which shows $(\submat{\mat{X}}{J}, \mat{B'}) \sim (\submat{\mat{X}}{J}, \submat{\mat{Y}}{J})$.
\end{proof}

\section{Proof of \Cref{thm:right-uniqueness-perm}}
\label{proof:right-uniqueness-perm}

We rely on the following lemma.

\begin{lemma}	
	\label{lemma:P-X-block-structure}
	Consider $\mat{X} = \begin{pmatrix}
		\vctorseq{\vctor{x}}{1} \transpose{\one}_{\card([1])} & \ldots &  \vctorseq{\vctor{x}}{K} \transpose{\one}_{\card([K])} 
	\end{pmatrix}$, where the sets of indices $[1], \ldots, [K]$ form a partition of $\integerSet{r}$. Then, a permutation matrix $\mat{P}$ leaves $\mat{X}$ invariant, in the sense that $\mat{P} \in \mathcal{P}(\mat{X})$, if, and only if, it is a product of $K$ permutation matrices $\mat{P}^{(k)}$ ($1 \leq k \leq K$) which leave respectively the block $\submat{\mat{X}}{[k]}$ invariant, in the sense that $\mat{P}^{(k)} \in \mathcal{P}(\submat{\mat{X}}{[k]})$.
\end{lemma}

\begin{proof}
	If $\mat{P} \in \mathcal{P}(\mat{X})$, then $\mat{P}$ cannot permute a column indexed by $i \in [k]$ to a column indexed by $i' \in [k']$ with $k' \neq k$, as the columns $\{  \vctorseq{\vctor{x}}{K} \}_{k=1}^K$ are all different. 
\end{proof}

\begin{remark}
	\label{eq:XY=XY'-iff-tXtY=tXtY'}
	Denoting $\mat{\underline{Y}'} := \begin{pmatrix}
		\submat{\mat{Y'}}{[1]} \one_{[1]} & \ldots & \submat{\mat{Y'}}{[K]} \one_{[K]}
	\end{pmatrix}$ for any right factor $\mat{Y'}$ with $r$ columns, we have: $\mat{X} \transpose{\mat{Y}} = \mat{X} \transpose{\mat{Y'}} \iff 
	\mat{\underline{X}} \transpose{ \mat{\underline{Y}} } =
	\mat{\underline{X}} 
	\transpose{ \mat{\underline{Y}'} }$.
\end{remark}

\begin{proof}[Proof of \Cref{thm:right-uniqueness-perm}]
	For sufficiency, suppose that conditions (i) and (ii) are verified, and let $\mat{Y'} \in \Sigma_{\Theta}$ such that $\mat{X} \transpose{\mat{Y'}} = \mat{X} \transpose{\mat{Y}}$. The goal is to show $(\mat{X}, \mat{Y'}) \sim (\mat{X}, \mat{Y})$.
	Denote $\mat{\underline{Y}'} := \begin{pmatrix}
		\submat{\mat{Y'}}{[1]} \one_{[1]} & \ldots & \submat{\mat{Y'}}{[K]} \one_{[K]}
	\end{pmatrix} \in \Sigma_{\tilde{\Theta}}$, where $\tilde{\Theta}$ is the fingerprint of $\Theta$ on $\{ [k] \}_{k=1}^K$. 
	Then, by \eqref{eq:XY=XY'-iff-tXtY=tXtY'}, $\mat{\underline{X}} \transpose{ \mat{\underline{Y}'}} = \mat{\underline{X}} \transpose{ \mat{\underline{Y}}}$. 
	By condition (i), $\mat{\underline{Y}'} = \mat{\underline{Y}}$. By fixing $k \in \integerSet{K}$, this implies  $\submat{\mat{Y'}}{[k]} \one_{[k]} = \submat{\mat{Y}}{[k]} \one_{[k]}$. 
	By condition (ii), there exists a permutation matrix $\mat{P}^{(k)}$ such that $\submat{\mat{Y'}}{[k]}  = \submat{\mat{Y}}{[k]} \mat{P}^{(k)}$. This is true for each $k$, so by \Cref{lemma:P-X-block-structure}, there exists $\mat{P} \in \mathcal{P}(\mat{X})$ such that $\mat{Y'} = \mat{Y} \mat{P}$, hence we conclude that $(\mat{X}, \mat{Y'}) \sim (\mat{X}, \mat{Y})$. 
	
	Suppose that condition (ii) is not verified. Fix $k \in \integerSet{K}$, and $\matseq{\mat{Y}}{k} \in \Sigma_{\signature{\Theta}{[k]}}$ such that $\submat{\mat{Y}}{[k]} \one_{[k]} = \matseq{\mat{Y}}{k} \one_{[k]}$, but $\matseq{\mat{Y}}{k} \neq \submat{\mat{Y}}{[k]} \mat{P}^{(k)}$ for each permutation matrix $\mat{P}^{(k)} \in \mathcal{P}(\submat{\mat{X}}{[k]})$. 
	Define $\mat{Y'} \in \Sigma_{\Theta}$ such that $\submat{\mat{Y'}}{[k]}  = \matseq{\mat{Y}}{k}$ and $\submat{\mat{Y'}}{[k']} = \submat{\mat{Y}}{[k']}$ for $k' \neq k$. 
	Then, by \eqref{eq:XY=XY'-iff-tXtY=tXtY'}, we obtain $\mat{X} \transpose{\mat{Y}} = \mat{X} \transpose{\mat{Y'}}$.
	Since $\matseq{\mat{Y}}{k} \neq \submat{\mat{Y}}{[k]} \mat{P}^{(k)}$ for each permutation matrix $\mat{P}^{(k)} \in \mathcal{P}(\submat{\mat{X}}{[k]})$, applying \Cref{lemma:P-X-block-structure} yields $\mat{Y'} \neq \mat{Y} \mat{P}$ for each $\mat{P} \in \mathcal{P}(\mat{X})$. 
	In conclusion, $(\mat{X}, \mat{Y'}) \not\sim_p (\mat{X}, \mat{Y})$, and $(\mat{X}, \mat{Y}) \notin \uniqueperm( \{ \mat{X} \} \times \Sigma_{\Theta})$.
	
	For necessity of condition (i), suppose that $(\mat{X}, \mat{Y}) \in \uniqueperm (\{ \mat{X} \} \times \Sigma_{\Theta})$, and let $\mat{\underline{Y}'} \in \Sigma_{\tilde{\Theta}}$ such that $\mat{\underline{X}} \transpose{\mat{\underline{Y}'}} =  \mat{\underline{X}} \transpose{ \mat{\underline{Y}}}$. We want to show $\mat{\underline{Y}'} = \mat{\underline{Y}}$.
	Fix $\mat{\underline{S}} \in \tilde{\Theta}$ such that $\mat{\underline{Y}'} \in \Sigma_{\mat{\underline{S}}}$. 
	By definition of the fingerprint, there exists $\mat{S} \in \Theta$ such that $\bigcup_{i \in [k]}{\matcol{\mat{S}}{i}} = \matcol{\mat{\underline{S}}}{k}$ ($1 \leq k \leq K$).
	Fix $k \in \integerSet{K}$. 
	Since $\supp(\matcol{\mat{\underline{Y}'}}{k}) \subseteq \matcol{\mat{\underline{S}}}{k} = \bigcup_{i \in [k]}{\matcol{\mat{S}}{i}}$, there exists $\matseq{\mat{Y}}{k} \in \Sigma_{\submat{\mat{S}}{[k]}}$ such that $\matcol{\mat{\underline{Y}'}}{k} = \matseq{\mat{Y}}{k} \one_{[k]}$. 
	Construct then $\mat{Y'} \in \Sigma_{\mat{S}}$ such that $\submat{\mat{Y'}}{[k]}  = \matseq{\mat{Y}}{k}$ for each $k \in \integerSet{K}$. Hence, $\mat{\underline{Y}'} = \begin{pmatrix}
		\matcol{\mat{\underline{Y}'}}{1} & \ldots & \matcol{\mat{\underline{Y}'}}{K}
	\end{pmatrix} = \begin{pmatrix}
		\submat{\mat{Y'}}{[1]} \one_{[1]} & \ldots & \submat{\mat{Y'}}{[K]} \one_{[K]}
	\end{pmatrix}$. By \eqref{eq:XY=XY'-iff-tXtY=tXtY'}, we obtain  $\mat{X} \transpose{\mat{Y}} = \mat{X} \transpose{\mat{Y'}}$. By assumption, $(\mat{X}, \mat{Y}) \sim_p (\mat{X}, \mat{Y'})$. 
	By \Cref{lemma:P-X-block-structure}, there exists a permutation matrix $\mat{P}^{(k)}$ such that $\submat{\mat{Y'}}{[k]}  = \submat{\mat{Y}}{[k]} \mat{P}^{(k)}$ ($1 \leq k \leq K$).
	In conclusion, for each $k \in \integerSet{K}$, $		\matcol{\mat{\underline{Y}'}}{k} = \submat{\mat{Y'}}{[k]} \one_{[k]} = (\submat{\mat{Y}}{[k]} \mat{P}^{(k)}) \one_{[k]} = \submat{\mat{Y}}{[k]} \one_{[k]} = \matcol{\mat{\underline{Y}}}{k}$, because the result of a sum does not depend on the summation order. This yields $\mat{\underline{Y}'} = \mat{\underline{Y}}$, which shows condition (i).
\end{proof}

\section{Proof of \Cref{prop:right-uniqueness-linear-inv-pb-matrix}}
\label{proof:prop-right-uniqueness-linear-inv-pb-matrix}

The proof of this proposition relies on the simpler case where the right factor is reduced to a vector. 

\begin{lemma}
	\label{lemma:uniqueness-reconstruction-vector}
	Let $\bm{A} \in \mathbb{C}^{m \times N}$, and $\theta \subseteq \mathbb{B}^N$ be a family of allowed vector supports. Let $\mat{x^*} \in \Sigma_{\theta}$ be an allowed vector. Then, uniqueness in the linear inverse problem $\mat{A} \vctor{x} = \mat{A} \mat{x^*}$ for $\vctor{x} \in \Sigma_{\theta}$, namely
	\begin{displaymath}
		\forall \vctor{x} \in \Sigma_{\theta}, \; \mat{A} \vctor{x} = \mat{A} \mat{x^*} \implies \vctor{x} = \mat{x^*},
	\end{displaymath}
	holds if, and only if, both of the following conditions hold:
	\begin{enumerate}[label=(\roman*)]
		\item For all $\vctor{s} \in \theta$ such that $\supp(\mat{x^*}) \subseteq \vctor{s}$, the columns of $\submat{\mat{A}}{\vctor{s}}$ are linearly independent.
		\item For all $\vctor{s} \in \theta$ such that $\mat{A} \mat{x^*} \in \im(\submat{\mat{A}}{\vctor{s}})$, the support of $\mat{x^*}$ is included in $\vctor{s}$.
	\end{enumerate}
\end{lemma}

\begin{proof}
	For sufficiency, let $\vctor{x} \in \Sigma_{\theta}$ such that $\mat{A} \vctor{x} = \mat{A} \mat{x^*}$. Fix $\vctor{s} \in \theta$ such that $\vctor{x} \in \Sigma_{\vctor{s}}$. Then, $\mat{A} \mat{x^*} \in \im(\submat{\mat{A}}{\vctor{s}})$, so by condition (ii), the support of $\mat{x^*}$ is included in $\vctor{s}$. 
	By condition (i), we obtain $\mat{x^*} = \vctor{x}$.
	For necessity, suppose that $\mat{x^*}$ is the unique solution to the linear inverse problem $\mat{A} \vctor{x} = \mat{A} \mat{x^*}$ for $\vctor{x} \in \Sigma_{\theta}$. Let $\vctor{s} \in \theta$ such that $\supp(\mat{x^*}) \subseteq \vctor{s}$. 
	Then, by assumption, the application $\vctor{x} \mapsto \mat{A} \vctor{x}$ defined on $\Sigma_{\mat{s}}$ is injective, so $\ker(\submat{\mat{A}}{\vctor{s}}) = \{ \vctor{0} \}$.
	Now, let $\vctor{s} \in \theta$ such that $\mat{A} \mat{x^*} \in \im(\submat{\mat{A}}{\vctor{s}})$. Fix $\vctor{x} \in \Sigma_{\vctor{s}}$ such that $\mat{A} \mat{x^*} = \mat{A} \vctor{x}$. Then, by assumption, $\vctor{x} = \mat{x^*}$, which yields $\supp(\mat{x^*}) \subseteq \vctor{s}$.
	
\end{proof}

\begin{proof}[Proof of \Cref{prop:right-uniqueness-linear-inv-pb-matrix}]
	We simply apply \Cref{lemma:uniqueness-reconstruction-vector} for the choices $\bm{A} = \identity{n} \otimes \mat{X} \in \mathbb{C}^{nm \times nr}$, $\mat{x^*} = \vectorization(\transpose{\mat{Y}}) \in \mathbb{C}^{nr}$ and $\theta = \{ \vectorization(\transpose{\mat{S}}) \; | \; \mat{S} \in \Theta \} \subseteq \mathbb{B}^{nr}$, where $\vectorization(\cdot)$ is the vectorization operator, namely:
	\begin{displaymath}
		\forall \mat{Y'} \in \mathbb{C}^{r \times n}, \quad \vectorization(\mat{Y'}) = \begin{pmatrix}
			\matcol{\mat{Y'}}{1} \\ \vdots \\ \matcol{\mat{Y'}}{n}
		\end{pmatrix} \in \mathbb{C}^{nr}.
	\end{displaymath}
\end{proof}

\section{Proof of \Cref{prop:right-uniqueness-perm-vector-of-ones}}
\label{proof:prop-right-uniqueness-perm-vector-of-ones}

The following proof shows that the two conditions of the proposition are equivalent to $(\mat{Y}, \transpose{\one}) \in \unique( \Sigma_{\Theta} \times \{ \transpose{\one} \})$, which, as mentioned in the remark of \Cref{thm:right-uniqueness-perm}, is equivalent to $( \transpose{\one},\mat{Y}) \in \unique( \{ \transpose{\one} \} \times \Sigma_{\Theta})$.

\begin{proof}
	Suppose that condition (i) and (ii) are verified. Let $\mat{Y'} \in \Sigma_{\Theta}$ such that $\mat{Y'} \one = \mat{Y} \one$. We want to show $(\mat{Y}, \transpose{\one}) \sim_p (\mat{Y'}, \transpose{\one})$. Fix $\mat{S} \in \Theta$ such that $\supp(\mat{Y'}) \subseteq \mat{S}$. Then, $I = \supp(\mat{Y} \one) = \supp(\mat{Y'} \one) \subseteq \bigcup_{i=1}^r \matcol{\supp(\mat{Y'})}{i} \subseteq \bigcup_{i=1}^r \matcol{\mat{S}}{i}$, so by the first part of condition (ii), the columns $( \matindex{\mat{S}}{I^c}{j} )_{j=1}^r$ are pairwise disjoint. 
	Consequently, since $\sum_{j=1}^r \matindex{\mat{Y'}}{k}{j} = 0$ for each $k \in I^c$, the rows of $\mat{Y'}$ indexed by $k \in I^c$ are zero rows. 
	In addition, by the second part of condition (ii), the columns $( \matindex{\mat{S}}{I}{j} )_{j=1}^r$ and $(\matindex{\supp(\mat{Y})}{I}{j})_{j=1}^r$ are equal, up to a permutation of indices $j$. By condition (i), this means that the columns $\{ \matindex{\mat{S}}{I}{j} \}_{j=1}^r$ are pairwise disjoint. Thus, the entries of $\mat{Y'}$ are directly identified from its sum of columns $\mat{Y'} \one = \mat{Y} \one$, hence, $(\mat{Y}, \transpose{\one}) \sim_p (\mat{Y'}, \transpose{\one})$. 
	
	Conversely, suppose that condition (i) is not verified. Fix two different indices $i, j$ such that $\matcol{\supp(\mat{Y})}{i} \cap \matcol{\supp(\mat{Y})}{j} \neq \emptyset$. Fix $k$ an index in this intersection. Define $\mat{Y'} \in \Sigma_{\Theta}$ such that:
	\begin{displaymath}
		\forall l \in \integerSet{r}, \quad \matcol{\mat{Y'}}{l} = \begin{cases}
			\matcol{\mat{Y}}{i} + \lambda \basis{k} & \text{if $l=i$} \\
			\matcol{\mat{Y}}{j} - \lambda \basis{k} & \text{if $l=j$} \\
			\matcol{\mat{Y}}{l} & \text{otherwise} \\
		\end{cases},
	\end{displaymath}
	where $\lambda$ is a scalar different from $\matindex{\mat{Y}}{k}{l} - \matindex{\mat{Y}}{k}{i}$ for each $l \in \integerSet{r}$. Then, one easily verifies that $\mat{Y} \one = \mat{Y'} \one$. 
	However, $\matcol{\mat{Y'}}{i} \neq \matcol{\mat{Y}}{l}$ for each $l \in \integerSet{r}$ by construction, thus, $\matcol{\mat{Y'}}{i} \neq \matcol{(\mat{Y} \mat{P})}{i}$ for each $\mat{P} \in \permutation{r}$, and necessarily, $(\mat{Y}, \transpose{\one}) \not\sim_p (\mat{Y'}, \transpose{\one})$. 
	Suppose now that the first part of condition (ii) is not verified. Fix $\mat{S} \in \Theta$ such that $I := \supp(\mat{Y} \one) \subseteq \bigcup_{i=1}^r \matcol{\mat{S}}{i}$, and two different indices $i \neq j$ such that $\matindex{\mat{S}}{I^c}{i} \cap \matindex{\mat{S}}{I^c}{j} \neq \emptyset$. 
	Fix $k$ in this intersection. Then, we can construct similarly $\mat{Y'} \in \Sigma_{\mat{S}}$ such that $\mat{Y'} \one = \mat{Y} \one$, but $(\mat{Y}, \transpose{\one}) \not\sim_p (\mat{Y'}, \transpose{\one})$.
	
	Finally, suppose that the second part of condition (iii) is not verified. Fix $\mat{S} \in \Theta$ such that $I := \supp(\mat{Y} \one) \subseteq \bigcup_{i=1}^r \matcol{\mat{S}}{i}$, and $(\matindex{\mat{S}}{I}{j})_{j=1}^r \neq (\matindex{\supp(\mat{Y})}{I}{\sigma(j)})_{j=1}^r$ for any permutation $\sigma \in \mathfrak{S}(\integerSet{r})$. Construct $\mat{Y'} \in \Sigma_{\mat{S}}$ in the following way:
	\begin{displaymath}
		\forall (k,l) \in \integerSet{n} \times \integerSet{r}, \quad \matindex{\mat{Y'}}{k}{l} = \begin{cases}
			\frac{\vctorindex{(\mat{Y}\one)}{k}}{ \| \matcol{(\transpose{\mat{S}})}{k} \|_0} & \text{if $k \in I$ and $l \in \matcol{(\transpose{\mat{S}})}{k}$} \\
			0 & \text{otherwise}
		\end{cases},
	\end{displaymath}
	which is well defined because since $I \subseteq \bigcup_{i=1}^r \matcol{\mat{S}}{i}$, we have $\| \matcol{(\transpose{\mat{S}})}{k} \|_0 > 0$ for $k \in I$. Then, one verifies that $\mat{Y'} \one = \mat{Y} \one$. However, by construction, $(\matindex{\supp(\mat{Y'})}{I}{j})_{j=1}^r = (\matindex{\mat{S}}{I}{j})_{j=1}^r \neq (\matindex{\supp(\mat{Y})}{I}{\sigma(j)})_{j=1}^r$ for any permutation $\sigma \in \mathfrak{S}(\integerSet{r})$. This means that $(\mat{Y}, \transpose{\one}) \not\sim_p (\mat{Y'}, \transpose{\one})$. 
\end{proof}

\section{Proof of \Cref{lemma:varphi-properties}}
\label{proof:lemma-varphi-properties}

\begin{proof}
	Consider $(\mat{X}, \mat{Y}) \sim (\mat{X'}, \mat{Y'})$. By definition, there exists $\mat{D} \in \diagonal{r}$ and $\mat{P} \in \permutation{r}$ such that $\mat{X'} = \bm{XDP}$ and $\mat{Y'} = \mat{Y} \inverse{\mat{D}} \mat{P}$. Since $\varphi(\bm{XD}, \inverse{\mat{D}} \mat{Y}) = \varphi(\mat{X}, \mat{Y})$, we immediately obtain $\varphi(\mat{X}, \mat{Y}) \sim \varphi(\mat{X'}, \mat{Y'})$.
	For surjectivity, let $\tuplerkone{C} \in \Gamma_{\Omega}$. By definition, there exists $\pair{S} \in \Omega$ such that $\tuplerkone{C} \in \Gamma_{\tuplerkone{S}}$, where $\tuplerkone{S} = \varphi(\pair{S})$. 
	Fix $i \in \integerSet{r}$. 
	If $\rankone{C}{i} = \mat{0}$, define $\vctorseq{\vctor{x}}{i} = \vctor{0}$ and $\vctorseq{\vctor{y}}{i} = \vctor{0}$. Otherwise, $\rankone{C}{i}$ is a rank-one matrix, and since $\supp(\rankone{C}{i}) \subseteq \matcol{\leftsupp{S}}{i} \times \matcol{\rightsupp{S}}{i}$ (viewing the columns $\matcol{\leftsupp{S}}{i}$ and $\matcol{\rightsupp{S}}{i}$ as subset of indices), we can find $(\vctorseq{\vctor{x}}{i}, \vctorseq{\vctor{y}}{i}) \in \Sigma_{\matcol{\leftsupp{S}}{i}} \times \Sigma_{\matcol{\rightsupp{S}}{i}}$ such that $\rankone{C}{i} = \vctorseq{\vctor{x}}{i} \transpose{\vctorseq{\vctor{y}}{i}}$, with $\vctorseq{\vctor{x}}{i} \neq \vctor{0}$, $\vctorseq{\vctor{y}}{i} \neq \vctor{0}$, by \Cref{lemma:id-from-outerproduct}. Then, the matrices $\mat{X} := (\vctorseq{\vctor{x}}{i})_{i \in \integerSet{r}}$, $\mat{Y} := (\vctorseq{\vctor{y}}{i})_{i \in \integerSet{r}}$ verifies $(\mat{X}, \mat{Y}) \in \idcolsupp{\Omega}$, and $\varphi_{\Omega}(\mat{X}, \mat{Y}) = \tuplerkone{C}$ by construction.
	For injectivity up to equivalences, let $(\mat{X}, \mat{Y}), (\mat{X'}, \mat{Y'}) \in \idcolsupp{\Omega}$ such that $\varphi_{\Omega}(\mat{X}, \mat{Y}) \sim \varphi_{\Omega}(\mat{X'}, \mat{Y'})$. 
	Fix $\sigma \in \mathfrak{S}(\integerSet{r})$ such that $\matcol{\mat{X}}{\sigma(i)} \transpose{\matcol{\mat{Y}}{\sigma(i)}} = \matcol{\mat{X'}}{i} \transpose{\matcol{\mat{Y'}}{i}}$ ($1 \leq i \leq r$). 
	If $\matcol{\mat{X'}}{i} \transpose{\matcol{\mat{Y'}}{i}} = \mat{0}$, then since $(\mat{X}, \mat{Y}), (\mat{X'}, \mat{Y'}) \in \idcolsupp{\Omega}$, we have $\matcol{\mat{X}}{\sigma(i)} = \matcol{\mat{X'}}{i} = \vctor{0}$ and $\transpose{\matcol{\mat{Y}}{\sigma(i)}} = \transpose{\matcol{\mat{Y'}}{i}} = \vctor{0}$ by \Cref{lemma:id-from-outerproduct}. 
	Otherwise, $\matcol{\mat{X'}}{i} \transpose{\matcol{\mat{Y'}}{i}} \neq \mat{0}$, and, by \Cref{lemma:id-from-outerproduct}, there exists $\lambda_i \neq 0$ such that $\matcol{\mat{X'}}{i} = \lambda_i \matcol{\mat{X}}{\sigma(i)}$ and $\matcol{\mat{Y'}}{i} = \frac{1}{\lambda_i} \matcol{\mat{Y}}{\sigma(i)}$. Thus, $(\mat{X}, \mat{Y}) \sim (\mat{X'}, \mat{Y'})$.
\end{proof}

\section{Proof of \Cref{prop:id-support-and-entries}}
\label{proof:prop-id-support-and-entries}

\begin{proof}
	For necessity, suppose $\tuplerkone{C} \in \unique(\Gamma_{\Omega})$. By analogy with \Cref{lemma:nec-conditions-EMF}, condition (ii) is verified.
	Now, let $\tuplerkone{S} \in \varphi(\Omega)$ such that $\mathcal{A}(\tuplerkone{C}) \in \mathcal{A}(\Gamma_{\tuplerkone{S}})$. Fix $\tuplerkone{C}' \in \Gamma_{\tuplerkone{S}}$ such that $\mathcal{A}(\tuplerkone{C'}) = \mathcal{A}(\tuplerkone{C})$. By assumption, $\tuplerkone{C} \sim \tuplerkone{C}'$. Fix $\sigma \in \mathfrak{S}(\integerSet{r})$ such that $\rankone{C}{i} = \rankone{C'}{\sigma(i)} \in \Sigma_{\rankone{S}{\sigma(i)}}$ for all $i \in \integerSet{r}$. This yields condition (i).
	For sufficiency, suppose that condition (i) and (ii) are verified, and let $\tuplerkone{C'} \in \Gamma_{\Omega}$ such that $\mathcal{A}(\tuplerkone{C}) = \mathcal{A}(\tuplerkone{C'})$. 
	We want to show $\tuplerkone{C'} \sim \tuplerkone{C}$.
	Fix $\tuplerkone{S'} \in \varphi(\Omega)$ such that $\tuplerkone{C'} \in \Gamma_{\tuplerkone{S'}}$. Then by construction, $\mathcal{A}(\tuplerkone{C}) \in \mathcal{A}(\Gamma_{\tuplerkone{S'}})$, so by condition (i), fix $\sigma \in \mathfrak{S}(\integerSet{r})$ such that $\supp(\rankone{C}{i}) \subseteq \rankone{S'}{\sigma(i)}$ ($1 \leq i \leq r$). In other words, defining $\tuplerkone{S} := (\rankone{S'}{\sigma(i)})_{i=1}^r$, we have $\tuplerkone{C} \in \Gamma_{\tuplerkone{S}}$. But $\Omega$ is stable by permutation, so $\tuplerkone{S} \in \varphi(\Omega)$. 
	By definition, $\tuplerkone{C''} := (\rankone{C'}{\sigma(i)})$ belongs to $\Gamma_{\tuplerkone{S}}$. As $\mathcal{A}$ is invariant to permutation, $\mathcal{A}(\tuplerkone{C''}) = \mathcal{A}(\tuplerkone{C'}) = \mathcal{A}(\tuplerkone{C})$. By condition (ii), $\tuplerkone{C} \in \unique(\Gamma_{\tuplerkone{S}})$, so $\tuplerkone{C''} \sim \tuplerkone{C}$. As $\tuplerkone{C''} \sim \tuplerkone{C'}$, we obtain $\tuplerkone{C'} \sim \tuplerkone{C}$.
\end{proof}

\section{Proof of \Cref{cor:equivalence-EMD-fixed-support}}
\label{proof:cor-equivalence-EMD-fixed-support}

In order to prove this corollary, we have to establish the following result. Denote the equivalence class of $\pair{S} \in \mathbb{B}^{m \times r} \times \mathbb{B}^{n \times r}$ with respect to permutation equivalence (see \Cref{def:puniqueness}) as:

\begin{displaymath}
	[\pair{S}] := \{ \pair{S'} \in \mathbb{B}^{m \times r} \times \mathbb{B}^{n \times r} \; | \; \pair{S'} \sim_p \pair{S} \}.
\end{displaymath}

\begin{lemma}
	\label{lemma:unique-equivalence-class}
	For any fixed pair of supports $\pair{S}$, $\unique(\Sigma_{[\pair{S}]}) \cap \Sigma_{\pair{S}} = \unique(\Sigma_{\pair{S}})$. Similarly, denoting $\tuplerkone{S} := \varphi(\pair{S})$, $\unique(\Gamma_{[\pair{S}]}) \cap \Gamma_{\tuplerkone{S}} = \unique(\Gamma_{\tuplerkone{S}})$.
\end{lemma}

\begin{proof}
	The inclusion $\unique(\Sigma_{[\pair{S}]}) \cap \Sigma_{\pair{S}} \subseteq \unique(\Sigma_{\pair{S}})$ is an application of \Cref{lemma:nec-conditions-EMF}. For the inverse inclusion, let $(\mat{X}, \mat{Y}) \in \unique(\Sigma_{\pair{S}})$, and $(\mat{X'}, \mat{Y'}) \in \Sigma_{[\pair{S}]}$ such that $\mat{X'} \transpose{\mat{Y'}} = \mat{X} \transpose{\mat{Y}}$. Fix $\mat{P} \in \permutation{r}$ such that $(\mat{X'}, \mat{Y'}) \in \Sigma_{(\leftsupp{S} \mat{P}, \rightsupp{S} \mat{P})}$. This means that $(\mat{X'} \inverse{\mat{P}}, \mat{Y'} \inverse{\mat{P}}) \in \Sigma_{\pair{S}}$. Since $(\mat{X}, \mat{Y}) \in \unique(\Sigma_{\pair{S}})$, we have $(\mat{X'} \inverse{\mat{P}}, \mat{Y'} \inverse{\mat{P}}) \sim (\mat{X}, \mat{Y})$, hence, $(\mat{X'}, \mat{Y'}) \sim (\mat{X}, \mat{Y})$. The proof is similar for the equality $\unique(\Gamma_{[\pair{S}]}) \cap \Gamma_{\tuplerkone{S}} = \unique(\Gamma_{\tuplerkone{S}})$.
\end{proof}

\begin{proof}[Proof of \Cref{cor:equivalence-EMD-fixed-support}]
	Suppose $(\mat{X}, \mat{Y}) \in \unique(\Sigma_{\pair{S}})$. By \Cref{lemma:unique-equivalence-class}, $(\mat{X}, \mat{Y}) \in \unique(\Sigma_{[\pair{S}]})$, so by applying \Cref{thm:equivalence-with-EMD} to the family $[\pair{S}]$ which is stable by permutation, we obtain $\varphi(\mat{X}, \mat{Y}) \in \unique(\Gamma_{[\pair{S}]})$ and $(\mat{X}, \mat{Y}) \in \idcolsupp{[\pair{S}]} \cap \maxcolsupp{[\pair{S}]}$. One the one hand, by \Cref{lemma:unique-equivalence-class}, $\varphi(\mat{X}, \mat{Y}) \in \unique(\Gamma_{\tuplerkone{S}})$. On the other hand, since $\idcolsupp{[\pair{S}]} \cap \maxcolsupp{[\pair{S}]} \cap \Sigma_{\pair{S}} = \idcolsupp{\pair{S}} \cap \maxcolsupp{\pair{S}}$, and $(\mat{X}, \mat{Y}) \in \Sigma_{\pair{S}}$, we obtain $(\mat{X}, \mat{Y}) \in \idcolsupp{\pair{S}} \cap \maxcolsupp{\pair{S}}$. 
	
	Conversely, suppose $\varphi(\mat{X}, \mat{Y}) \in \unique(\Gamma_{\tuplerkone{S}}) \text{ and } (\mat{X}, \mat{Y}) \in \idcolsupp{\pair{S}} \cap \maxcolsupp{\pair{S}}$.
	By \Cref{lemma:unique-equivalence-class}, $\varphi(\mat{X}, \mat{Y}) \in \unique(\Gamma_{[\pair{S}]})$. Moreover, $\idcolsupp{\pair{S}} \subseteq \idcolsupp{[\pair{S}]}$, so $(\mat{X}, \mat{Y}) \in \idcolsupp{[\pair{S}]}$. We now show that $(\mat{X}, \mat{Y}) \in \maxcolsupp{[\pair{S}]}$. Let $\pair{S'} \sim_p \pair{S}$ such that $(\mat{X}, \mat{Y}) \in \Sigma_{\pair{S'}}$. Fix $\mat{P} \in \permutation{r}$ such that $\pair{S'} = (\leftsupp{S} \mat{P}, \rightsupp{S} \mat{P})$. 
	We have $\supp(\mat{X}) \subseteq \leftsupp{S'}$. Since $(\mat{X}, \mat{Y}) \in \maxcolsupp{\pair{S}}$, $\colsupp(\mat{X}) = \colsupp(\leftsupp{S})$. Hence, $\colsupp(\leftsupp{S}) \subseteq \colsupp(\leftsupp{S'})$. But $\pair{S'} \sim_p \pair{S}$, so $\colsupp(\leftsupp{S})$ and $\colsupp(\leftsupp{S'})$ have the same cardinality. Therefore, $\colsupp(\leftsupp{S'}) = \colsupp(\leftsupp{S}) = \colsupp(\mat{X})$. Similarly, $\colsupp(\rightsupp{S'}) = \colsupp(\rightsupp{S}) = \colsupp(\mat{Y})$. This shows $(\mat{X}, \mat{Y}) \in \maxcolsupp{[\pair{S}]}$. By \Cref{thm:equivalence-with-EMD}, we conclude that $(\mat{X}, \mat{Y}) \in \unique(\Sigma_{[\pair{S}]})$, which means $(\mat{X}, \mat{Y}) \in \unique(\Sigma_{\pair{S}})$ by \Cref{lemma:unique-equivalence-class}.
\end{proof}

\section{Proof of \Cref{prop:simple-SC-identify-supports}}
\label{proof:prop-simple-SC-identify-supports}

\begin{proof} 
	Consider $\tuplerkone{S} \in \varphi(\Omega)$ a tuple of rank-one supports such that $\mat{Z} \in \mathcal{A}(\Gamma_{\tuplerkone{S}})$, and $\tuplerkone{C'} \in \Gamma_{\tuplerkone{S}}$ a tuple of rank-one matrices with these supports such that $\mathcal{A}(\mathcal{C'}) = \mat{Z}$. By condition (i), the supports in the tuple $\tuplerkone{S'} := ( \supp(\rankone{C'}{i}) )_{i=1}^r$ are pairwise disjoint. 
	Therefore, $\supp(\mat{Z}) = \supp(\sum_{i=1}^r \rankone{C'}{i}) = \bigcup_{i=1}^r \supp(\rankone{C'}{i}) = \bigcup_{i=1}^r \rankone{S'}{i}$. In other words, $\{ \rankone{S'}{i} \}_{i=1}^r$ is a partition of $\supp(\mat{Z})$. As a consequence, $(\matindex{\mat{Z}}{k}{l})_{(k,l) \in \rankone{S'}{i}} = (\matindex{\rankone{C'}{i}}{k}{l})_{(k,l) \in \rankone{S'}{i}}$ ($1 \leq i \leq r$), which by definition is of rank at most one. Similarly, $(\supp(\rankone{C}{i}))_{i=1}^r$ is also a partition of $\supp(\mat{Z})$, and the submatrices $(\matindex{\mat{Z}}{k}{l})_{(k,l) \in \supp(\rankone{C}{i})} = (\matindex{\rankone{C}{i}}{k}{l})_{(k,l) \in \supp(\rankone{C}{i})}$ are also of rank at most one.
	But $\tuplerkone{S'} \in \varphi^{\mathrm{cp}}(\Omega)$ because $\rankone{S'}{i} = \supp(\rankone{C'}{i}) \subseteq \rankone{S}{i}$ ($1 \leq i \leq r$), and $(\supp(\rankone{C}{i}))_{i=1}^r \in \varphi^{\mathrm{cp}}(\Omega)$ because $\tuplerkone{C} \in \Gamma_{\Omega}$.
	Hence, by condition (ii), $\tuplerkone{S'} \sim (\supp(\rankone{C}{i}))_{i=1}^r$, which shows the existence of $\sigma \in \mathfrak{S}(\integerSet{r})$ such that $\supp(\rankone{C}{i}) = \rankone{S'}{\sigma(i)} = \supp(\rankone{C'}{\sigma(i)}) \subseteq \rankone{S}{\sigma(i)}$. This precisely yields condition (i) of \Cref{prop:id-support-and-entries}.
\end{proof}

\section{Proof of \Cref{thm:NC-fixed-support-id}}
\label{proof:thm-NC-fixed-support-id}

The proof of the theorem is composed of the four following steps.

\subsection{Equivalence with fixed-support identifiability in EMF}
Given an $r$-tuple of rank-one supports $\tuplerkone{S}$, a preliminary necessary condition of $\tuplerkone{C} \in \unique(\Gamma_{\tuplerkone{S}})$ for a given $\tuplerkone{C} \in \Gamma_{\tuplerkone{S}}$ is based on the non-degeneration properties of \cref{section:non-degen-prop}. 

\begin{lemma}
	\label{lemma:preliminary-NC-fixed-support-id}
	Consider $\pair{S}$ a pair of supports, $\tuplerkone{S} = \varphi(\pair{S})$, and $\tuplerkone{C} \in \Gamma_{\tuplerkone{S}}$. Suppose $\tuplerkone{C} \in \unique(\Gamma_{\tuplerkone{S}})$. Then, for any $(\mat{X}, \mat{Y}) \in \Sigma_{\pair{S}}$ such that $\varphi(\mat{X}, \mat{Y}) = \tuplerkone{C}$, $(\mat{X}, \mat{Y}) \in \idcolsupp{\pair{S}} \cap \maxcolsupp{\pair{S}}$, or in other words, $\colsupp(\leftsupp{S}) = \colsupp(\mat{X}) = \colsupp(\mat{Y}) = \colsupp(\rightsupp{S})$.
\end{lemma}

\begin{proof}
	Suppose there exists $(\mat{X}, \mat{Y}) \in \Sigma_{\pair{S}}$ such that $\tuplerkone{C} = \varphi(\mat{X}, \mat{Y})$ but $(\mat{X}, \mat{Y}) \notin \idcolsupp{\pair{S}} \cap \maxcolsupp{\pair{S}}$. As $\idcolsupp{[\pair{S}]} \cap \maxcolsupp{[\pair{S}]} \cap \Sigma_{\pair{S}} = \idcolsupp{\pair{S}} \cap \maxcolsupp{\pair{S}}$, this means that $(\mat{X}, \mat{Y}) \notin \idcolsupp{[\pair{S}]} \cap \maxcolsupp{[\pair{S}]}$. 
	By \Cref{lemma:identical-colsupp,lemma:maximal-colsupp}, $(\mat{X}, \mat{Y}) \notin \unique(\Sigma_{[\pair{S}]})$. Then, by \Cref{lemma:unique-equivalence-class}, $(\mat{X}, \mat{Y}) \notin \unique(\Sigma_{\pair{S}})$, which means, by \Cref{cor:equivalence-EMD-fixed-support}, that $\tuplerkone{C} = \varphi(\mat{X}, \mat{Y}) \notin \unique(\Gamma_{\tuplerkone{S}})$.
\end{proof}

Consequently, thanks to \Cref{cor:equivalence-EMD-fixed-support}, the necessary condition of \Cref{lemma:preliminary-NC-fixed-support-id} is equivalent to the one in the following lemma.

\begin{lemma}
	Suppose $\tuplerkone{C} \in \unique(\Gamma_{\tuplerkone{S}})$. Then, for any $(\mat{X}, \mat{Y}) \in \Sigma_{\pair{S}}$ such that $\varphi(\mat{X}, \mat{Y}) = \tuplerkone{C}$, $(\mat{X}, \mat{Y}) \in \unique(\Sigma_{\pair{S}})$. The converse is also true.
\end{lemma} 

\begin{proof}
	Let $(\mat{X}, \mat{Y}) \in \Sigma_{\pair{S}}$ such that $\tuplerkone{C} = \varphi(\mat{X}, \mat{Y})$. By \Cref{lemma:preliminary-NC-fixed-support-id}, this means that $(\mat{X}, \mat{Y}) \in \idcolsupp{\pair{S}} \cap \maxcolsupp{\pair{S}}$. But by assumption, $\varphi(\mat{X}, \mat{Y}) = \tuplerkone{C} \in \unique(\Gamma_{\tuplerkone{S}})$. By \Cref{cor:equivalence-EMD-fixed-support}, $(\mat{X}, \mat{Y}) \in \unique(\Sigma_{\pair{S}})$. The converse is true by \Cref{cor:equivalence-EMD-fixed-support}.
\end{proof}

\subsection{Fixing the left factor}  
\label{section:step2-sm11}
As a consequence, we obtain the following necessary condition for $\tuplerkone{C} \in \unique(\Gamma_{\tuplerkone{S}})$ by fixing the left factor.

\begin{lemma}
	\label{lemma:fixed-support-right-id-NC}
	Suppose $\tuplerkone{C} \in \unique(\Gamma_{\tuplerkone{S}})$. Then, for any $(\mat{X}, \mat{Y}) \in \Sigma_{\pair{S}}$ such that $\varphi(\mat{X}, \mat{Y}) = \tuplerkone{C}$, $(\mat{X}, \mat{Y}) \in \unique(\{ \mat{X} \} \times \Sigma_{\rightsupp{S}})$. 
\end{lemma}

\begin{proof}
	We apply \Cref{lemma:nec-conditions-EMF} with $\Sigma' := \{ \mat{X} \} \times \Sigma_{\rightsupp{S}} \subseteq \Sigma_{\pair{S}} =:  \Sigma$, 
\end{proof}

We leverage the analysis of \cref{section:right-identifiability} to provide the following necessary conditions for $\tuplerkone{C} \in \unique(\Gamma_{\tuplerkone{S}})$.

\begin{lemma}
	\label{lemma:appication-analysis-right-id-fixed-supp}
	Suppose $\tuplerkone{C} \in \unique(\Gamma_{\tuplerkone{S}})$. Then, using the same notations introduced in \Cref{thm:NC-fixed-support-id} for any $(\mat{X}, \mat{Y}) \in \Sigma_{\pair{S}}$ such that $\varphi(\mat{X}, \mat{Y}) = \tuplerkone{C}$, we have:
	\begin{enumerate}[label=(\roman*)]
		\item for each $k \in \integerSet{K}$, the columns of the submatrix $\submat{\rightsupp{S}}{[k]}$ are pairwise disjoint;
		\item for each $j \in \integerSet{n}$, the subset of columns $\{ \vctor{x}^{(k)} \; : \; k \in \integerSet{K}, \, \supp(\rightsupp{S}_{[k]}\one_{[k]}) \ni j \}$ is linearly independent.
	\end{enumerate}
\end{lemma}

\begin{proof}
	Let $(\mat{X}, \mat{Y}) \in \Sigma_{\pair{S}}$ such that $\varphi(\mat{X}, \mat{Y}) = \tuplerkone{C}$. 
	By \Cref{lemma:fixed-support-right-id-NC},  $(\mat{X}, \mat{Y}) \in \unique(\{ \mat{X} \} \times \Sigma_{\rightsupp{S}})$.
	By \Cref{lemma:right-id-first-lemma}, we get $\colsupp(\mat{Y}) \subseteq \colsupp(\mat{X}) := J$. By \Cref{lemma:eliminating-zero-col-right-id}, this yields $(\submat{\mat{X}}{J}, \submat{\mat{Y}}{J}) \in \unique(\{ \submat{\mat{X}}{J} \} \times \Sigma_{\rightsupp{S}})$. 
	From \Cref{prop:right-id-up-permutation}, we obtain $(\submat{\mat{X}}{J} \mat{N}, \submat{\mat{Y}}{J} \inverse{\mat{N}}) \in \uniqueperm(\{ \submat{\mat{X}}{J} \mat{N} \} \times \Sigma_{\rightsupp{S}})$, where $\mat{N}$ is the diagonal matrix which normalizes the columns of $\submat{\mat{X}}{J}$, in a such way that the first nonzero entry of each normalized nonzero column is 1. 
	Hence, by applying \Cref{thm:right-uniqueness-perm,prop:right-uniqueness-linear-inv-pb-matrix,prop:right-uniqueness-perm-vector-of-ones} to the fixed left factor $\submat{\mat{X}}{J} \mat{N}$ which has no zero column, and to the family of right supports reduced to the singleton $\{ \submat{\rightsupp{S}}{J} \}$, we obtain the claimed necessary conditions of the proposition.
\end{proof}

\subsection{Fixing the right factor}
Because of the following result, the analysis when fixing the right factor is the same as the one when fixing the left factor.
\begin{lemma}
	\label{lemma:equivalence-fixing-right-factor}
	For any pair of supports $(\leftsupp{S}, \rightsupp{S})$, we have:
	\begin{displaymath}
		(\mat{X}, \mat{Y}) \in \unique(\Sigma_{(\leftsupp{S}, \rightsupp{S})}) \iff (\mat{Y}, \mat{X}) \in \unique(\Sigma_{(\rightsupp{S}, \leftsupp{S})}).
	\end{displaymath}
\end{lemma}

\begin{proof}
	Because of the symmetry, we only prove the direct sense. Suppose $	(\mat{X}, \mat{Y}) \in \unique(\Sigma_{(\leftsupp{S}, \rightsupp{S})})$. Let $(\mat{B}, \mat{A}) \in \Sigma_{(\rightsupp{S}, \leftsupp{S})}$ such that $\mat{B} \transpose{\mat{A}} = \mat{Y} \transpose{\mat{X}}$, and let us show that $(\mat{B}, \mat{A}) \sim (\mat{Y}, \mat{X})$. The previous equality is equivalent to $\mat{A} \transpose{\mat{B}} = \mat{X} \transpose{\mat{Y}}$. Then, by assumption, $(\mat{A}, \mat{B}) \sim (\mat{X}, \mat{Y})$. Hence, there exists $\mat{G} \in \genperm{r}$ such that $\mat{X} = \mat{A} \mat{G}$ and $\transpose{\mat{Y}} = \inverse{\mat{G}} \transpose{\mat{B}}$. Thus, $\mat{Y} = \mat{B} \inverse{(\transpose{\mat{G}})}$, and $\transpose{\mat{X}} = \transpose{\mat{G}} \transpose{\mat{B}}$. Since $\inverse{(\transpose{\mat{G}})} \in \genperm{r}$, we obtain $(\mat{B}, \mat{A}) \sim (\mat{Y}, \mat{X})$, which ends the proof.
\end{proof}

Consequently, we can express similar necessary conditions as the ones of \Cref{lemma:appication-analysis-right-id-fixed-supp} for the case where we fix the right factor.

\begin{lemma}
	\label{lemma:sm11-fixing-right}
	Suppose $\tuplerkone{C} \in \unique(\Gamma_{\tuplerkone{S}})$. Then, using the same notations introduced in \Cref{thm:NC-fixed-support-id} for any $(\mat{X}, \mat{Y}) \in \Sigma_{\pair{S}}$ such that $\varphi(\mat{X}, \mat{Y}) = \tuplerkone{C}$, we have:
	\begin{enumerate}[label=(\roman*)]
		\item for each $u \in \integerSet{U}$, the columns of the submatrix $\submat{\leftsupp{S}}{[u]}$ are pairwise disjoint;
		\item for each $i \in \integerSet{m}$, the subset of columns $\{ \vctor{y}^{(u)} \; : \; u \in \integerSet{U}, \, \supp(\leftsupp{S}_{[u]}\one_{[u]}) \ni i \}$ is linearly independent.
	\end{enumerate}
\end{lemma}

\begin{proof}
	We use \Cref{lemma:equivalence-fixing-right-factor} and repeat the reasoning of \cref{section:step2-sm11}.
\end{proof}

\subsection{Combine the obtained necessary conditions} 
It is now possible to formulate the obtained necessary conditions in the following equivalent manner.

\begin{lemma}
	Let $\tuplerkone{C} \in \unique(\Gamma_{\tuplerkone{S}})$. The two following conditions are equivalent:
	\begin{enumerate}[label=(\roman*)]
		\item for any $(\mat{X}, \mat{Y}) \in \Sigma_{\pair{S}}$ such that $\varphi(\mat{X}, \mat{Y}) = \tuplerkone{C}$, for each $k \in \integerSet{K}$, the columns of the submatrix $\submat{\rightsupp{S}}{[k]}$ are pairwise disjoint;
		\item whenever two nonzero rank-one contributions $\rankone{C}{t}$ and $\rankone{C}{i'}$ have nonzero collinear columns ($ i, i' \in \integerSet{r}$, $i \neq i'$), the column support $\matcol{\rightsupp{S}}{i}$, $\matcol{\rightsupp{S}}{i'}$ are disjoint;
		\item for all $ i,i' \in \integerSet{r}$, $i \neq i'$, if the rank-one contributions $\rankone{C}{i}$ and $\rankone{C}{i'}$ are nonzero with the same column span, then the rank-one supports $\rankone{S}{i}$ and $\rankone{S}{i'}$ are disjoint.
	\end{enumerate}
\end{lemma}

\begin{proof}
	The condition (iii) is simply a reformulation of condition (ii).
	
	Suppose (i). Consider $i, i' \in \integerSet{r}$, $i \neq i'$ such that $\rankone{C}{i}$ and $\rankone{C}{i'}$ have nonzero collinear columns. By \Cref{lemma:varphi-properties}, there exists $(\mat{X}, \mat{Y}) \in \Sigma_{\pair{S}}$ such that $\varphi(\mat{X}, \mat{Y}) = \tuplerkone{C}$. In particular, $\rankone{C}{i} = \matcol{\mat{X}}{i} \transpose{\matcol{\mat{Y}}{i}}$ and $\rankone{C}{i'} = \matcol{\mat{X}}{i'} \transpose{\matcol{\mat{Y}}{i'}}$. By assumption, $\rankone{C}{i}$ and $\rankone{C}{i'}$ have nonzero collinear columns. This means that $\matcol{\mat{X}}{i} \neq \vctor{0}$ and $\matcol{\mat{X}}{i'} \neq \vctor{0}$ are collinear. By assumption (i), the columns support $\matcol{\rightsupp{S}}{i}$, $\matcol{\rightsupp{S}}{i'}$ are disjoint.
	
	Conversely, suppose (ii). Let $(\mat{X}, \mat{Y}) \in \Sigma_{\pair{S}}$ such that $\varphi(\mat{X}, \mat{Y}) = \tuplerkone{C}$. Let $k \in \integerSet{K}$, and consider any $i, i' \in [k]$, with $i \neq i'$. By definition, the nonzero columns $\matcol{\mat{X}}{i}$ and $\matcol{\mat{X}}{i'}$ are collinear. Fix $\lambda \neq 0$ such that $\matcol{\mat{X}}{i} = \lambda \matcol{\mat{X}}{i'}$. Then, $\rankone{C}{i} = \matcol{\mat{X}}{i} \transpose{\matcol{\mat{Y}}{i}} = \lambda \matcol{\mat{X}}{i'} \transpose{\matcol{\mat{Y}}{i}}$. But $\lambda \matcol{\mat{X}}{i'} \transpose{\matcol{\mat{Y}}{i}}$ and $\matcol{\mat{X}}{i'} \transpose{\matcol{\mat{Y}}{i'}} = \rankone{C}{i'}$ have nonzero collinear columns. By assumption (ii), the columns support $\matcol{\rightsupp{S}}{i}$, $\matcol{\rightsupp{S}}{i'}$ are disjoint.
\end{proof}
We express a similar lemma for the case where the right factor has been fixed.
We conclude by expressing jointly condition (iii) in these two lemmas to obtain the necessary condition (i) in \Cref{thm:NC-fixed-support-id}. We then combine conditions (ii) of \Cref{lemma:appication-analysis-right-id-fixed-supp,lemma:sm11-fixing-right} to obtain the necessary condition (ii) of \Cref{thm:NC-fixed-support-id}.

\section{Proof of \Cref{example:NC-fixed-supp-id-is-sometime-SC}}
\label{proof:example-NC-fixed-supp-id-is-sometime-SC}

\begin{proof}
	
	First, we verify that conditions (i) and (ii) of \Cref{thm:NC-fixed-support-id} are  equivalent to linear independence of columns $\begin{pmatrix}
		a & b
	\end{pmatrix}^T$ and $\begin{pmatrix}
		c & d
	\end{pmatrix}^T$. Indeed, if these two columns are collinear, then condition (i) cannot be verified, as $\rankone{S}{1}$ and $\rankone{S}{2}$ are not disjoint. Conversely, if these two columns are linearly independent, then condition (ii) is verified. Indeed, for any $(\mat{X}, \mat{Y}) \in \Sigma_{\pair{S}}$ (where $\pair{S}$ is the unique pair of supports such that $\varphi(\pair{S}) = \tuplerkone{S}$) such that $\tuplerkone{C} = \varphi(\mat{X}, \mat{Y})$, $\mat{X}$ is equal to $\begin{pmatrix}
		a & b \\ c & d
	\end{pmatrix}$ up to rescaling of columns, and $\mat{Y}$ is equal to $\begin{pmatrix}
		1 & 0 \\ x & y \\ 0 & 1
	\end{pmatrix}$ up to rescaling of columns. Therefore, the columns of $\mat{X}$ are linearly independent, and the columns of $\mat{Y}$ are linearly independent. Moreover, since $\rankone{C}{1}$ and $\rankone{C}{2}$ have no nonzero collinear columns nor nonzero collinear rows, condition (i) is also verified.
	
	Therefore, by \Cref{thm:NC-fixed-support-id} and the previously shown equivalence, linear independence of columns $\begin{pmatrix}
		a & b
	\end{pmatrix}^T$ and $\begin{pmatrix}
		c & d
	\end{pmatrix}^T$ is necessary for $\tuplerkone{C} \in \unique(\Gamma_{\tuplerkone{S}})$. We now show that it is sufficient. Suppose that $\transpose{\begin{pmatrix}
			a & b
	\end{pmatrix}}$ and $\transpose{\begin{pmatrix}
			c & d
	\end{pmatrix}}$ are linearly independent. Let $\tuplerkone{C'} \in \Gamma_{\tuplerkone{S}}$ such that $\rankone{C'}{1} + \rankone{C'}{2} = \mat{Z}$.  As $\rankone{C'}{1}$ and $\rankone{C'}{2}$ are rank-one matrices, denote $x', y' \in \mathbb{C}$ such that $\matcol{\rankone{C'}{1}}{2} = x' \matcol{\rankone{C'}{1}}{1}$ and $\matcol{\rankone{C'}{2}}{2} = y' \matcol{\rankone{C'}{2}}{3}$. Rewriting the equality $\matcol{\mat{Z}}{2} = \matcol{(\rankone{C'}{1} + \rankone{C'}{2})}{2}$, we have:
	\begin{equation*}
		\begin{cases}
			a x' + c y' = \matindex{\mat{Z}}{1}{2} = a x + c y \\
			b x' + d y' = \matindex{\mat{Z}}{2}{2} = b x + d y \\
		\end{cases}
	\end{equation*}
	Then, by assumption, $x' = x$ and $y = y'$, so $\tuplerkone{C'} = \tuplerkone{C}$, which shows $\tuplerkone{C} \in \unique(\Gamma_{\tuplerkone{S}})$.
	
\end{proof}

\section{Proof of \Cref{thm:SC-fixed-support-id}}
\label{app:SC-fixed-support-id}

\begin{proof}
	Suppose that $\mathcal{S}$ is closable. Let $\tuplerkone{C} \in \exact{\tuplerkone{S}}$, and let us show that $\tuplerkone{C} \in \unique(\Gamma_{\tuplerkone{S}})$. Let $ \tuplerkone{C}' \in \Gamma_{\tuplerkone{S}}$ such that $\mathcal{A}(\tuplerkone{C}) = \mathcal{A}(\tuplerkone{C'})$. The goal of the proof is to show that $\tuplerkone{C} \sim \tuplerkone{C'}$.
	Denote $\mat{Z} := \mathcal{A}(\tuplerkone{C})$, $\mathcal{G} := \tupleobsgraph{\tuplerkone{S}}$ the observable bipartite graphs associated to $\tuplerkone{S}$, and $N := N(\tuplegraph{G})$. As $\tuplerkone{S}$ is closable, all the rank-one supports $\rankone{S}{i}$ ($1 \leq i \leq r$) are different, so $\graphindex{G}{i}$ is not the empty graph for any $i \in \integerSet{r}$. The complete graph $\completegraph{\graphindex{G}{i}}$ is denoted $\completegraph{i}$ for simplicity ($1 \leq i \leq r$).
	For any integer $n$, denote $\tuplegraph{G}_n := \iterate{(b \circ a)}{n}(\tuplegraph{G})$, where $a$ and $b$ are the completion operations defined in \Cref{def:completion-operations}, and $\tuplegraph{H}_n = a(\tuplegraph{G}_n)$.
	
	For $n \in \integerSet{N}$, denote $P_n$ the assertion: ``for all $i \in \integerSet{r}$, for all $(k,l) \in E \left( \graphsubscriptindex{G}{n}{i} \right)$, $\matindex{(\rankone{C}{i})}{k}{l} = \matindex{(\rankone{C'}{i})}{k}{l}$''. 
	Let us prove by induction that $P_n$ is true for all $n \in \integerSet{N}$. Since $N = N(\tuplegraph{G})$, $\tuplegraph{G}_N = \closure{\tuplegraph{G}}$, and by assumption, $\closure{\tuplegraph{G}} = (\completegraph{i})_{i=1}^r$.
	Therefore, proving $P_N$ would yield $\tuplerkone{C} = \tuplerkone{C}'$ and end the proof.
	
	\paragraph{Initialization}
	Fix $i \in \integerSet{r}$, and $(k,l) \in E(\graphsubscriptindex{G}{0}{i})$. By definition, $E(\graphsubscriptindex{G}{0}{i}) = \obssupport{\tuplerkone{S}}{i} = \rankone{S}{i} \backslash \left( \bigcup_{j \neq i} \rankone{S}{j} \right)$, 
	so for $j \neq i$, $\matindex{(\rankone{C}{j})}{k}{l} = 0 = \matindex{(\rankone{C'}{j})}{k}{l}$. Therefore, $\matindex{(\rankone{C}{i})}{k}{l} = \matindex{\mat{Z}}{k}{l} = \matindex{(\rankone{C'}{i})}{k}{l}$, which shows $P_0$.
	
	\paragraph{Inductive step}
	Let $n \in \integerSet{N-1}$, and suppose that $P_n$ is verified. Denote $\tuplegraph{H} := a(\tuplegraph{G})$. Recall that $\tuplegraph{G}_n \preceq \tuplegraph{H}_n \preceq \tuplegraph{G}_{n+1}$.
	\begin{itemize}
		\item For any $i$, for any $(k,l) \in E(\graphsubscriptindex{G}{n}{i})$, $\matindex{(\rankone{C}{i})}{k}{l} = \matindex{(\rankone{C'}{i})}{k}{l}$ by assumption $H_n$.
		
		\item Let $i \in \integerSet{r}$ and $(k,l) \in E(\graphsubscriptindex{H}{n}{i}) \backslash E(\graphsubscriptindex{G}{n}{i})$. Then, $(k,l)$ corresponds to a missing edge in a subgraph of $\graphsubscriptindex{G}{n}{i}$ isomorph to $\completegraph{2,2}^{-}$. For any $(k', l')$ of this subgraph, $(k', l') \in E(\graphsubscriptindex{G}{n}{i})$ by definition of the completion operation $a$, so $\matindex{(\rankone{C}{i})}{k'}{l'} = \matindex{(\rankone{C'}{i})}{k'}{l'}$.
		Then, by \Cref{prop:connected-bipartite-graph}, we obtain $\matindex{(\rankone{C}{i})}{k}{l} = \matindex{(\rankone{C'}{i})}{k}{l}$. This is true for any $i \in \integerSet{r}$ and $(k, l) \in E(\graphsubscriptindex{H}{n}{i})$.
		
		\item Finally, let $(k,l) \in E(\graphsubscriptindex{H}{n+1}{i}) \backslash E(\graphsubscriptindex{H}{n}{i})$. Then, by definition of the completion operation $b$, $(k,l) \in E(\graphsubscriptindex{H}{n}{j})$ for all $j \neq i$ such that $(k,l) \in \completegraph{j}$. This means that $\matindex{(\rankone{C}{j})}{k}{l} = \matindex{(\rankone{C'}{j})}{k}{l}$ for such $j$. For $j \neq i$ such that $(k,l) \notin \completegraph{j}$, we simply have $\matindex{(\rankone{C}{j})}{k}{l} = 0 = \matindex{(\rankone{C'}{j})}{k}{l}$.
		Therefore, we obtain $\matindex{(\rankone{C}{i})}{k}{l} = \matindex{\mat{Z}}{k}{l} - \sum_{j \neq i} \matindex{(\rankone{C}{j})}{k}{l} = \matindex{\mat{Z}}{k}{l} - \sum_{j \neq i} \matindex{(\rankone{C'}{j})}{k}{l} = \matindex{(\rankone{C'}{i})}{k}{l}$, which shows $P_{n+1}$ and ends the inductive step.
	\end{itemize}
\end{proof}

\section{Proof of \Cref{prop:completion-algo}}
\label{proof:completion-algo}

\begin{proof}
	To prove the first assertion, one can assume by contradiction that the algorithm breaks the loop because of incompatibility, with $\mat{Z} \in \mathcal{A}(\Gamma_{\tuplerkone{S}})$. Fix $\tuplerkone{C} \in \Gamma_{\tuplerkone{S}}$ such that $\mathcal{A}(\tuplerkone{C}) = \mat{Z}$. Then, we would have $\matindex{\rankone{C}{i}}{k}{l} = \matindex{\mat{Z}}{k}{l}$ for all $i \in \integerSet{r}$, $(k, l) \in \obssupport{\tuplerkone{S}}{i}$. Following the completion algorithm, since the algorithm breaks the loop because of incompatibility, there exists an index $(k, l)$ such that $\sum_{i} \matindex{\rankone{C}{i}}{k}{l} \neq \matindex{\mat{Z}}{k}{l}$. This is a contradiction.
	
	For the second assertion, one verifies that, at initialization, for each index $(k,l)$, whenever all entries $\matindex{\rankone{C}{i}}{k}{l}$ ($1 \leq i \leq r$) are not missing, the sum $\sum_{i} \matindex{\rankone{C}{i}}{k}{l}$ is equal to $\matindex{\mat{Z}}{k}{l}$. This property is preserved at the end of each loop of the algorithm, as by assumption, the algorithm did not break the loop because of incompatibility. Since there is no missing values in the output of the algorithm, we obtain $\mathcal{A}(\tuplerkone{C}) = \mat{Z}$. Uniqueness comes simply from the fact that each completed value in the algorithm is unique.
\end{proof}

\section{Proof of \Cref{lemma:disjoint-support-characterizing-uniform-EMD}}
\label{proof:lemma-disjoint-support-characterizing-uniform-EMD}

\begin{proof}
	Suppose (iii). Let $\tuplerkone{C} \in \Gamma_{\tuplerkone{S}}$. 
	Then, the entries of $\rankone{C}{i}$ ($1 \leq i \leq r$) can be directly identified from the submatrix $(\matindex{\mat{Z}}{k}{l})_{(k,l) \in \rankone{S}{i}}$, because the rank-one supports in the tuple $\tuplerkone{S}$ are pairwise disjoint. This shows condition (ii). 
	Suppose now (ii). A fortiori, uniform P-uniqueness of EMD in $\Gamma_{\mathcal{S}}$ holds, because the equality $\tuplerkone{C} = \tuplerkone{C'}$ implies $\tuplerkone{C} \sim \tuplerkone{C'}$, which shows (i). Let us now show the implication (i) $\implies$ (iii) by contraposition. Suppose there exists $i, j \in \integerSet{r}$ such that $\rankone{S}{i} \cap \rankone{S}{j} \neq \emptyset$. Let $(k, l)$ be an index in this intersection. Denote $\mat{E}^{(k, l)} \in \mathbb{B}^{n \times m}$ the binary matrix full of zero, except for the index $(k,l)$ for which the entry is one. Define $\tuplerkone{C}, \tuplerkone{C'} \in \Gamma_{\tuplerkone{S}}$ as follow:
	\begin{displaymath}
		\forall t \in \integerSet{r}, \quad \rankone{C}{t} = \begin{cases}
			\bm{E}^{(k, l)} & \text{if $t=i$} \\
			- \bm{E}^{(k, l)} & \text{if $t=j$} \\
			\mat{0} & \text{otherwise}
		\end{cases}, \quad \rankone{C'}{t} = \begin{cases}
			2 \bm{E}^{(k, l)} & \text{if $t=i$} \\
			- 2 \bm{E}^{(k, l)} & \text{if $t=j$} \\
			\mat{0} & \text{otherwise}
		\end{cases}
	\end{displaymath}
	Then, $ \mathcal{A}(\tuplerkone{C}) = 0 = \mathcal{A}(\tuplerkone{C'})$, but for any permutation $\sigma \in \mathfrak{S}(\integerSet{r})$, $\bm{C'}^{i} \neq \bm{C}{\sigma(i)}$, which means that $\tuplerkone{C} \not\sim \tuplerkone{C'}$. 
\end{proof}

\section{Proof of \Cref{prop:uniform-uniqueness-EMD}}
\label{proof:prop-uniform-uniqueness-EMD}

\begin{proof}
	Suppose $\Gamma_{\Omega} = \unique(\Gamma_{\Omega})$. By analogy with \Cref{lemma:nec-conditions-EMF}, we show that uniform P-uniqueness of EMD in $\Gamma_{\tuplerkone{S}}$ holds for any $\tuplerkone{S} \in \varphi(\Omega)$. This yields condition (i) by applying \Cref{lemma:disjoint-support-characterizing-uniform-EMD}.
	Suppose by contradiction that condition (ii) is not verified. Without loss of generality, we can suppose that condition (i) is verified. Fix $\tuplerkone{S}, \tuplerkone{S}' \in \varphi^{\mathrm{cp}}(\Omega)$ such that $\bigcup_{i=1}^r \rankone{S}{i} = \bigcup_{i=1}^r \rankone{S'}{i}$, but $\tuplerkone{S} \not\sim \tuplerkone{S'}$. By condition (i), $\{ \rankone{S}{i} \}_{i=1}^r$ and $\{ \rankone{S'}{i} \}_{i=1}^r$ are respectively pairwise disjoint. 
	Defining $\tuplerkone{C} := \tuplerkone{S}$ and $\tuplerkone{C'} := \tuplerkone{S'}$, one verifies that $\tuplerkone{C}, \tuplerkone{C'} \in \Gamma_{\Omega}$, and $\mathcal{A}(\tuplerkone{C}) = \mathcal{A}(\tuplerkone{C'})$. But by construction, $\tuplerkone{C} \not\sim \tuplerkone{C'}$, so we conclude that $\tuplerkone{C}, \tuplerkone{C'} \notin \unique(\Gamma_{\Omega})$. 
	
	Suppose that condition (i) and condition (ii) are verified. Let $\tuplerkone{C}, \tuplerkone{C'} \in \Gamma_{\Omega}$ such that $\sum_{i=1}^r \bm{C}^ i= \sum_{i=1}^r \bm{C'}^i$. This sum is denoted $\bm{Z}$. 
	Denote $\tuplerkone{S} := (\supp(\rankone{C}{i}))_{i=1}^r$ and $\tuplerkone{S'} := (\supp(\rankone{C'}{i}))_{i=1}^r$. By construction, $\tuplerkone{S}, \tuplerkone{S'}  \in \varphi^{\mathrm{cp}}(\Omega)$. By condition (i), the supports $\{ \rankone{S}{i} \}_{i=1}^r$ and $\{ \rankone{S'}{i} \}_{i=1}^r$ are respectively pairwise disjoint. Therefore, $\bigcup_{i=1}^r \rankone{S}{i} = \supp(\mat{Z}) = \bigcup_{i=1}^r \rankone{S'}{i}$. By condition (ii), we obtain $\tuplerkone{S} \sim \tuplerkone{S'}$. Fix $\sigma \in \mathfrak{S}(\integerSet{r})$ such that $\rankone{S'}{i} = \rankone{S}{\sigma(i)}$ for all $i \in \integerSet{r}$. 
	Since the supports $\{ \supp(\bm{C'}^t) \}_{t=1}^r$ are pairwise disjoint, we have $\rankone{C'}{i} = \sum_{(k,l) \in \rankone{S'}{i}} \matindex{\mat{Z}}{k}{l} \bm{E}^{(k, l)} $, and $\rankone{C}{\sigma(i)} = \sum_{(k,l) \in \rankone{S}{\sigma(i)}} \matindex{\mat{Z}}{k}{l} \bm{E}^{(k, l)} $ ($1 \leq i \leq r$). 
	We conclude that $\bm{C'}^ i= \bm{C}^{\sigma(i)}$ for all $i \in \integerSet{r}$.
\end{proof}

\section{Proof of \Cref{cor:uniform-right-identifiability-classical-family}}
\label{proof:cor-uniform-right-identifiability-classical-family}

Before proving the corollary, let us start with some technical lemmas.

\begin{lemma}
	\label{lemma:zero-column-uniform-p-uniqueness}
	For $\mat{X} \neq 0$, denote $J := \colsupp(\mat{X})$. Then, for any family of right supports $\Theta$:
	\begin{displaymath}
		\{ \mat{X} \} \times \Sigma_{\Theta} = \unique(\{ \mat{X} \} \times \Sigma_{\Theta}) \iff \begin{cases}
			\Theta_{J^c} \subseteq  \{ \mat{0} \}, \text{ and} \\
			\{ \submat{\mat{X}}{J} \} \times \Sigma_{\signature{\Theta}{J}} = \unique( \{ \submat{\mat{X}}{J} \} \times \Sigma_{\signature{\Theta}{J}} )
		\end{cases},
	\end{displaymath}
	with the convention $\signature{\Theta}{J^c} = \emptyset$ when $J^c = \emptyset$.
\end{lemma}

\begin{proof}
	Suppose $\{ \mat{X} \} \times \Sigma_{\Theta} = \unique( \{\mat{X}\} \times \Sigma_{\Theta})$. Let $\mat{S} \in \Theta$. Viewing $\mat{S}$ as a binary matrix, we have $\mat{S} \in \Sigma_{\Theta}$, so $(\mat{X}, \mat{S}) \in \{ \mat{X} \} \times  \Sigma_{\Theta}$. By assumption, $(\mat{X}, \mat{S}) \in \unique(\{ \mat{X} \} \times \Sigma_{\Theta})$. Hence, by \Cref{lemma:right-id-first-lemma}, we obtain $\colsupp(\mat{S}) \subseteq \colsupp(\mat{X}) = J$. This is true for all $\mat{S} \in \Theta$, so $\signature{\Theta}{J^c} \subseteq \{ \mat{0} \}$. Let us now show that uniform PS-uniqueness of EMF in $\{ \submat{\mat{X}}{J} \} \times \Sigma_{\signature{\Theta}{J}}$ holds. Let $\mat{Y'} \in \Sigma_{\signature{\Theta}{J}}$. Define $\mat{Y} \in \Sigma_{\Theta}$ such that $\submat{\mat{Y}}{J} = \mat{Y'}$ and $\submat{\mat{Y}}{J^c} = \mat{0}$. Then, $(\mat{X}, \mat{Y}) \in \{ \mat{X} \} \times \Sigma_{\Theta}$, so by assumption, $(\mat{X}, \mat{Y}) \in \unique(\{ \mat{X} \} \times \Sigma_{\Theta})$. Since $\colsupp(\mat{Y}) \subseteq J = \colsupp(\mat{X})$, $(\submat{\mat{X}}{J}, \mat{Y'}) \in \unique(\{ \submat{\mat{X}}{J}\} \times \Sigma_{\signature{\Theta}{J}})$ by \Cref{lemma:eliminating-zero-col-right-id}, which ends the proof of necessity.
	
	Conversely, suppose that $\signature{\Theta}{J^c} \subseteq \{ \mat{0} \}$, and $\{ \submat{\mat{X}}{J} \} \times \Sigma_{\signature{\Theta}{J}} = \unique( \{ \submat{\mat{X}}{J} \} \times \Sigma_{\signature{\Theta}{J}} )$. Let $\mat{Y} \in \Sigma_{\Theta}$. Since $\signature{\Theta}{J^c} \subseteq \{ \mat{0} \}$, we necessarily have $\colsupp(\mat{Y}) \subseteq J = \colsupp(\mat{X})$. But $(\submat{\mat{X}}{J}, \submat{\mat{Y}}{J}) \in \unique( \{ \submat{\mat{X}}{J}\} \times \Sigma_{\signature{\Theta}{J}})$ by assumption. Hence, by \Cref{lemma:eliminating-zero-col-right-id}, $(\mat{X}, \mat{Y}) \in \unique( \{ \mat{X} \} \times \Sigma_{\Theta})$, which ends the proof.
\end{proof}

\begin{corollary}
	\label{cor:uniform-right-p-uniqueness-equiv}
	Consider $\mat{X}$ with no zero column and $\mat{N} \in \diagonal{r}$ the (unique) diagonal matrix that normalizes its columns, setting the first nonzero entry of each column of $\mat{X} \mat{N}$ to 1. We have:
	\begin{equation*}
		\begin{split}
			\{ \mat{X} \}\times \Sigma_{\Theta} = \unique(\{ \mat{X} \}\times \Sigma_{\Theta}) \iff& \{ \mat{X} \mat{N} \} \times \Sigma_{\Theta} = \unique(\{ \mat{X} \mat{N} \} \times \Sigma_{\Theta})  \\
			\iff & \{ \mat{X} \mat{N} \} \times \Sigma_{\Theta} = \uniqueperm(\{ \mat{X} \mat{N} \} \times \Sigma_{\Theta})
		\end{split}
	\end{equation*}
\end{corollary}

\begin{proof}
	This is a direct consequence of \Cref{prop:right-id-up-permutation}.
\end{proof}

\begin{proof}[Proof of \Cref{cor:uniform-right-identifiability-classical-family}, (i)]
	Remark that  $\mathcal{I}^\alpha_{\mathtt{row}} := \left \{ \matcol{(\transpose{\mat{S}})}{k} \cup \matcol{(\transpose{\mat{S'}})}{k} \; | \; \mat{S}, \mat{S'} \in \rowsparse{\alpha}, \; k \in \integerSet{n} \right \}$ is equal to the set $\{ I \subseteq \integerSet{r} \; | \; \card(I) \leq 2\alpha\}$.
	Suppose $\krank(\mat{X}) \geq \min(r, 2 \alpha)$. In particular, $\mat{X}$ has no zero column, and there is no pair of collinear columns in $\mat{X}$. By assumption, the subset of columns $\{ \matcol{\mat{X}}{I} \; | \; l \in I\}$ for any $I \in \mathcal{I}^\alpha_{\mathtt{row}}$ is linearly independent.
	By \Cref{cor:uniform-right-uniqueness,prop:uniform-right-uniqueness-inv-prob,prop:uniform-right-id-perm-vector-of-ones}, we conclude that uniform PS-uniqueness of EMF in $\{ \mat{X} \} \times \Sigma_{\rowsparse{\alpha}}$ holds.
	
	Conversely, suppose that uniform PS-uniqueness of EMF in $\{ \mat{X} \} \times \Sigma_{\rowsparse{\alpha}}$ holds, and assume in addition that $n \geq 2$ or $\alpha \geq 2$.
	By \Cref{lemma:zero-column-uniform-p-uniqueness}, $\signature{(\rowsparse{\alpha})}{J^c} \subseteq \{ \mat{0} \}$ where $J := \colsupp(\mat{X})$. But this cannot be verified when $J \neq \integerSet{r}$. So necessarily, $\mat{X}$ does not have zero column. 
	Denote $\{ [k] \}_{k=1}^K$ a partition of $\integerSet{r}$ defined by collinearity of columns in $\mat{X}$. Let us show that there is no pair of collinear columns in $\mat{X}$. Suppose there exists $k \in \integerSet{K}$ such that $\card([k]) \geq 2$, and fix $i, j \in [k]$ such that $i \neq j$.
	By \Cref{cor:uniform-right-p-uniqueness-equiv,cor:uniform-right-uniqueness}, the two conditions of \Cref{prop:uniform-right-id-perm-vector-of-ones} are verified for  $\signature{(\rowsparse{\alpha})}{[k]}$. 
	In the case where $\alpha \geq 2$, we can find $\mat{S} \in \rowsparse{\alpha}$ such that the two columns $\matcol{\mat{S}}{i}$ and $\matcol{\mat{S}}{j}$ has a nonzero entry at the first row, which contradicts the first condition of \Cref{prop:uniform-right-id-perm-vector-of-ones} applied to $\signature{(\rowsparse{\alpha})}{[k]}$.
	Otherwise, in the case where $n \geq 2$, we can find $\mat{S}, \mat{S'} \in \rowsparse{\alpha}$ such that:
	\begin{equation*}
		\begin{split}
			\matcol{\mat{S}}{i} = \transpose{\begin{pmatrix}
					1 & 0 & 0 & \ldots & 0
			\end{pmatrix}}, \quad \matcol{\mat{S'}}{i} = \transpose{\begin{pmatrix}
					1 & 1 & 0 & \ldots & 0
			\end{pmatrix}}, \\
			\matcol{\mat{S}}{j} = \transpose{\begin{pmatrix}
					0 & 1 & 0 & \ldots & 0
			\end{pmatrix}}, \quad \matcol{\mat{S'}}{j} = \transpose{\begin{pmatrix}
					0 & 0 & 0 & \ldots & 0
			\end{pmatrix}},
		\end{split}
	\end{equation*}
	with zero columns elsewhere. This is in contradiction with the second condition of \Cref{prop:uniform-right-id-perm-vector-of-ones} applied to $\signature{(\rowsparse{\alpha})}{[k]}$.
	Consequently, $\{ [k] \}_{k=1}^K$ is a partition of $\integerSet{r}$ where each equivalence class is a singleton, so the fingerprint of $\rowsparse{\alpha}$ on $\{ [k] \}_{k=1}^K$ is simply $\rowsparse{\alpha}$. By \Cref{cor:uniform-right-uniqueness,prop:right-uniqueness-linear-inv-pb-matrix}, each subset of columns $\{ \matcol{\mat{X}}{l} \; | \; l \in I\}$ for $I \in \mathcal{I}^\alpha_{\mathtt{row}}$ is linearly independent, meaning that $\krank(\mat{X}) \geq \min(r, 2 \alpha)$.
\end{proof}

\begin{proof}[Proof of (ii)]
	Suppose $\krank(\mat{X}) = r$. In particular, $\mat{X}$ has no zero column, and there is no pair of collinear columns in $\mat{X}$. By assumption, the subset of columns $\{ \matcol{\mat{X}}{I} \; | \; l \in I\}$ for any $I \in \mathcal{I}^\beta_{\mathtt{col}}$ is linearly independent, where $\mathcal{I}^\beta_{\mathtt{col}} := \{ \matcol{(\transpose{\mat{S}})}{k} \cup \matcol{(\transpose{\mat{S'}})}{k} \; | \; \mat{S}, \mat{S'} \in \columnsparse{\beta}, \; k \in \integerSet{n} \}$.
	By \Cref{cor:uniform-right-uniqueness,prop:uniform-right-uniqueness-inv-prob,prop:uniform-right-id-perm-vector-of-ones}, we conclude that uniform PS-uniqueness of EMF in $\{ \mat{X} \} \times \Sigma_{\columnsparse{\beta}}$ holds.
	Conversely, suppose that PS-uniqueness of EMF in $\{ \mat{X} \} \times \Sigma_{\columnsparse{\beta}}$ holds. 
	Using some similar arguments as the previous proof, we can show that necessarily, $\mat{X}$ has no zero column, and no pair of collinear columns.
	Consequently, $\{ [k] \}_{k=1}^K$ is a partition of $\integerSet{r}$ where each equivalence class is a singleton, so the fingerprint of $\columnsparse{\beta}$ on $\{ [k] \}_{k=1}^K$ is simply $\columnsparse{\beta}$. 
	By \Cref{cor:uniform-right-uniqueness,prop:right-uniqueness-linear-inv-pb-matrix}, each subset of columns $\{ \matcol{\mat{X}}{l} \; | \; l \in I\}$ for $I \in \mathcal{I}^\beta_{\mathtt{col}}$ is linearly independent, meaning that $\krank(\mat{X}) = r$, because $\integerSet{r} = \matcol{(\transpose{\mat{S}})}{k} \cup \matcol{(\transpose{\mat{S}})}{k} \in \mathcal{I}^\beta_{\mathtt{col}}$, where $\mat{S}$ is the support full of ones in the first column, and filled with zero elsewhere.
\end{proof}

\begin{proof}[Proof of (iii)]
	Suppose $\krank(\mat{X}) \geq \min(r, 2\alpha)$. Then, by assertion (i) of \Cref{cor:uniform-right-identifiability-classical-family}, uniform PS-uniqueness of EMF in $\{ \mat{X} \} \times \Sigma_{\rowsparse{\alpha}}$ holds. But since $\Sigma_{\rowsparse{\alpha} \cap \columnsparse{\beta}} \subseteq \Sigma_{\rowsparse{\alpha}}$, applying \Cref{lemma:nec-conditions-EMF} yields uniform PS-uniqueness of EMF in $\{ \mat{X} \} \times \Sigma_{\rowsparse{\alpha} \cap \columnsparse{\beta}}$.
	The proof of the converse is essentially the same as the one of assertion (i) in \Cref{cor:uniform-right-identifiability-classical-family}.
\end{proof}

\begin{proof}[Proof of (iv)]
	Remark that  $\mathcal{I}^s_{\mathtt{glob}} := \{ \matcol{(\transpose{\mat{S}})}{k} \cup \matcol{(\transpose{\mat{S'}})}{k} \; | \; \mat{S}, \mat{S'} \in \globalsparse{s}, \; k \in \integerSet{n} \}$ is equal to the set $\{ I \subseteq \integerSet{r} \; | \; \card(I) \leq 2s \}$. 
	Hence, the proof is essentially the same as the one for assertion (i) of \Cref{cor:uniform-right-identifiability-classical-family}.
\end{proof}

